%% file: main.tex
\documentclass[lettersize,journal]{cls/IEEEtran}

\usepackage{enumitem}
\usepackage[table,usenames,dvipsnames]{xcolor}      
\usepackage[noadjust]{cite}
\usepackage{amsmath,amssymb,amsfonts,amsthm,dsfont,bbm} 
\usepackage{algorithm,algorithmicx,listings}        
\usepackage[noend]{algpseudocode}			        
\usepackage{graphicx,tabularx,adjustbox}
\usepackage{subcaption}
\usepackage{thmtools,thm-restate}
\usepackage[breaklinks=true, colorlinks, bookmarks=true, citecolor=Black, urlcolor=Violet,linkcolor=Black]{hyperref}

\newtheorem{theorem}{Theorem}

\newtheorem*{proposition*}{Proposition}

\newtheorem*{corollary*}{Corollary}

\theoremstyle{definition}
\newtheorem{definition}{Definition}[]

\newtheorem*{assumption*}{Assumption}
\newtheorem*{problem*}{Problem}
\newtheorem{problem}{Problem}
\theoremstyle{remark}

\newtheorem*{solution*}{Solution}

\newtheorem*{example*}{Example}

\newcommand{\prl}[1]{\left(#1\right)}
\newcommand{\brl}[1]{\left[#1\right]}

\algnewcommand{\IfThenElse}[3]{
	\State \algorithmicif\ #1\ \algorithmicthen\ #2\ \algorithmicelse\ #3}

\DeclareMathOperator{\tr}{tr}
\DeclareMathOperator{\diag}{diag}

\newcommand{\NEW}[1]{{\color{black}#1}}
\newcommand{\NEWW}[1]{{\color{black}#1}}

\input{cls/sym.tex}

\begin{document}

\title{Port-Hamiltonian Neural ODE Networks on Lie Groups For Robot Dynamics Learning and Control}
\author{Thai~Duong,~\IEEEmembership{Student Member,~IEEE,} Abdullah~Altawaitan,~\IEEEmembership{Student Member,~IEEE,} Jason Stanley~\IEEEmembership{Student Member,~IEEE,}~and~Nikolay Atanasov,~\IEEEmembership{Senior Member,~IEEE}
	\thanks{We gratefully acknowledge support from NSF CCF-2112665 (TILOS).}%
	\thanks{The authors are with the Department of Electrical and Computer Engineering, University of California San Diego, 
	La Jolla, CA 92093, USA (e-mails: \{tduong,aaltawaitan,jtstanle,natanasov\}@ucsd.edu). A. Altawaitan is also affiliated with Kuwait University as a holder of a scholarship.}
}

\markboth{IEEE Transactions on Robotics}
{Duong \MakeLowercase{\textit{et al.}}: Port-Hamiltonian Neural ODE Networks on Lie Groups For Robot Dynamics Learning and Control}


\maketitle

\begin{abstract}
Accurate models of robot dynamics are critical for safe and stable control and generalization to novel operational conditions. Hand-designed models, however, may be insufficiently accurate, even after careful parameter tuning. This motivates the use of machine learning techniques to approximate the robot dynamics over a training set of state-control trajectories. The dynamics of many robots are described in terms of their generalized coordinates on a matrix Lie group, e.g. on $SE(3)$ for ground, aerial, and underwater vehicles, and generalized velocity, and satisfy conservation of energy principles. This paper proposes a \NEWW{port-Hamiltonian} formulation over a Lie group of the structure of a neural ordinary differential equation (ODE) network to approximate the robot dynamics. In contrast to a black-box ODE network, our formulation \NEWW{embeds} energy conservation principle and Lie group's  constraints \NEWW{in the dynamics model} and explicitly accounts for energy-dissipation effect such as friction and drag forces in the dynamics model. We develop energy shaping and damping injection control for the learned, potentially under-actuated Hamiltonian dynamics to enable a unified approach for stabilization and trajectory tracking with various robot platforms.
\end{abstract}


\input{tex/Introduction.tex}
\input{tex/RelatedWork.tex}

\input{tex/ProblemFormulation.tex}

\input{tex/Prelim.tex}

\input{tex/TechnicalApproach.tex}

\input{tex/ControllerDesign.tex}

\input{tex/ExperimentalResults.tex}
\input{tex/Conclusion.tex}

\input{tex/Appendix.tex}

\bibliographystyle{cls/IEEEtran}
\bibliography{bib/thai_ref_clean.bib}

\begin{IEEEbiography}[{\includegraphics[width=1in,height=1.25in,clip,keepaspectratio]{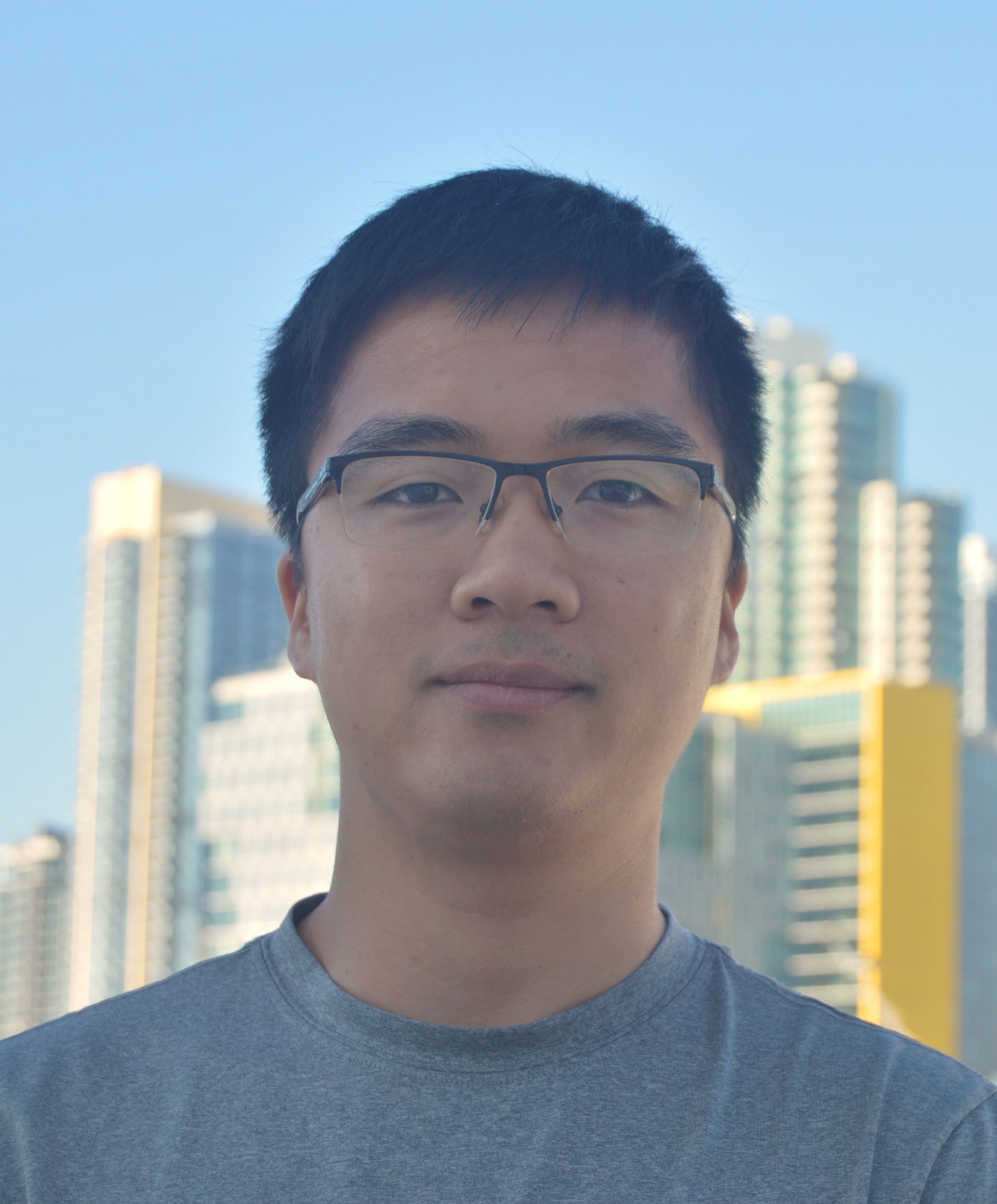}}]{Thai Duong}
	is a PhD candidate in Electrical and Computer Engineering at the University of California, San Diego. He received a B.S. degree in Electronics and Telecommunications from Hanoi University of Science and Technology, Hanoi, Vietnam in 2011 and an M.S. degree in Electrical and Computer Engineering from Oregon State University, Corvallis, OR,  in 2013. His research interests include machine learning with applications to robotics, mapping and active exploration using mobile robots, robot dynamics learning, planning and control.
\end{IEEEbiography}
\begin{IEEEbiography}[{\includegraphics[width=1in,height=1.25in,clip,keepaspectratio]{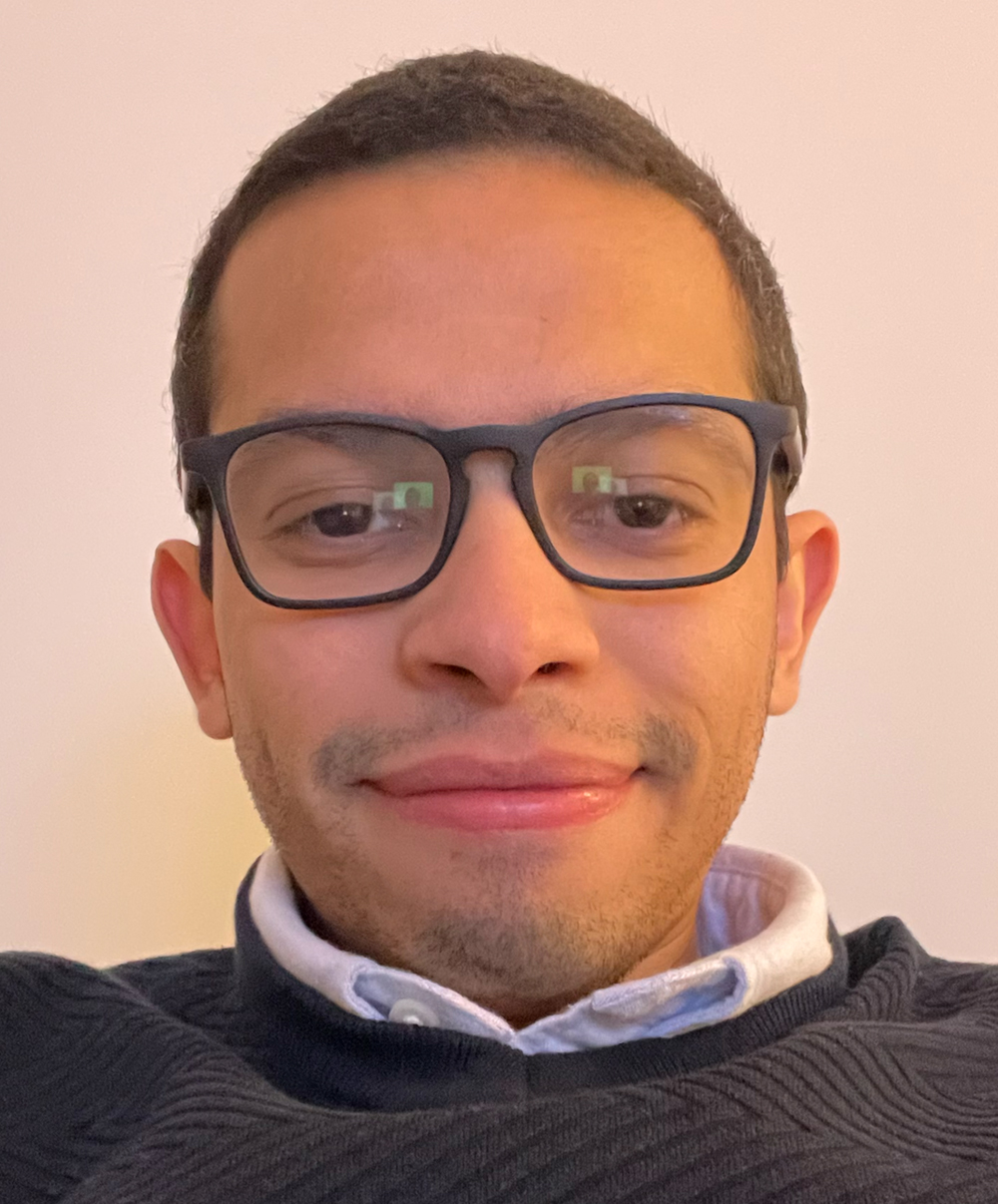}}]{Abdullah Altawaitan}
	 is a Ph.D. student in Electrical and Computer Engineering at the University of California San Diego, La Jolla, CA, USA. He received both the B.S. and M.S. degrees in Electrical Engineering from Arizona State University, Tempe, AZ, USA. He is also affiliated with the Electrical Engineering Department, College of Engineering and Petroleum, Kuwait University, Safat, Kuwait. His research interests include machine learning, control theory, and their applications to robotics.
\end{IEEEbiography}
\begin{IEEEbiography}
[{\includegraphics[width=1in,height=1.25in,clip,keepaspectratio]{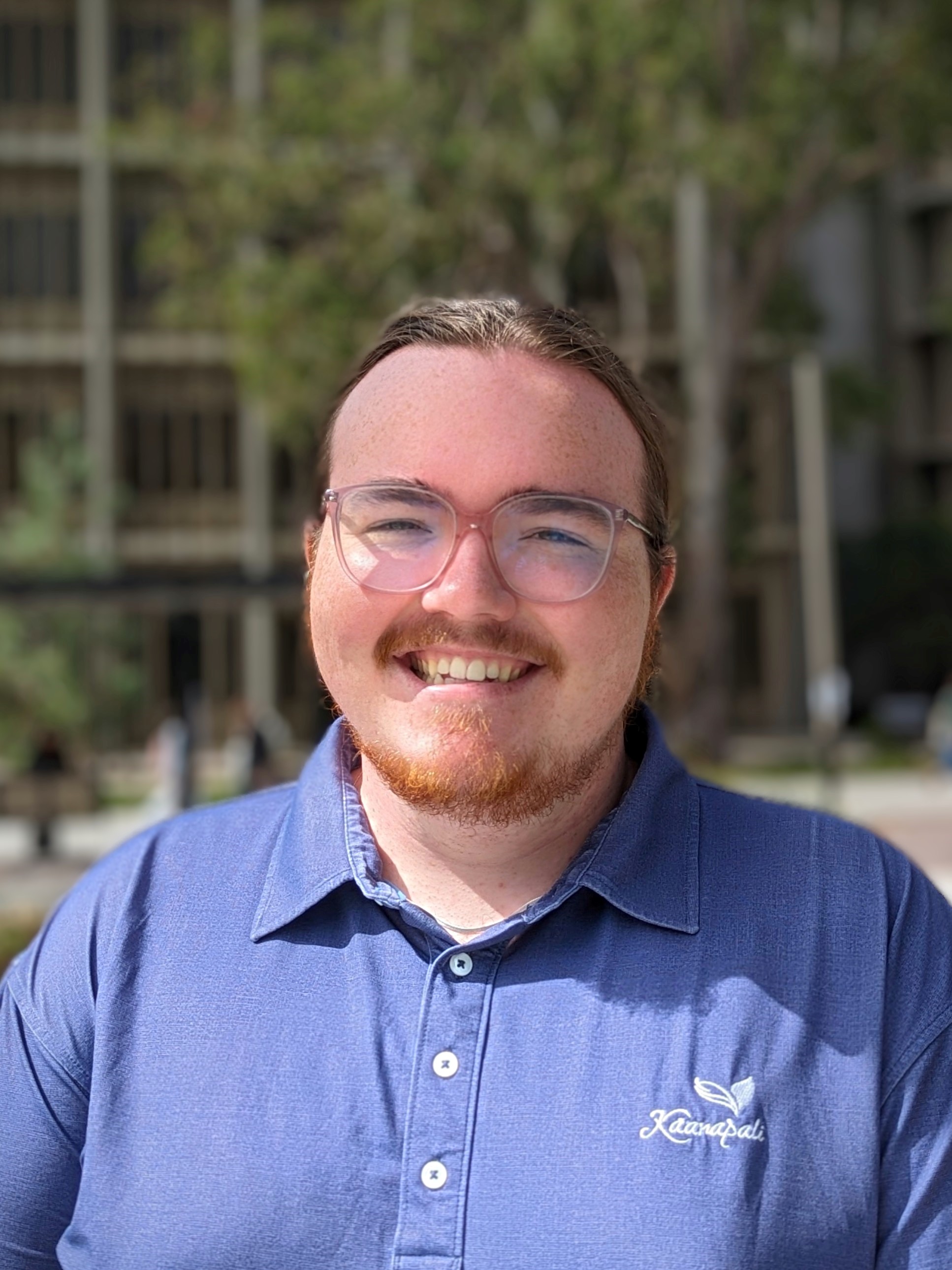}}]{Jason Stanley}
    is an undergraduate student in the Electrical and Computer Engineering department at the University of California, San Diego (UCSD). He is pursuing a major in Computer Engineering, and is planning to continue and do a Masters in Electrical Engineering, also at UCSD. His research interests include control theory, machine learning for robotics, and  mobile robots.
\end{IEEEbiography}

\begin{IEEEbiography}[{\includegraphics[width=1in,height=1.25in,clip,keepaspectratio]{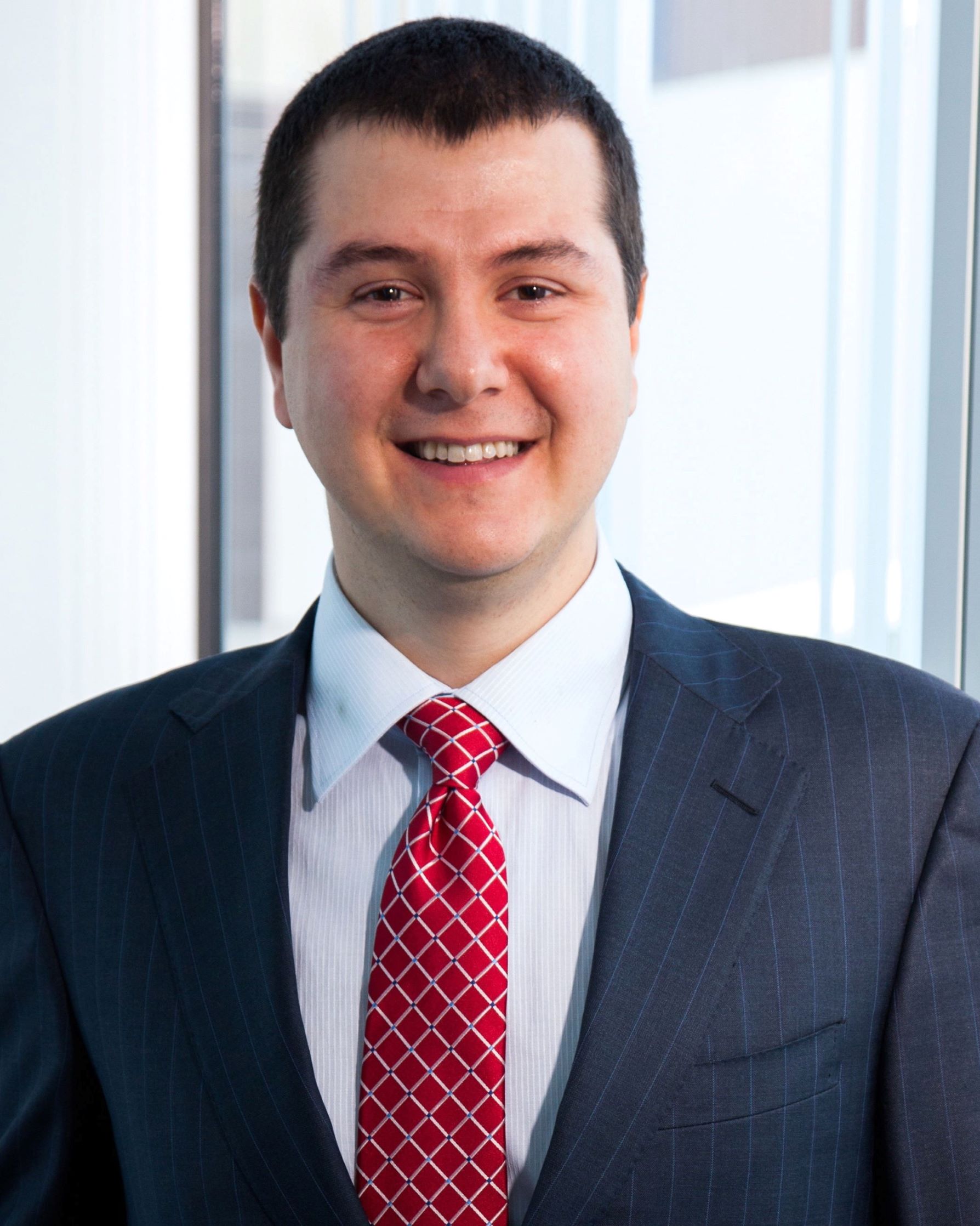}}]{Nikolay Atanasov}
(S'07-M'16-SM'23) is an Associate Professor of Electrical and Computer Engineering at the University of California San Diego, La Jolla, CA, USA. He obtained a B.S. degree in Electrical Engineering from Trinity College, Hartford, CT, USA in 2008 and M.S. and Ph.D. degrees in Electrical and Systems Engineering from the University of Pennsylvania, Philadelphia, PA, USA in 2012 and 2015, respectively. Dr. Atanasov's research focuses on robotics, control theory, and machine learning, applied to active perception problems for autonomous mobile robots. He works on probabilistic models that unify geometric and semantic information in simultaneous localization and mapping (SLAM) and on optimal control and reinforcement learning algorithms for minimizing probabilistic model uncertainty. Dr. Atanasov's work has been recognized by the Joseph and Rosaline Wolf award for the best Ph.D. dissertation in Electrical and Systems Engineering at the University of Pennsylvania in 2015, the Best Conference Paper Award at the IEEE International Conference on Robotics and Automation (ICRA) in 2017, the NSF CAREER Award in 2021, and the IEEE RAS Early Academic Career Award in Robotics and Automation in 2023.
\end{IEEEbiography}

\end{document}

%% file: cls/sym.tex
\newcommand{\calD}{{\cal D}}

\newcommand{\calG}{{\cal G}}
\newcommand{\calH}{{\cal H}}

\newcommand{\calJ}{{\cal J}}

\newcommand{\calL}{{\cal L}}

\newcommand{\calR}{{\cal R}}

\newcommand{\calT}{{\cal T}}

\newcommand{\calV}{{\cal V}}

\newcommand{\frakg}{{\boldsymbol{\mathfrak{g}}}}

\newcommand{\frakp}{{\boldsymbol{\mathfrak{p}}}}
\newcommand{\frakq}{{\boldsymbol{\mathfrak{q}}}}

\newcommand{\frakse}{{\boldsymbol{\mathfrak{s}}\boldsymbol{\mathfrak{e}}}}

\newcommand{\bfa}{\mathbf{a}}
\newcommand{\bfb}{\mathbf{b}}
\newcommand{\bfc}{\mathbf{c}}
\newcommand{\bfd}{\mathbf{d}}
\newcommand{\bfe}{\mathbf{e}}
\newcommand{\bff}{\mathbf{f}}

\newcommand{\bfh}{\mathbf{h}}

\newcommand{\bfp}{\mathbf{p}}

\newcommand{\bfr}{\mathbf{r}}
\newcommand{\bfs}{\mathbf{s}}

\newcommand{\bfu}{\mathbf{u}}
\newcommand{\bfv}{\mathbf{v}}

\newcommand{\bfx}{\mathbf{x}}
\newcommand{\bfy}{\mathbf{y}}

\newcommand{\bfzeta}{\boldsymbol{\zeta}}
\newcommand{\bfeta}{\boldsymbol{\eta}}
\newcommand{\bftheta}{\boldsymbol{\theta}}

\newcommand{\bfmu}{\boldsymbol{\mu}}

\newcommand{\bfpi}{\boldsymbol{\pi}}

\newcommand{\bftau}{\boldsymbol{\tau}}

\newcommand{\bfvarphi}{\boldsymbol{\varphi}}

\newcommand{\bfpsi}{\boldsymbol{\psi}}
\newcommand{\bfomega}{\boldsymbol{\omega}}
\newcommand{\bfxi}{\boldsymbol{\xi}}

\newcommand{\bfA}{\mathbf{A}}
\newcommand{\bfB}{\mathbf{B}}

\newcommand{\bfD}{\mathbf{D}}

\newcommand{\bfI}{\mathbf{I}}
\newcommand{\bfJ}{\mathbf{J}}
\newcommand{\bfK}{\mathbf{K}}
\newcommand{\bfL}{\mathbf{L}}
\newcommand{\bfM}{\mathbf{M}}

\newcommand{\bfR}{\mathbf{R}}

\newcommand{\bbR}{\mathbb{R}}
\newcommand{\bbS}{\mathbb{S}}

\newcommand{\sfG}{\mathsf{G}}

\newcommand{\sfL}{\mathsf{L}}

\newcommand{\sfP}{\mathsf{P}}

\newcommand{\sfT}{\mathsf{T}}

\newcommand{\sfa}{\mathsf{a}}

\newcommand{\sfd}{\mathsf{d}}

\newcommand{\sfg}{\mathsf{g}}

\newcommand{\sfad}{\mathsf{a}\mathsf{d}}
\newcommand{\sfAd}{\mathsf{A}\mathsf{d}}

%% file: tex/Introduction.tex
\section*{Supplementary Material}
Software and videos supplementing this paper:\\ 
\centerline{\url{https://thaipduong.github.io/LieGroupHamDL}}

\section{Introduction}
\label{sec:intro}

Motion planning and optimal control algorithms depend on the availability of accurate system dynamics models. Models obtained from first principles and calibrated over a small set of parameters via system identification~\cite{ljung1999system} are widely used for unmanned ground vehicles (UGVs), unmanned aerial vehicles (UAVs), and unmanned underwater vehicles (UUVs). Such models may over-simplify or even incorrectly describe the underlying structure of the dynamical system, leading to bias and modeling errors that cannot be corrected by adjusting a few parameters. Data-driven techniques \cite{nguyen2011model,deisenroth2011pilco,williams2017information, raissi2018multistep,chua2018deep} have emerged as a powerful approach to approximate system dynamics with an over-parameterized machine learning model, trained over a dataset of system state and control trajectories. Neural networks are expressive function approximation models, capable of identifying and generalizing interaction patterns from the training data. Training neural network models, however, typically requires large amounts of data and computation time, which may be impractical in mobile robotics applications. Recent works \cite{lutter2023combining,gupta2019general,cranmer2020lagrangian,greydanus2019hamiltonian,chen2019symplectic,roehrl2020modeling, lu2022ModLaNets, neary2023compositional} have considered a hybrid \NEWW{(gray-box)} approach, where prior knowledge of the physics, governing the system dynamics, is used to assist the learning process. The dynamics of physical systems obey kinematic constraints and energy conservation laws. These laws are known to be universally true but a black-box machine learning model might struggle to infer them from the training data, causing poor generalization. Instead, prior knowledge may be encoded into the learning model, e.g., using a prior distribution \cite{deisenroth2011pilco}, a graph-network forward kinematic model \cite{sanchez2018graph}, or a network architecture reflecting the structure of Lagrangian \cite{lutter2019deeplagrangian} or Hamiltonian \cite{greydanus2019hamiltonian} mechanical systems. Moreover, many physical robot platforms are composed of rigid-body interconnections and their state evolution respects the structure of a Lie group \cite{hall2013lie}, e.g., the position and orientation kinematics of a rigid body evolve on the $SE(3)$ Lie group \cite{LynchParkBook}. \NEWW{Existing works \cite{falorsi2020neural, elamvazhuthi2023learning, lou2020neural}  on Lie group neural ODEs networks for learning dynamics and normalizing flows focus on preservation of the Lie group structure during backpropagation using an adjoint method, either via a higher-dimensional space \cite{falorsi2020neural, elamvazhuthi2023learning} or local coordinates \cite{lou2020neural, wotte2024optimal}.}

\begin{figure}[t]
    \centering
    \includegraphics[width=\linewidth]{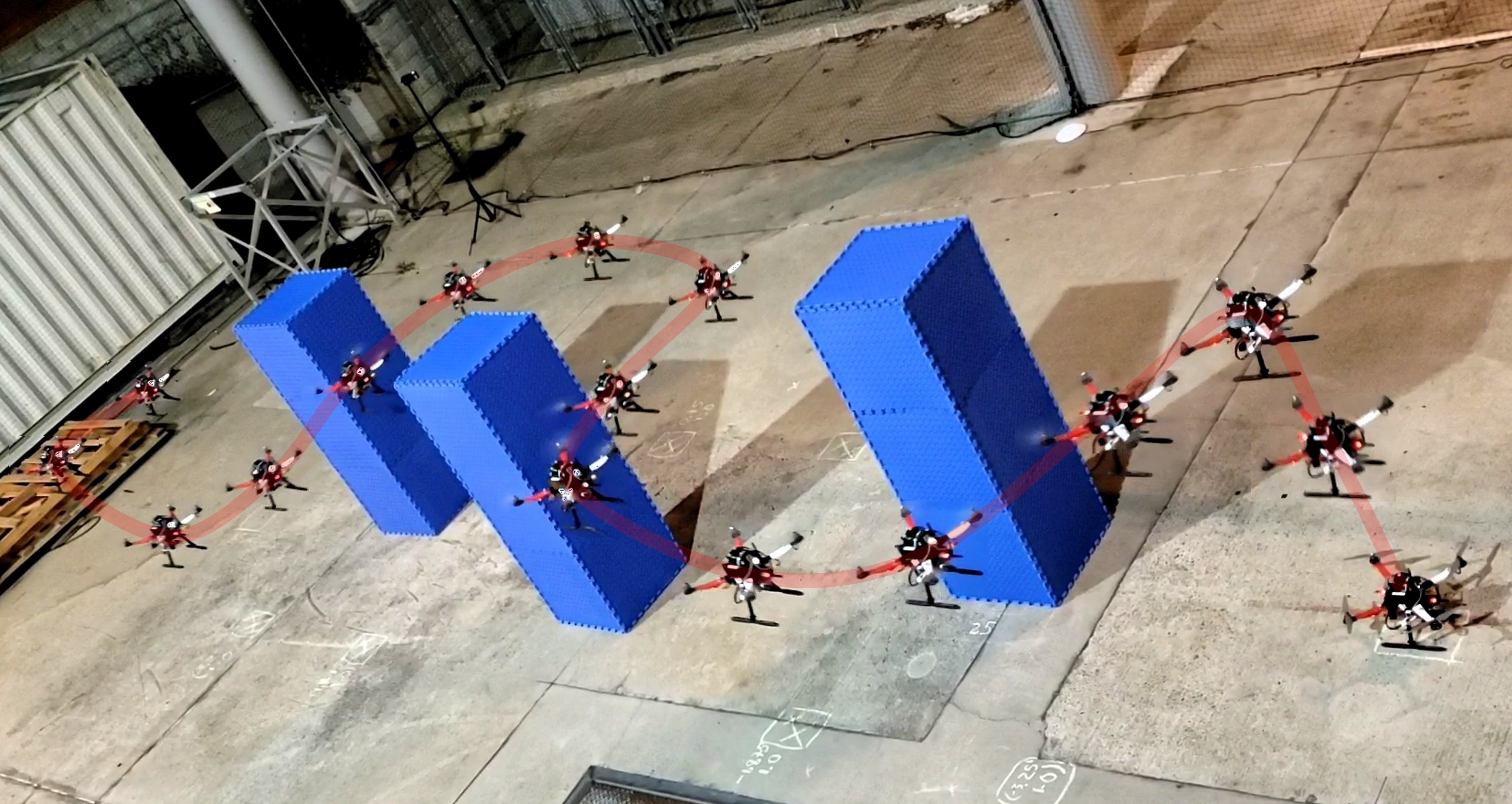}
    \caption{Quadrotor trajectory tracking using a learned port-Hamiltonian dynamics model.}
    \label{fig:zigzag_traj_tracking}
\end{figure}
The goal of this paper is to incorporate both the kinematic structure and the energy conservation properties of physical systems with Lie group states into the structure of a dynamics learning model. We also aim to design a control approach that achieves stabilization or trajectory tracking using the learned model without requiring prior knowledge of its parameters. In other words, the same control design should enable trajectory tracking for learned models of different rigid-body UGVs, UAVs, or UUVs.

Lagrangian and Hamiltonian mechanics \cite{lurie2013analytical,HolmBook} provide physical system descriptions that can be integrated into the structure of a neural network \cite{greydanus2019hamiltonian, bertalan2019learning, chen2019symplectic, finzi2020simplifying, zhong2020symplectic, willard2020integrating}. Prior work, however, has only considered vector-valued states, when designing Lagrangian- or Hamiltonian-structured neural networks. This limits the applicability of these techniques as many common robot systems have states on a Lie group. For example, Hamiltonian equations of motion are available for orientation but existing formulations rely predominantly on $3$ dimensional vector parametrizations, such as Euler angles \cite{maciejewski1985hamiltonian,shivarama2004hamilton}, which suffer from singularities. \NEWW{Similar to \cite{falorsi2020neural, elamvazhuthi2023learning}, our work 
embeds matrix Lie groups in a higher-dimensional space $\bbR^{n\times n}$, allowing us to train the model via the widely used neural ODE network on Euclidean space \cite{chen2018neural}. However, we focus on embedding the Hamiltonian formulation of robot dynamics in the model and deriving a control policy for trajectory tracking. Concurrently, Wotte et al. \cite{wotte2024optimal} offer an exciting approach to learn $SE(3)$ Hamiltonian dynamics from data using a neural ODE on Lie group, trained via local coordinates \cite{lou2020neural} with guarantees of satisfaction of Lie group constraints. The authors designed an adjoint method on the Lie algebra for neural ODE training, and a potential shaping controller for stabilization of a fully actuated rigid body. 

Our preliminary work \cite{duong21hamiltonian} designed a neural ODE network  \cite{chen2018neural} to capture Hamiltonian dynamics on the $SE(3)$ manifold \cite{lee2017global}, and \NEWW{derived} a trajectory tracking control policy for potentially under-actuated systems. Our model is shown to provide accurate long-term trajectory predictions, respecting $SE(3)$ constraints and conserve total energy with high precision}. Inspired by \cite{zhong2020symplectic, zhong2020dissipative}, we model kinetic energy and potential energy by separate neural networks, each governed by a set of Hamiltonian equations on $SE(3)$. \NEWW{However, our preliminary work \cite{duong21hamiltonian}} is developed specifically for the $SE(3)$ manifold and \NEWW{does not model} dissipation elements that drain energy from the system, such as friction or drag forces \NEWW{in real robot systems}. \NEWW{The Hamiltonian formulation of robot dynamics can be generalized to port-Hamiltonian formulation, where the dynamics are governed by energy exchange, i.e., the law of energy conservation, energy dissipation, e.g., from friction, and energy injection, e.g., from control inputs.} In this paper, we generalize our \NEWW{dynamics learning and control method in \cite{duong21hamiltonian} using a port-Hamiltonian} neural ODE network to embed general matrix Lie group constraints and introduce an energy dissipation term, represented by another neural network, to model friction and air drag in physical systems. \NEWW{We compensate for energy dissipation in the trajectory tracking control design to provide accurate tracking performance}. We verify our approach extensively with simulated robot systems, including a pendulum and a Crazyflie quadrotor, and with several real quadrotor platforms. In summary, this paper makes the following \emph{contributions}.

\begin{itemize}

  \item We design a neural ODE model that respects port-Hamiltonian dynamics over \NEWW{\emph{a matrix Lie group} to enable data-driven learning of robot dynamics}.
  
  \item We develop a unified control policy for port-Hamiltonian dynamics on a Lie group that achieves trajectory tracking if permissible by the system's degree of underactuation.
  
  \item We demonstrate our dynamics learning and control approach \NEWW{extensively} with simulated robot systems (a pendulum and a quadrotor) and several real quadrotor robots.
\end{itemize}

%% file: tex/RelatedWork.tex
\section{Related Work}
\label{sec:related_work}

Data-driven techniques \cite{song2023reaching, salzmann2023real, scaramuzza2022learning, loquercio2021learning, ibarz2021train} have shown impressive results in learning robot dynamics models from state-control trajectories. Neural networks offer especially expressive system models but their training requires large amounts of data, which may be impractical in mobile robot applications. Recently, a hybrid approach \cite{lutter2023combining,gupta2019general,cranmer2020lagrangian,greydanus2019hamiltonian,chen2019symplectic,roehrl2020modeling, lu2022ModLaNets, neary2023compositional, willard2020integrating, nghiem2023physics, djeumou2022neural, thorpe2023physics}  has been considered where prior knowledge about a physical system is integrated into the design of a machine learning model. Models designed with structure respecting kinematic constraints \cite{sanchez2018graph}, symmetry \cite{ruthotto2019deep,wang2020incorporating}, Lagrangian mechanics \cite{roehrl2020modeling, lutter2019deeplagrangian,gupta2019general,cranmer2020lagrangian,lutter2019deepunderactuated, lutter2023combining, lu2022ModLaNets} or Hamiltonian mechanics \cite{greydanus2019hamiltonian, bertalan2019learning, chen2019symplectic, finzi2020simplifying, zhong2020symplectic, willard2020integrating, neary2023compositional, beckers2022gaussian, beckers2023bayesian} guarantee that the laws of physics are satisfied by construction, regardless of the training data. 

Sanchez-Gonzalez et al. \cite{sanchez2018graph} design graph neural networks to represent the kinematic structure of complex dynamical systems and demonstrate forward model learning and online planning via gradient-based trajectory optimization. Ruthotto et al. \cite{ruthotto2019deep} propose a partial differential equation (PDE) interpretation of convolutional neural networks and derive new parabolic and hyperbolic ResNet architectures guided by PDE theory. Wang et al. \cite{wang2020incorporating} design symmetry equivariant neural network models, encoding rotation, scaling, and uniform motion, to learn physical dynamics that are robust to symmetry group distributional shifts. 

Lagrangian-based methods \cite{roehrl2020modeling, lutter2019deeplagrangian,gupta2019general,cranmer2020lagrangian,lutter2019deepunderactuated, lutter2023combining, lu2022ModLaNets} design neural network models for physical systems based on the Euler-Lagrange differential equations of motion \cite{lurie2013analytical,HolmBook}, in terms of generalized coordinates $\mathbf\frakq$, their velocity $\dot{\mathbf\frakq}$ and a Lagrangian function $\mathcal{L}(\mathbf\frakq, \dot{\mathbf\frakq})$, defined as the difference between the kinetic and potential energies. The energy terms are modeled by neural networks, either separately \cite{lutter2019deeplagrangian,lutter2019deepunderactuated, lutter2023combining} or together \cite{cranmer2020lagrangian}. 

Hamiltonian-based methods \cite{greydanus2019hamiltonian, bertalan2019learning, chen2019symplectic, finzi2020simplifying, zhong2020symplectic, willard2020integrating, neary2023compositional} use a Hamiltonian formulation \cite{lurie2013analytical,HolmBook} of the system dynamics, instead, in terms of generalized coordinates $\mathbf\frakq$, generalized momenta $\mathbf\frakp$, and a Hamiltonian function, $\mathcal{H}(\mathbf\frakq, \mathbf\frakp)$, representing the total energy of the system.
Greydanus et al. \cite{greydanus2019hamiltonian} model the Hamiltonian as a neural network and update its parameters by minimizing the discrepancy between its symplectic gradients and the time derivatives of the states $(\mathbf\frakq, \mathbf\frakp)$. This approach, however, requires that the state time derivatives are available in the training data set. Chen et al. \cite{chen2019symplectic}, Zhong et al. \cite{zhong2020symplectic} relax this assumption by using differentiable leapfrog integrators \cite{leimkuhler2004simulating} and differentiable ODE solvers \cite{chen2018neural}, respectively.  The need for time derivatives of the states is eliminated by back-propagating a loss function measuring state discrepancy through the ODE solvers via the adjoint method. \NEW{Our work extends the approach in \cite{zhong2020symplectic, zhong2020dissipative} by formulating the Hamiltonian dynamics over a matrix Lie group, which enforces kinematic constraints in the neural ODE network used to learn the dynamics.}
Toth et al. \cite{toth2019hamiltonian} and Mason et al. \cite{mason2022learning} show that, instead of from state trajectories, the Hamiltonian function can be learned from high-dimensional image observations. Finzi et al. \cite{finzi2020simplifying} show that using Cartesian coordinates with explicit constraints improves both the accuracy and data efficiency for the Lagrangian- and Hamiltonian-based approaches. In a closely related work, Zhong et al. \cite{zhong2020dissipative} showed that dissipating elements, such as friction or air drag, can be incorporated in a Hamiltonian-based neural ODE network by reformulating the system dynamics in port-Hamiltonian form \cite{van2014port}. The continuous-time equations of motions in Lagrangian or Hamiltonian dynamics can also be discretized using variational integrators \cite{MaWe2001} to learn discrete-time Lagrangian and Hamiltonian systems \cite{havens2021forced, duruisseaux2023lie, chen2021datadriven, so2022data} and provide long-term prediction for control methods such as model predictive control \cite{borrelli_MPC_book}. This approach eliminates the need to use an ODE solver to roll out the dynamics but its prediction accuracy depends on the discretization time step. \NEWW{Meanwhile, our work encodes not only the Hamiltonian structure but also the Lie group constraints, satisfied by the states of rigid-body robot systems, such as UGVs, UAVs and UUVs, in a neural ODE network to learn robot dynamics from data.}

While most existing dynamics learning approaches focus on Euclidean dynamics, many robot systems have states evolving on a matrix Lie group. \NEWW{Recent works on neural ODE networks \cite{falorsi2020neural, elamvazhuthi2023learning, lou2020neural, wotte2024optimal} on Lie groups are classified into either extrinsic \cite{falorsi2020neural, elamvazhuthi2023learning} or intrinsic methods \cite{lou2020neural, wotte2024optimal}. The extrinsic approach \cite{falorsi2020neural, elamvazhuthi2023learning} embeds a Lie group in a higher-dimensional space with Lie group constraints, enabling training with neural ODEs on Euclidean space. Our work belongs to the extrinsic approach by enforcing matrix Lie group constraints on its embedding space $\bbR^{n\times n}$, but focuses on incorporating the law of energy conservation, via a Hamiltonian formulation, in the dynamics model and providing control design for trajectory tracking. Meanwhile, the intrinsic approach \cite{lou2020neural, wotte2024optimal} develops adjoint methods for training using local coordinates on the Lie algebra, e.g. via local charts \cite{lou2020neural}, and therefore, guarantees Lie groups constraints by design.} \NEWW{Concurrently to our work, Wotte et al. \cite{wotte2024optimal} develop a neural ODE network on Lie group to learn Hamiltonian dynamics on the $SE(3)$ manifold by~deriving an adjoint method on the Lie algebra, offering an exciting approach for learning structure-preserving dynamics model.}

\NEWW{While few dynamics learning papers consider control design based on the learned model, we develop a general trajectory tracking controller for Lie group Hamiltonian dynamics}. The Hamiltonian formulation and its port-Hamiltonian generalization \cite{van2014port} are built around the notion of system energy and, hence, are naturally related to control techniques for stabilization aiming to minimize the total energy. Since the minimum point of the Hamiltonian might not correspond to a desired regulation point, the control design needs to inject additional energy to ensure that the minimum of the total energy is at the desired equilibrium. For fully-actuated \mbox{\NEWW{port-Hamiltonian}} systems, it is sufficient to shape the potential energy only using an energy-shaping and damping-injection (ES-DI) controller \cite{van2014port}. For underactuated systems, both the kinetic and potential energies needs to be shaped, e.g., via interconnection and damping assignment passivity-based control (IDA-PBC) \cite{van2014port,ortega2002stabilization,acosta2014robust,cieza2019ida}. Wang and Goldsmith \cite{wang2008modified} extend the IDA-PBC controller from stabilization to trajectory tracking. Closely related to our controller design, Souza et. al. \cite{souza2014passivity} apply this technique to design a controller for an underactuated quadrotor robot but use Euler angles as the orientation representation. Port-Hamiltonian structure and energy-based control design are also used to learn distributed control policies from state-control trajectories \cite{furieri2022distributed, galimberti2023hamiltonian, sebastian2023lemurs}.

We connect Hamiltonian-dynamics learning with the idea of IDA-PBC control to allow stabilization of any rigid-body robot without relying on its model parameters a priori. We design a trajectory-tracking controller for underactuated systems, \NEW{e.g., quadrotor robots}, based on the IDA-PBC approach and show how to construct desired pose and momentum trajectories given only desired position and yaw. We demonstrate the tight integration of dynamics learning and control to achieve closed-loop trajectory tracking with underactuated quadrotor robots.

%% file: tex/ProblemFormulation.tex
\section{Problem Statement}
\label{sec:problem_statement}

Consider a robot with state $\bfx$ consisting of generalized coordinates $\frakq$ evolving on a Lie group $\sfG$ and generalized velocity $\bfxi$ on the Lie algebra $\frakg$ of $\sfG$. Let $\dot{\bfx} = \bff(\bfx, \bfu)$ characterize the robot dynamics with control input $\bfu \in \bbR^m$. For example, the state of rigid-body mobile robot, such as a UGV or UAV, may be modeled by its pose on the $SE(3)$ group, consisting of position and orientation, and its twist on the $\mathfrak{se}(3)$ Lie algebra, consisting of linear and angular velocity. The control input of an Ackermann-drive UGV may include its linear acceleration and steering angle rate, and that of a quadrotor UAV may include the total thrust and moment generated by the propellers. See Sec.~\ref{subsec:ham_dyn_lie_group} for more details.

We assume that the function $\bff$ specifying the robot dynamics is unknown and aim to approximate it using a dataset $\mathcal{D}$ of state and control trajectories. Specifically, let $\calD = \{t_{0:N}^{(i)}, \bfx^{(i)}_{0:N}, \bfu^{(i)}\}_{i=1}^D$ consist of $D$ state sequences $\bfx^{(i)}_{0:N}$, obtained by applying a constant control input $\bfu^{(i)}$ to the system with initial condition $\bfx^{(i)}_0$ at time $t_0^{(i)}$ and sampling its state $\bfx^{(i)}(t_n^{(i)}) =: \bfx^{(i)}_n$ at times $t_0^{(i)} < t_1^{(i)} < \ldots < t_N^{(i)}$. Using the dataset $\mathcal{D}$, we aim to find a function $\bar{\bff}_{\bftheta}$ with parameters $\bftheta$ that approximates the true dynamics $\bff$ well. To optimize $\bftheta$, we roll out the approximate dynamics $\bar{\bff}_{\bftheta}$ with initial state $\bfx_0^{(i)}$ and constant control $\bfu^{(i)}$ and minimize the discrepancy between the computed state sequence $\bar{\bfx}^{(i)}_{1:N}$ and the true state sequence $\bfx^{(i)}_{1:N}$ in $\calD$. 

\begin{problem}
\label{problem:dynamics_learning}
Given a dataset $\calD = \{t_{0:N}^{(i)}, \bfx^{(i)}_{0:N}, \bfu^{(i)}\}_{i=1}^D$ and a function $\bar{\bff}_{\bftheta}$, find the parameters $\bftheta$ that minimize:
\begin{equation}
\label{problem_formulation_unknown_env_equation}
\begin{aligned}
\min_{\bftheta} \;&\sum_{i = 1}^D \sum_{n = 1}^N \ell(\bfx^{(i)}_n,\bar{\bfx}^{(i)}_n)\\
\text{s.t.} \;\; & \dot{\bar{\bfx}}^{(i)}(t) = \bar{\bff}_{\bftheta}(\bar{\bfx}^{(i)}(t), \bfu^{(i)}), \;\;\bar{\bfx}^{(i)}(t_0) = \bfx^{(i)}_0,\\
& \bar{\bfx}^{(i)}_n = \bar{\bfx}^{(i)}(t_n), \;\;\forall n = 1, \ldots, N,\;\;\forall i = 1, \ldots, D,
\end{aligned}
\end{equation}
where $\ell$ is a distance metric on the state space.
\end{problem}

Further, we aim to design a feedback controller capable of tracking a desired state trajectory $\bfx^*(t)$, $t \geq t_0$, for the learned model $\bar{\bff}_{\bftheta}$ of the robot dynamics.

\begin{problem}
\label{problem:setpoint_reg}
Given an initial condition $\bfx_0$ at time $t_0$, desired state trajectory $\bfx^*(t)$, $t \geq t_0$, and learned dynamics $\bar{\bff}_{\bftheta}$, design a feedback control law $\bfu = \bfpi(\bfx, \bftheta, \bfx^*(t))$ such that $\limsup_{t \to \infty} \ell(\bfx(t),\bfx^*(t))$ is bounded.
\end{problem}

We consider robot kinematics on the Lie group $\sfG$ such that when there is no control input, $\bfu = \bf0$, the dynamics $\bff(\bfx, \bfu)$ respect the law of energy conservation. We embed these constraints in the structure of the parametric function $\bar{\bff}_{\bftheta}$. We review matrix Lie groups, with the $SE(3)$ manifold as an example, and Hamiltonian dynamics equations next.

%% file: tex/Prelim.tex
\section{Preliminaries}
\label{sec:prelim}

\subsection{Matrix Lie Groups}
In this section, we cover the background needed to define Hamiltonian dynamics on a Lie group. Please refer to \cite{hall2013lie, lee2017global, marsden2013introduction} for a more detailed overview of matrix Lie groups.

\begin{definition}[Dot Product]\label{def:dot_product}
    The dot product $\langle \cdot, \cdot \rangle$ between two matrices $\bfxi$ and $\bfpsi$ in $\bfR^{n\times m}$ is \NEWW{can be chosen} as:
    \begin{equation}
        \langle\bfxi, \bfpsi \rangle = \tr(\bfxi^\top\bfpsi).
    \end{equation}
\end{definition}

\NEWW{The dot product definition above is used to define the dual maps in Def. \ref{def:left_translation} and \ref{def:coadjoint}, and loss functions in Sec. \ref{subsec:training} and \ref{subsec:SE3_dyn_learning}.}

\begin{definition}[General Linear Group \cite{hall2013lie}]
    The general linear group $\mathsf{GL}(n,\bbR)$ is the group of $n\times n$ invertible real matrices.
\end{definition}

\begin{definition}[Matrix Lie Group \cite{hall2013lie}] 
    A matrix Lie group $\sfG$ is a subgroup of $\mathsf{GL}(n,\bbR)$ with identity element $\bfe$ such that if any sequence of matrices $\{A_n\}_{n=0}^\infty$ in $\sfG$ converges to a matrix $A$, then either $A$ is in $\sfG$ or $A$ is not invertible. A matrix Lie group is also a smooth embedded submanifold on $\bbR^{n\times n}$.
\end{definition}

\begin{definition}[Tangent Space and Bundle]
    The tangent space $\sfT_\frakq\sfG$ is the set of all tangent vectors $\bfxi$ to the manifold $\sfG$ at $\frakq$. The tangent bundle $\sfT\sfG$ is the set of all the pairs $(\frakq, \bfxi)$ with $\frakq \in \sfG$ and $\bfxi \in \sfT_\frakq\sfG$.
\end{definition}

\begin{definition}[Lie Algebra and Lie Bracket]
    A Lie algebra is a vector space $\frakg$, equipped with a Lie bracket operator $[\cdot,\cdot]: \frakg \times \frakg \rightarrow \frakg$ that satisfies:
    \begin{equation*}
        \begin{aligned}
            &\text{bilinearity:}\phantom{-} [a\bfxi_1 + b\bfxi_2, \bfxi_3] = a [\bfxi_1, \bfxi_3] + b [\bfxi_2, \bfxi_3], \\
            &\phantom{bilinearity:} [\bfxi_3, a\bfxi_1 + b\bfxi_2] = a [\bfxi_3, \bfxi_1] + b [\bfxi_3, \bfxi_2],\\
            &\text{skew-symmetry:}\phantom{-} [\bfxi_1, \bfxi_2] = -[\bfxi_2, \bfxi_1],\\
            &\text{Jacobi identity:} \\
            &\phantom{bilinearity:} [\bfxi_1, [\bfxi_2, \bfxi_3]] + [\bfxi_2, [\bfxi_3, \bfxi_1]] + [\bfxi_3, [\bfxi_1, \bfxi_2]] = 0.
        \end{aligned}
    \end{equation*}
\end{definition}


Every matrix Lie group $\sfG$ is associated with a Lie algebra $\frakg$, which is the tangent space at the identity element $\sfT_\bfe \sfG$. An element $\frakq \in \sfG$ is linked with an element $\bfxi \in \frakg$ via the exponential map $\exp_{\sfG} : \frakg \rightarrow \sfG$ and the logarithm map $\log_{\sfG} : \sfG \rightarrow \frakg$ \cite{hall2013lie}. Since tangent spaces of $\sfG$, and in particular the Lie algebra $\frakg$, are isomorphic to Euclidean space, we can define a linear mapping $(\cdot)^{\wedge}: \bbR^n \rightarrow \frakg$ and its inverse $(\cdot)^\vee: \sfg \rightarrow \bbR^n$, where $n$ is the dimension of $\sfG$. Thus, we can map between $\sfG$ and $\bbR^n$ using the compositions:
\begin{equation}
    \exp^\wedge_\sfG = \exp_\sfG \circ \;^\wedge, \quad \log_\sfG^\vee = \;^\vee \circ \log_\sfG.
\end{equation}

\NEWW{In this paper, we consider a matrix Lie group element $\frakq \in \sfG$ embedded in $\bbR^{n\times n}$, instead of its $n$-dimensional representations, e.g., $\log_\sfG^\vee \frakq$, due to potential issues of discontinuity \cite{zhou2019continuity} and singularity \cite{BarfootBook, hemingway2018perspectives}. For example, Zhou. et al. \cite{zhou2019continuity} show that higher-dimensional representations of rotations, e.g., $(n^2 -n)$ dimensions for $SO(n)$, are more suitable for learning using neural networks because they ensure continuity.}

\begin{definition} [Left Translation and Invariant Vectors] \label{def:left_translation}
    The left translation $\sfL_{\frakq}: \sfG \rightarrow \sfG$ with $\frakq \in \sfG$ is defined as: 
    \begin{equation}
        \sfL_\frakq(\bfh) = \frakq \bfh.
    \end{equation}
    The left-invariant vector $\sfT_\bfe \sfL_{\frakq}(\bfxi)$ is defined as the derivative of the left translation $\sfL_{\frakq}$ at $\bfh=\bfe$ in the direction of $\bfxi$.
    This vector describes the kinematics of the Lie group, which relates the velocity $\bfxi \in \frakg$ to the change $\dot{\frakq} \in \sfT_\frakq\sfG$ of coordinates $\frakq$:
    \begin{equation}
        \dot{\frakq} = \sfT_\bfe\sfL_\frakq (\bfxi) = \frakq \bfxi.
    \end{equation}
\end{definition}
\NEWW{Given a pairing $\langle \cdot,\cdot\rangle$ on $\frakg^* \times \frakg$ (e.g., Def.~\ref{def:dot_product}), the dual map $\sfT_\bfe^*\sfL_\frakq$ of $\sfT_\bfe\sfL_\frakq$ satisfies
\begin{equation} \label{eq:dual_TeLq}
     \langle\sfT_\bfe^*\sfL_\frakq (\bfeta), \bfxi\rangle = \langle\bfeta, \sfT_\bfe\sfL_\frakq (\bfxi)\rangle,
\end{equation}
for any $\bfeta \in \frakg^*$ and $\bfxi \in \frakg$.}

\begin{definition}[Adjoint Operator]
    \NEWW{For $\frakq \in \sfG$}, the adjoint $\sfAd_\frakq: \frakg \rightarrow \frakg$ is defined as:
    \begin{equation}
        \sfAd_\frakq(\bfpsi) = \frakq \bfpsi \frakq^{-1}.
    \end{equation}
    The algebra adjoint $\sfad_{\bfxi}: \frakg \rightarrow \frakg$ is the directional derivative of $\sfAd_\frakq$ at $\frakq = \bfe$ in the direction of $\bfxi \in \frakg$:
    \begin{equation}
        \sfad_{\bfxi}(\bfpsi) = \left.{\frac{d}{dt}\sfAd_{\exp_{\sfG}({t\bfxi})}(\bfpsi)}\right|_{t=0} = [\bfxi, \bfpsi].
    \end{equation}
\end{definition}

\begin{definition}[Cotangent Space and Bundle]
    The dual space of the tangent space $\sfT_\frakq\sfG$, i.e., the space of all linear functionals from $\sfT_\frakq\sfG$ to $\bbR$, is called the cotangent space $\sfT^*_\frakq\sfG$. At the identity $\bfe$, the cotangent space of the Lie algebra $\frakg = \sfT_\bfe\sfG$ is denoted $\frakg^*$. The cotangent bundle $\sfT^*\sfG$ is the set of all the pairs $(\frakq, \bf\frakp)$ with $\frakq \in \sfG$ and $\frakp \in \sfT^*_\frakq\sfG$.
\end{definition}

\begin{definition}[Coadjoint Operator] \label{def:coadjoint}
    \NEWW{For $\frakq \in \sfG$}, the coadjoint $\sfAd^*_\frakq: \frakg^* \rightarrow \frakg^*$ is defined as $\langle \sfAd^*_\frakq(\bfvarphi), \bfpsi\rangle = \langle \bfvarphi, \sfAd_\frakq(\bfpsi)\rangle$, \NEWW{where $\bfvarphi \in \frakg^*$, $\bfpsi \in \frakg$ and $\langle \cdot,\cdot\rangle$ is a pairing on $\frakg^* \times \frakg$}. The algebra coadjoint $\sfad^*_{\bfxi}: \frakg^* \rightarrow \frakg^*$ is the dual map of $\sfad_{\bfxi}$, satisfying $\langle \sfad^*_{\bfxi}(\bfvarphi), \bfpsi\rangle = \langle \bfvarphi, \sfad_{\bfxi}(\bfpsi)\rangle$.
\end{definition}

 The next section describes the $SE(3)$ Lie group to illustrate the definitions above. The $SE(3)$ Lie group is used to represent the position and orientation of a rigid body.

\subsection{Example: SE(3) Manifold}
\label{subsec:SO3_SE3_kinematics}

Consider a fixed world inertial frame of reference and a rigid body with a body-fixed frame attached to its center of mass. The pose of the body-fixed frame in the world frame is determined by the position $\bfp = [x,y,z]^\top \in \bbR^3$ of the center of mass and the orientation of the body-fixed frame's coordinate axes:
\begin{equation}
\label{eq:rotmat_def}
\bfR = \begin{bmatrix}
\bfr_1 & \bfr_2 & \bfr_3
\end{bmatrix}^\top \in SO(3),
\end{equation}
where $\bfr_1, \bfr_2, \bfr_3 \in \mathbb{R}^3$ are the rows of the rotation matrix $\bfR$. A rotation matrix is an element of the special orthogonal group:
\begin{equation}
\label{eq:SO3_def}
SO(3) = \left\{\bfR \in \bbR^{3\times 3} : \bfR^\top \bfR = \bfI, \det(\bfR) = 1\right\}.
\end{equation}
The rigid-body position and orientation can be combined in a single pose matrix $\frakq\in SE(3)$, which is an element of the special Euclidean group:
\begin{equation}
\label{eq:SE3_def}
SE(3) = \left\{ \begin{bmatrix}
\bfR & \bfp \\
\bf0^\top & 1
\end{bmatrix} \in \mathbb{R}^{4\times 4}: \bfR \in SO(3), \bfp \in \mathbb{R}^3\right\}.
\end{equation}
The kinematic equations of motion of the rigid body are determined by the linear velocity $\bfv \in \bbR^3$ and angular velocity $\bfomega \in \bbR^3$ of the body-fixed frame with respect to the world frame, expressed in body-frame coordinates. The generalized velocity $\bfzeta = [\bfv^\top\!,\, \bfomega^\top]^\top \in \mathbb{R}^6$ determines the rate of change of the rigid-body pose according to the $SE(3)$ kinematics:
\begin{equation}
\label{eq:pose_kinematics}
\dot{\frakq} = \frakq \bfxi = \frakq\hat{\bfzeta} =: \frakq\begin{bmatrix}
\hat{\bfomega} & \bfv\\
\bf0^\top & 0
\end{bmatrix},
\end{equation}
where we overload $\hat{\cdot}$ to denote the mapping from a vector $\bfzeta \in \bbR^6$ to a $4 \times 4$ twist matrix $\bfxi = \hat{\bfzeta}$ in the Lie algebra $\mathfrak{se}(3)$ of $SE(3)$ and from a vector $\bfomega \in \bbR^3$ to a $3 \times 3$ skew-symmetric matrix $\hat{\bfomega}$ in the Lie algebra $\mathfrak{so}(3)$ of $SO(3)$:
\begin{equation}
\label{eq:hatmap}
\hat{\bfomega} = \begin{bmatrix}
0 & -\omega_3 & \omega_2 \\
\omega_3 & 0 & -\omega_1 \\
-\omega_2 & \omega_1 & 0 
\end{bmatrix}.
\end{equation}
Please refer to \cite{BarfootBook} for an excellent introduction to the use of $SE(3)$ in robot state estimation problems.

\subsection{Hamiltonian Dynamics on Matrix Lie Groups}
\label{subsec:ham_dyn_lie_group}

In this section, we describe Hamilton's equations of motion on a matrix Lie group \cite{marsden2013introduction,lee2017global}. Our neural network architecture design in Sec. \ref{sec:data_gen_net_design} is based on Lie group Hamiltonian dynamics, which \NEWW{encode} both kinematic constraints and energy conservation. 

Consider a system with generalized coordinates $\frakq$ in a matrix Lie group $\sfG$ and generalized velocity $\dot{\frakq} \in \sfT_\frakq\sfG$. The dynamics of the state $\bfx = (\frakq, \dot{\frakq}) \in \sfT\sfG$ satisfy:
\begin{equation} \label{eq:kinematics_lie_group}
    \dot{\frakq} = \sfT_\bfe\sfL_\frakq(\bfxi) = \frakq \bfxi,
\end{equation}
where $\bfxi$ is a element in the Lie algebra $\frakg$. 

The Lagrangian on a Lie group $\calL : \sfG \times \frakg \rightarrow \bbR$ is defined as the difference between the kinetic energy $\calT : \sfG \times \frakg \rightarrow \bbR$ and the potential energy $\calV : \sfG \rightarrow \bbR$:
\begin{equation}
    \calL(\frakq, \bfxi) = \calT(\frakq, \bfxi) - \calV(\frakq).
\end{equation}
The Hamiltonian is obtained using a Legendre transformation:
\begin{equation} \label{eq:ham_legendre}
    \calH(\frakq, \frakp) = \frakp \cdot \bfxi - \calL(\frakq, \bfxi),
\end{equation}
where the momentum $\frakp$ is defined as:
\begin{equation} \label{eq:p_lie_alg}
    \frakp = \frac{\partial \calL (\frakq, \bfxi)}{\partial \bfxi}.
\end{equation}

The state $(\frakq, \frakp) \in \sfT^*\sfG$ evolves according to the Hamiltonian dynamics \cite{lee2017global} as:
\begin{subequations}\label{eq:ham_dyn_lie_group}
    \begin{align}
        \dot{\frakq} &= \sfT_\bfe\sfL_\frakq\left(\frac{\partial \calH (\frakq, \frakp)}{\partial \frakp}\right), \label{eq:ham_dyn_lie_group_q}\\
        \dot{\frakp} &= \sfa\sfd^*_{\bfxi} (\frakp) - \sfT_\bfe^*\sfL_\frakq\left(\frac{\partial \calH (\frakq, \frakp)}{\partial \frakq}\right) + \bfB(\frakq)\bfu.\label{eq:ham_dyn_lie_group_p}
    \end{align}
\end{subequations}

\NEWW{Obtaining explicit expressions for $\sfa\sfd^*_{\bfxi} (\frakp)$ and $\sfT_\bfe^*\sfL_\frakq\left(\bfeta\right)$ with $\frakq \in \sfG$, $\bfxi \in \frakg$, and $\frakp,\bfeta \in \frakg^*$ depends on the structure of the matrix Lie algebra $\frakg$ and the pairing $\langle \cdot,\cdot\rangle$ on $\frakg^* \times \frakg$. Appendix \ref{subsec:derivation_dual_map} provides details and an example for the $SE(3)$ manifold.}

By comparing \eqref{eq:kinematics_lie_group} and \eqref{eq:ham_dyn_lie_group}, we have:
\begin{equation} \label{eq:xi_ham}
    \bfxi = \frac{\partial \calH (\frakq, \frakp)}{\partial \frakp}.
\end{equation}
Let $\bfeta = \frac{\partial \calH (\frakq, \frakp)}{\partial \frakq}$. When there is no control input, i.e., $\bfu = 0$, the conservation of energy is guaranteed as:
\begin{align} \label{eq:dH_dt}
    \frac{d \calH(\frakq, \frakp)}{dt} &= \langle \bfeta, \dot{\frakq}\rangle + \langle\bfxi, \dot{\frakp}\rangle,\notag\\
    &= \langle\bfeta, \sfT_\bfe\sfL_\frakq\left(\bfxi\right)\rangle - \langle\bfxi, \sfT_\bfe^*\sfL_\frakq\left(\bfeta\right)\rangle + \langle\bfxi, \sfa\sfd^*_{\bfxi}(\frakp)\rangle,\notag\\
    &= 0,
\end{align}
because of Eq. \eqref{eq:dual_TeLq} and, by definition,
\begin{equation}
    \langle\bfxi, \sfa\sfd^*_{\bfxi}(\frakp)\rangle = \langle\sfa\sfd_{\bfxi}(\bfxi), \frakp\rangle = \langle[\bfxi, \bfxi], \frakp\rangle = 0.
\end{equation}

\subsection{Reformulation as Port-Hamiltonian Dynamics} \label{subsec:port_ham_dyn}

The notion of energy in dynamical systems is shared across multiple domains, including mechanical, electrical, and thermal. A port-Hamiltonian generalization \cite{van2014port} of Hamiltonian mechanics is used to model systems with energy-storing elements (e.g., kinetic and potential energy), energy-dissipating elements (e.g., friction or resistance), and external energy sources (e.g., control inputs), connected via energy ports. An input-state-output port-Hamiltonian system has the form:
\begin{equation}
\label{eq:port_Hal_dyn}
\begin{bmatrix}
\dot{\frakq} \\
\dot{\frakp} \\
\end{bmatrix}
= (\mathbf\calJ(\frakq, \frakp) - \mathbf\calR(\frakq, \frakp))
\begin{bmatrix}
\frac{\partial \mathcal{H}}{\partial\frakq} \\
\frac{\partial \mathcal{H}}{\partial\frakp} 
\end{bmatrix} + \mathbf\calG(\frakq, \frakp)\bfu,
\end{equation}
where $\mathbf\calJ(\frakq, \frakp)$ is a skew-symmetric interconnection matrix, representing the energy-storing elements, $\mathbf\calR(\frakq, \frakp) \succeq 0$ is a positive semi-definite dissipation matrix, representing the energy-dissipating elements, and $\calG(\frakq, \frakp)$ is an input matrix such that $\mathbf\calG(\frakq, \frakp)\bfu$ represents the external energy sources. In the absence of energy-dissipating elements and external energy sources, the skew-symmetry of $\mathbf\calJ(\frakq, \frakp)$ guarantees the energy conservation of the system.

To model energy dissipating elements such as friction or drag forces, we reformulate the Hamiltonian dynamics on a matrix Lie group \eqref{eq:ham_dyn_lie_group} in port-Hamiltonian form \eqref{eq:port_Hal_dyn}.
Such elements are often modeled \cite{faessler2017differential} as a linear transformation $\bfD(\frakq, \frakp) \succeq 0$ of the velocity $\bfxi$ and only affect the generalized momenta $\frakp$, i.e.,
\begin{equation}
\bf\mathcal{R}(\frakq, \frakp) = \begin{bmatrix}
\bf0 & \bf0\\
\bf0 & \bfD(\frakq, \frakp)
\end{bmatrix}.
\end{equation}
The Hamiltonian dynamics \eqref{eq:ham_dyn_lie_group} is a special case of \eqref{eq:port_Hal_dyn}, where the dissipation matrix is $\bfD(\frakq, \frakp) = \bf0$, the input matrix is $\calG(\frakq, \frakp) = \begin{bmatrix}\bf0^\top & \bfB(\frakq)^\top\end{bmatrix}^\top$ and the interconnection matrix  $\calJ(\frakq, \frakp)$ can be obtained by rearranging \eqref{eq:ham_dyn_lie_group} with $\bfu = 0$ and is guaranteed to be skew-symmetric due to the energy conservation \eqref{eq:dH_dt}. \NEWW{In an implementation, the coordinates $\frakq$ and momentum $\frakp$ may be represented as vectors in $\bbR^{n^2}$, leading to an interconnection matrix $\calJ(\frakq, \frakp) \in \bbR^{2n^2 \times 2n^2}$.}

\subsection{Example: Hamiltonian Dynamics on the SE(3) Manifold}
\label{subsec:ham_dyn_se3}

In this section, we consider the generalized coordinate $\frakq$ of a mobile robot consisting of its position $\bfp \in \mathbb{R}^3$ and orientation $\bfR\in SO(3)$. Let $\frakq = (\bfp, \bfR)$ be the generalized coordinates and $\bfzeta = (\bfv, \bfomega) \in \bbR^6$ be the generalized velocity, consisting of the body-frame linear velocity $\bfv \in \bbR^3$ and the body-frame angular velocity $\bfomega\in \mathbb{R}^3$. The coordinate $\frakq$ evolves on the Lie group $SE(3)$ while the generalized velocity satisfies $\dot{\frakq} = \frakq\bfxi$, where $\bfxi = \hat{\bfzeta}$ is a twist matrix in $\frakse(3)$, as shown in Eq. \eqref{eq:pose_kinematics}.

The isomorphism between $\frakse(3)$ and $\bbR^6$ via \eqref{eq:pose_kinematics} simplifies the Hamiltonian \eqref{eq:ham_dyn_lie_group} and its port-Hamiltonian formulation \eqref{eq:port_Hal_dyn} as follows. The Lagrangian function on $SE(3)$ can be expressed in terms of $\frakq$ and $\bfzeta$, instead of $\frakq$ and $\bfxi$:
\begin{equation} \label{eq:lagrangian_angvel_linvel}
\mathcal{L}(\frakq, \bfzeta) = \frac{1}{2} \bfzeta^\top \bfM(\frakq)\bfzeta -\calV(\frakq).
\end{equation}
%
The generalized mass matrix has a block-diagonal form when the body frame is attached to the center of mass \cite{lee2017global}:
\begin{equation}
\label{eq:mass_matrix_hamiltonian}
\bfM(\frakq) = \begin{bmatrix}
\bfM_\bfv(\frakq) & \bf0 \\
\bf0 & \bfM_{\bfomega}(\frakq)
\end{bmatrix} \in \bbS_{\succ0}^{6 \times 6},
\end{equation}
where $\bfM_\bfv(\frakq), \bfM_{\bfomega}(\frakq) \in \bbS_{\succ0}^{3 \times 3}$.
The generalized momenta are defined, as before, via the partial derivative of the Lagrangian with respect to the twist:
\begin{equation}
\label{eq:momenta_Mtwist}
\frakp = \begin{bmatrix} {\frakp}_{\bfv} \\ {\frakp}_{\bfomega} \end{bmatrix} = \frac{\partial \mathcal{L}(\frakq, \bfzeta)}{\partial \bfzeta} = \bfM(\frakq)\bfzeta \in \mathbb{R}^6.
\end{equation}
%
The Hamiltonian function of the system becomes:
\begin{equation} \label{eq:hamiltonian_se3}
    \calH(\frakq, \frakp) = \frakp \cdot \bfzeta - \calL(\frakq, \bfzeta) = \frac{1}{2} \frakp^\top \bfM^{-1}(\frakq)\frakp + \calV(\frakq).
\end{equation}

By vectorizing the generalized coordinates $\frakq = [\bfp^\top \quad \bfr_1^\top \quad \bfr_2^\top \quad \bfr_3^\top]^\top$, the Hamiltonian dynamics on $SE(3)$ can be described in port-Hamiltonian form \eqref{eq:port_Hal_dyn} \cite{lee2017global, forni2015port, rashad2019port} with interconnection matrix:
\begin{equation}
\label{eq:SE3_PH_J}
\bf\mathcal{J}(\frakq, \frakp) = \begin{bmatrix}
\bf0 & \frakq^{\times} \\
-\frakq^{\times\top} & \frakp^{\times} 
\end{bmatrix}, \qquad \frakp^{\times} = \begin{bmatrix}
\bf0 & \hat{\frakp}_{\bfv}\\
\hat{\frakp}_{\bfv} & \hat{\frakp}_{\bfomega}
\end{bmatrix},
\end{equation}
and input matrix $\qquad \mathcal{G}(\frakq, \frakp) = \begin{bmatrix} \bf0^\top & \bfB(\frakq)^\top \end{bmatrix}^\top$,
where $\frakq^{\times} = \begin{bmatrix}
\bfR^\top\!\!\!\! & \bf0 & \bf0 & \bf0 \\
\bf0 & \hat{\bfr}_1^\top & \hat{\bfr}_2^\top & \hat{\bfr}_3^\top
\end{bmatrix}^\top$. \NEWW{Note that the kinematics constraints $\dot{\frakq} = \frakq \bfxi$ of the coordinates $\frakq$ on $SE(3)$, which guarantee that the coordinates $\frakq$ remain in $SE(3)$, are not affected by vectorization}. The port-Hamiltonian formulation allows us to model dissipation elements in the dynamics by the dissipation matrix:
\begin{equation}
\label{eq:dissipation_matrix_hamiltonian}
\bfD(\frakq, \frakp) = \begin{bmatrix}
\bfD_\bfv(\frakq, \frakp) & \bf0 \\
\bf0 & \bfD_{\bfomega}(\frakq, \frakp)
\end{bmatrix} \in \bbS_{\succ0}^{6 \times 6},
\end{equation}
where the components $\bfD_\bfv(\frakq, \frakp)$ and $\bfD_{\bfomega}(\frakq, \frakp)$ correspond to $\frakp_\bfv$ and $\frakp_{\bfomega}$, respectively. 
The equations of motions on the $SE(3)$ manifold are written in port-Hamiltonian form as:
\begin{subequations} \label{eq:portham_dyn_SE3}
\begin{align}
    \;\dot{\bfp} &=\;\;\bfR\frac{\partial{\mathcal{H}(\frakq, \frakp)}}{\partial \frakp_{\bfv}}, \label{eq:ham_se3_pos_dot}\\
    \dot{\bfr_i} &=\;\;\bfr_i \times \frac{\partial{\mathcal{H}(\frakq, \frakp)}}{\partial \frakp_{\bfomega}}, \quad i = 1,2,3 \label{eq:ham_se3_rot_dot}\\
    \dot{\frakp}_{\bfv} &=\;\;\frakp_{\bfv}\times \frac{\partial{\mathcal{H}(\frakq, \frakp)}}{\partial \frakp_{\bfomega}} - \bfR^\top \frac{\partial{\mathcal{H}(\frakq, \frakp)}}{\partial \bfp} \label{eq:ham_se3_pv_dot}\\
    & \qquad -\;\bfD_\bfv(\frakq, \frakp) \frac{\partial{\mathcal{H}(\frakq, \frakp)}}{\partial \frakp_{\bfv}} +  \bfb_{\bfv}(\frakq)\bfu, \notag \\
    \dot{\frakp}_{\bfomega} &=\;\;\frakp_{\bfomega} \times \frac{\partial{\mathcal{H}(\frakq, \frakp)}}{\partial \frakp_{\bfomega}} + \frakp_{\bfv}\times \frac{\partial{\mathcal{H}(\frakq, \frakp)}}{\partial \frakp_{\bfv}} +  \label{eq:ham_se3_pw_dot}\\
     &\quad\sum_{i = 1}^3 \bfr_i \times \frac{\partial{\mathcal{H}(\frakq, \frakp)}}{\partial \bfr_i} - \bfD_{\bfomega}(\frakq, \frakp) \frac{\partial{\mathcal{H}(\frakq, \frakp)}}{\partial \frakp_{\bfomega}} + \bfb_{\bfomega}(\frakq)\bfu, \notag 
\end{align}
\end{subequations}
where the input matrix is $\bfB(\frakq) = \begin{bmatrix} \bfb_{\bfv}(\frakq)^\top & \bfb_{\bfomega}(\frakq)^\top \end{bmatrix}^\top$.

\subsection{Neural ODE Networks}
\label{subsec:neuralode}

In this section, we briefly describe neural ODE networks \cite{chen2018neural}, which approximate the closed-loop dynamics $\dot{\bfx} = \bff(\bfx,\bfpi(\bfx))$ of a system for some unknown control policy $\bfu = \bfpi(\bfx)$ by a neural network $\bar{\bff}_{\bftheta}(\bfx)$. The parameters of $\bar{\bff}_{\bftheta}(\bfx)$ are trained using a dataset $\calD = \{t^{(i)}_{0:N}, \bfx_{0:N}^{(i)}\}_i$ of state trajectory samples $\bfx_n^{(i)} = \bfx^{(i)}(t_n^{(i)})$ via forward and backward passes through a differentiable ODE solver, where the backward passes provide the gradient of the loss function. Given an initial state $\bfx_0^{(i)}$ at time $t_0^{(i)}$, a forward pass returns predicted states at times $t_1^{(i)}, \ldots, t_N^{(i)}$:
\begin{equation}
\label{eq:ode_solver}
\{\bar{\bfx}_1^{(i)}, \ldots, \bar{\bfx}_N^{(i)}\} = \text{ODESolver}(\bfx_0^{(i)}, \bar{\bff}_{\bftheta}, t_1^{(i)}, \ldots, t_N^{(i)}).
\end{equation}
The gradient of a loss function, $\sum_{i=1}^D\sum_{j=1}^N \ell(\bfx_j^{(i)},\bar{\bfx}_j^{(i)})$, is back-propagated by solving another ODE with adjoint states. The parameters $\bftheta$ are updated by gradient descent to minimize the loss.
For physical systems, Zhong et al. \cite{zhong2020symplectic} extends the neural ODE by integrating the Hamiltonian dynamics on $\bbR^n$ into the neural network model $\bar{\bff}_{\bftheta}(\bfx)$, and consider zero-order hold control input $\bfu$, leading to a neural ODE network with the following approximated dynamics:
\begin{equation}
\label{eq:gen_dyn_state_input}
\begin{bmatrix}
\dot{\bfx}\\
\dot{\bfu}
\end{bmatrix} = \begin{bmatrix}
\bar{\bff}_{\bftheta}(\bfx, \bfu)\\
\bf0
\end{bmatrix}.
\end{equation} 

\NEWW{Recently, neural ODE networks have been extended from Euclidean space to Lie groups \cite{falorsi2020neural, elamvazhuthi2023learning, lou2020neural, wotte2024optimal}, guaranteeing that the Lie group constraints are satisfied by the predicted states by design. While it is possible to train our Hamiltonian dynamics model using a Lie group neural ODE network, we leave this investigation for future work due to the lack of suitable open-source software for Lie group ODE integration.}

%% file: tex/TechnicalApproach.tex
\section{Learning Lie Group Hamiltonian Dynamics}
\label{sec:data_gen_net_design}

We consider a Hamiltonian system with unknown kinetic energy $\calT(\frakq)$, potential energy $\calV(\frakq)$, input matrix $\bfB(\frakq)$, dissipation matrix $\bfD(\frakq, \frakp)$, and design a structured neural ODE network to learn these terms from state-control trajectories.

\subsection{Data Collection}
\label{subsubsec:data_gen}

We collect a data set $\calD = \{t_{0:N}^{(i)}, \bfx^{(i)}_{0:N}, \bfu^{(i)}\}_{i=1}^D$ consisting of state sequences $\bfx^{(i)}_{0:N}$, where $\bfx_n^{(i)}~=~[\frakq_n^{(i)\top} \quad \bfxi_n^{(i)\top}]^\top$ for $n = 0, \ldots, N$. Such data are generated by applying a constant control input $\bfu^{(i)}$ to the system and sampling the state $\bfx_{n}^{(i)} = \bfx^{(i)}(t_n^{(i)})$ at times $t_n^{(i)}$ for $n = 0, \ldots, N$. The generalized coordinates $\frakq$ and velocity $\bfxi$ may be obtained from a state estimation algorithm, such as odometry algorithm for mobile robots \cite{vio_benchmark,OdometrySurvey}, or from a motion capture system. In physics-based simulation the data can be generated by applying random control inputs $\bfu^{(i)}$. In real-world applications, where safety is a concern, data may be collected by a human operator manually controlling the robot.

\subsection{Model Architecture}

Since robots are physical systems, their dynamics $\bff(\bfx,\bfu)$ satisfy the Hamiltonian formulation (Sec. \ref{subsec:ham_dyn_lie_group}). To learn the dynamics $\bff(\bfx,\bfu)$ from a trajectory dataset $\calD$, we design a neural ODE network (Sec. \ref{subsec:neuralode}), approximating the dynamics via a parametric function $\bar{\bff}_{\bftheta}(\bfx,\bfu)$ based on Eq. \eqref{eq:ham_dyn_lie_group}. 

\begin{figure}[t]
\centering
\includegraphics[width=\linewidth]{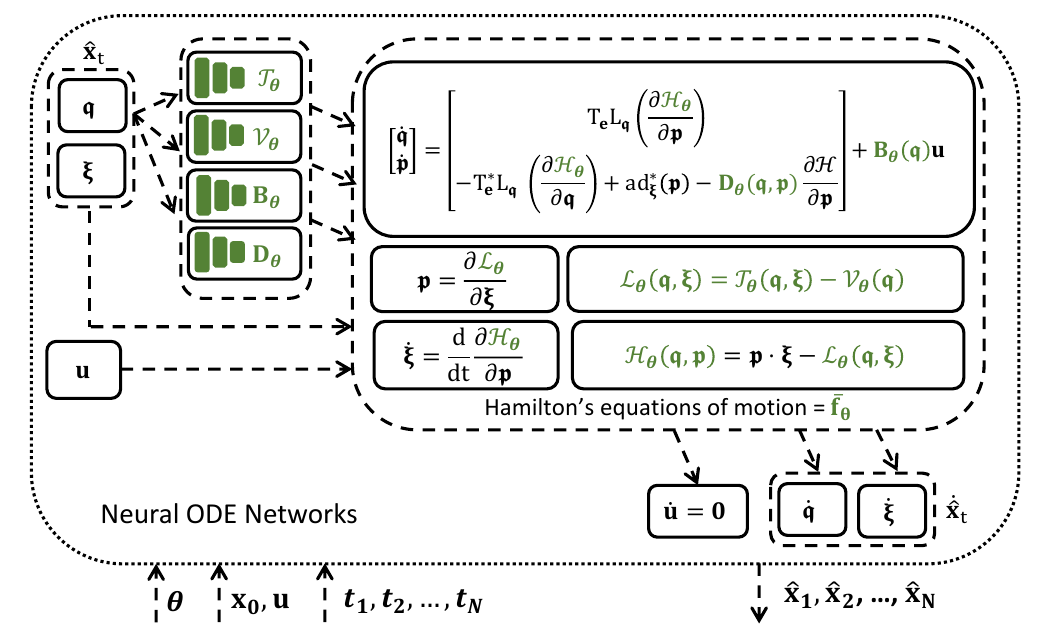}
\caption{\NEWW{Architecture of port-Hamiltonian neural ODE network on matrix Lie group. The trainable terms are shown in green.}}
\label{fig:model_arch}
\end{figure}

To integrate the Hamiltonian equations into the structure of $\bar{\bff}_{\bftheta}(\bfx,\bfu)$, we use four neural networks with parameters $\bftheta = (\bftheta_{\calT}, \bftheta_{\calV}, \bftheta_{\bfD}, \bftheta_{\bfB})$ to approximate the kinetic energy by $\calT_{\bftheta}(\frakq, \bfxi)$, the potential energy by $\calV_{\bftheta}(\frakq)$, the dissipation matrix by $\bfD_{\bftheta}(\frakq, \frakp)$, and the input matrix by $\bfB_{\bftheta}(\frakq)$, respectively. Since the generalized momenta $\frakp$ are not directly available in $\calD$, the time derivative of the generalized velocity $\bfxi$ is obtained from Eq. \eqref{eq:xi_ham}.  
The approximated dynamics function $\bar{\bff}_{\bftheta}(\bfx,\bfu)$ is described with an internal state $\frakp$ as follows:
\begin{subequations}\label{eq:learned_ham_dyn_lie_group}
    \begin{align}
        \dot{\frakq} &= \sfT_\bfe\sfL_\frakq\left(\frac{\partial \calH_{\bftheta} (\frakq, \frakp)}{\partial \frakp}\right), \label{eq:learned_ham_dyn_lie_group_q}\\
        \dot{\frakp} &= \sfa\sfd^*_{\bfxi} (\frakp) - \bfD_{\bftheta}(\frakq, \frakp)\frac{\partial \calH_{\bftheta} (\frakq, \frakp)}{\partial \frakp} \notag \\
        & \qquad -\sfT_\bfe^*\sfL_\frakq\left(\frac{\partial \calH_{\bftheta} (\frakq, \frakp)}{\partial \frakq}\right) + \bfB_{\bftheta}(\frakq)\bfu,\label{eq:learned_ham_dyn_lie_group_p} \\
        \dot{\bfxi} &= \frac{d}{dt}\frac{\partial \calH_{\bftheta}(\frakq, \frakp)}{\partial \frakp},
    \end{align}
\end{subequations}
where $\calH_{\bftheta}(\frakq, \frakp) = \frakp \cdot \bfxi - \calL_{\bftheta}(\frakq, \bfxi)$, and $\calL_{\bftheta}(\frakq, \bfxi) = \calT_{\bftheta}(\frakq, \bfxi) - \calV_{\bftheta}(\frakq)$.
The time derivative $\frac{d}{dt}\frac{\partial \calH_{\bftheta} (\frakq, \frakp)}{\partial \frakp}$ is calculated using automatic differentiation, e.g. by Pytorch \cite{pytorch}. The approximated dynamics function $\bar{\bff}_{\bftheta}(\bfx,\bfu)$ is implemented in a neural ODE network architecture for training, shown in Fig. \ref{fig:model_arch}.

\subsection{Training Process}
\label{subsec:training}

Let $\bar{\bfx}^{(i)}(t)$ denote the state trajectory predicted with control input $\bfu^{(i)}$ by the approximate dynamics $\bar{\bff}_{\bftheta}$ initialized at $\bar{\bfx}^{(i)}(t_0^{(i)}) = \bfx_0^{(i)}$. For sequence $i$, forward passes through the ODE solver in \eqref{eq:ode_solver} return the predicted states $\bar{\bfx}^{(i)}_{0:N}$ at times $t^{(i)}_{0:N}$, where $\bar{\bfx}_n^{(i)} = [\bar{\frakq}_n^{(i)\top} \quad \bar{\bfxi}_n^{(i)\top}]^\top$, for $n = 1, \ldots, N$. The predicted coordinates $\bar{\frakq}_n^{(i)}$ and the ground-truth ones $\frakq_n^{(i)}$ are used to calculate a loss on the Lie group manifold:
\begin{equation} \label{eq:loss_function}
L_{\frakq}(\bftheta) = \sum_{i = 1}^D \sum_{n = 1}^N \left\|\log_\sfG^\vee \left(\bar{\frakq}_n^{(i)} \left(\frakq_n^{(i)}\right)^{-1}\right) \right\|_2^2.
\end{equation}
We use the squared Euclidean norm to calculate losses for the generalized velocity terms: 
\begin{equation}
L_{\bfxi}(\bftheta) = \sum_{i = 1}^D \sum_{n = 1}^N \Vert \bfxi_n^{(i)} - \bar{\bfxi}_n^{(i)} \Vert^2_2.
\end{equation}
The total loss $L(\bftheta)$ is defined as:
\begin{equation}
L(\bftheta)= L_{\frakq}(\bftheta) + L_{\bfxi}(\bftheta).
\end{equation}
The gradient of the total loss function $L(\bftheta)$ is back-propagated by solving an ODE with adjoint states \cite{chen2018neural}. Specifically, let $\bfa = \frac{\partial L}{\partial \bar{\bfx}}$ be the adjoint state and $\bfs = (\bar{\bfx}, \bfa, \frac{\partial L}{\partial \bftheta})$ be the augmented state. The augmented state dynamics are \cite{chen2018neural}:
\begin{equation}
\dot{\bfs} = \bar{\bff}_{\bfs} = (\bar{\bff}_{\bftheta}, -\bfa^\top\frac{\partial \bar{\bff}_{\bftheta}}{\partial \bar{\bfx}}, -\bfa^\top\frac{\partial \bar{\bff}_{\bftheta}}{\partial \bftheta}).
\end{equation}
The predicted state $\bar{\bfx}$, the adjoint state $\bfa$, and the derivatives $\frac{\partial L}{\partial \bftheta}$ can be obtained by a single call to a reverse-time ODE solver starting from $\bfs_N = \bfs({t_N})$:
\begin{equation}
\bfs_0 = \left(\bar{\bfx}_0, \bfa_0, \frac{\partial \calL}{\partial \bftheta}\right) = \text{ODESolver}(\bfs_N, \bar{\bff}_\bfs, t_N),
\end{equation}
where at each time $t_k, k = 1, \ldots, N$, the adjoint state $\bfa_{k}$ at time $t_k$ is reset to $ \frac{\partial L}{\partial \bar{\bfx}_{k}}$. The resulting derivative $ \frac{\partial L}{\partial \bftheta}$ is used to update the parameters $\bftheta$ using gradient descent.  \NEWW{Note that even though the Lie group and Hamiltonian structures are preserved in the continuous-time dynamics function $\bar{\bff}$, the Lie group constraints might still be violated by the predicted states when the model is trained with a neural ODE network on the embedding space $\bbR^{n\times n}$ instead of on the matrix Lie group \cite{lou2020neural, wotte2024optimal}. To prevent this, we use a high-order integrator, such as $5$th-order Runge-Kutta \cite{dormand1980family}, in the Euclidean space neural ODE, for which open-source software is available \cite{chen2018neural,lienen2022torchode}. Investigating how to train Hamiltonian models with a neural ODE network defined directly on the Lie group \cite{iserles2000lie, lou2020neural, wotte2024optimal} is an interesting future direction. For example, the Lie group neural ODE by Wotte et al. \cite{wotte2024optimal} offers an approach to guarantee Lie group constraints in the dynamics model. Closely related to this direction, Duruisseaux et al. \cite{duruisseaux2023lie} learn discrete-time Hamiltonian dynamics while guaranteeing Lie group constraints by design using variational integration}.

\subsection{Application to SE(3) Hamiltonian Dynamics Learning} \label{subsec:SE3_dyn_learning}

This section applies our Lie group Hamiltonian dynamics learning approach to estimate mobile robot dynamics on the $SE(3)$ manifold (Sec. \ref{subsec:ham_dyn_se3}).

\paragraph*{Neural ODE model architecture}
When the Hamiltonian dynamics in \eqref{eq:ham_dyn_lie_group} are defined on the $SE(3)$ manifold, the equations of motion become \eqref{eq:portham_dyn_SE3}. The neural ODE network architecture in  \eqref{eq:learned_ham_dyn_lie_group} is simplified as follows.
We use five neural networks with parameters $\bftheta = (\bftheta_{\bfv}, \bftheta_{\bfomega}, \bftheta_{V}, \bftheta_{\bfD}, \bftheta_{\bfB})$ to approximate the blocks $\bfM^{-1}_{\bfv;{\bftheta}}(\frakq)$, $\bfM^{-1}_{\bfomega;{\bftheta}}(\frakq)$ of the inverse generalized mass in \eqref{eq:mass_matrix_hamiltonian}, the potential energy $\calV_{\bftheta}(\frakq)$, the dissipation matrix $\bfD_{\bftheta}(\frakq, \frakp)$, and the input matrix $\bfB_{\bftheta}(\frakq)$, respectively. The approximated kinetic energy is calculated as $\calT_{\bftheta}(\frakq, \frakp) = \frac{1}{2}\frakp^\top \bfM^{-1}_{\bftheta}(\frakq)\frakp$, where $\bfM_{\bftheta}(\frakq) = \diag (\bfM_{\bfv;{\bftheta}}(\frakq), \bfM_{\bfomega;{\bftheta}}(\frakq))$.

\paragraph*{Neural network design}
In many applications, nominal information is available about the generalized mass matrices $\bfM^{-1}_{\bfv;{\bftheta}}(\frakq)$, $\bfM^{-1}_{\bfomega;{\bftheta}}(\frakq)$, the potential energy $\calV_{\bftheta}(\frakq)$, the dissipation matrix $\bfD_{\bftheta}(\frakq, \frakp)$, and the input matrix $\bfB_{\bftheta}(\frakq)$, and can be included in the neural network design.

Let $\bfM^{-1}_{\bfv 0}(\frakq)$, $\bfM^{-1}_{\bfomega 0}(\frakq)$, and $\bfD_0(\frakq, \frakp)$ be the nominal values of the generalized mass matrices $\bfM^{-1}_{\bfv;\bftheta}(\frakq)$, $\bfM^{-1}_{\bfomega;\bftheta}(\frakq)$ and the dissipation matrix $\bfD_{\bftheta}(\frakq, \frakp)$ with Cholesky decomposition:
\begin{equation}
\begin{aligned}
    \bfM^{-1}_{\bfv 0}(\frakq) &= \bfL_{\bfv 0}(\frakq) \bfL^\top_{\bfv 0}(\frakq),\\
    \bfM^{-1}_{\bfomega 0}(\frakq) &=\bfL_{\bfomega 0}(\frakq) \bfL^\top_{\bfomega 0}(\frakq),\\
    \bfD_{0}(\frakq) &=\bfL_{\bfD 0}(\frakq) \bfL^\top_{\bfD 0}(\frakq).
\end{aligned}
\end{equation}
The learned terms $\bfM^{-1}_{\bfv;{\bftheta}}(\frakq)$, $\bfM^{-1}_{\bfomega;{\bftheta}}(\frakq)$, and $\bfD_{\bftheta}(\frakq, \frakp)$ are obtained using Cholesky decomposition:
\begin{equation}
\label{eq:M_cholesky}
\begin{aligned}
\bfM^{-1}_{\bfv;{\bftheta}}(\frakq) &= \prl{\bfL_{\bfv 0}(\frakq) + \bfL_{\bfv}(\frakq)}\prl{\bfL_{\bfv 0}(\frakq) + \bfL_{\bfv}(\frakq)}^\top + \varepsilon_{\bfv} \bfI,\\
\bfM^{-1}_{\bfomega;{\bftheta}}(\frakq) &= \prl{\bfL_{\bfomega 0}(\frakq) + \bfL_{\bfomega}(\frakq)}\prl{\bfL_{\bfomega 0}(\frakq) + \bfL_{\bfomega}(\frakq)}^\top + \varepsilon_{\bfomega} \bfI,\\
\bfD_{\bftheta}(\frakq, \frakp) &= \prl{\bfL_{\bfD 0}(\frakq, \frakp) \!+\! \bfL_{\bfD}(\frakq, \frakp)}\prl{\bfL_{\bfD 0}(\frakq, \frakp) \!+\! \bfL_{D}(\frakq, \frakp)}^\top
\end{aligned}
\end{equation}
where $\bfL_{_{\bfv}}(\frakq)$, $\bfL_{\bfomega}(\frakq)$, and $\bfL_{\bfD}(\frakq, \frakp)$ are lower-triangular matrices implemented as three neural networks with parameters $\bftheta_{\bfv}$, $\bftheta_{\bfomega}$, and $\bftheta_{\bfD}$ respectively, and $\varepsilon_{\bfv}, \varepsilon_{\bfomega} > 0$.

The potential energy $\calV(\frakq)$ and the input matrix $\bfB(\frakq)$ are implemented with nominal values $\calV_0(\frakq)$ and $\bfB_0(\frakq)$ as follows:
\begin{equation}
    \begin{aligned}
        \calV_{\bftheta}(\frakq) &= \calV_0(\frakq) + \text{L}_{\calV}(\frakq),\\
        \bfB_{\bftheta}(\frakq) &= \bfB_0(\frakq) + \bfL_\bfB(\frakq),
    \end{aligned}
\end{equation}
where $\text{L}_\calV(\frakq)$ and $\bfL_\bfB(\frakq)$ are two neural networks with parameters $\bftheta_\calV$ and $\bftheta_\bfB$, respectively.

\paragraph*{Loss function}

The orientation loss is calculated as:
\begin{equation}
L_{\bfR}(\bftheta) = \sum_{i = 1}^D \sum_{n = 1}^N \left\|\log_{SO(3)}^\vee (\bar{\bfR}_n^{(i)} \bfR_n^{(i)\top}) \right\|_2^2,
\end{equation}
We use the squared Euclidean norm to calculate losses for the position and generalized velocity terms: 
\begin{equation}
\begin{aligned}
L_{\bfp}(\bftheta) &=& \sum_{i = 1}^D \sum_{n = 1}^N \Vert \bfp_n^{(i)} - \bar{\bfp}_n^{(i)} \Vert^2_2,\\
L_{\bfzeta}(\bftheta) &=& \sum_{i = 1}^D \sum_{n = 1}^N \Vert \bfzeta_n^{(i)} - \bar{\bfzeta}_n^{(i)} \Vert^2_2.
\end{aligned}
\end{equation}
The total loss $\calL(\bftheta)$ is defined as:
\begin{equation}
L(\bftheta)= L_{\bfR}(\bftheta) + L_{\bfp}(\bftheta) + L_{\bfzeta}(\bftheta).
\end{equation}

%% file: tex/ControllerDesign.tex
\section{Energy-based Control Design}
\label{sec:controller_design}

The function $\bar{\bff}_{\bftheta}$ \eqref{eq:learned_ham_dyn_lie_group} learned in Sec.~\ref{sec:data_gen_net_design} satisfies the port-Hamiltonian dynamics in Eq. \eqref{eq:port_Hal_dyn} by design. This section extends the interconnection and damping assignment passivity-based control (IDA-PBC) approach \cite{van2014port, wang2008modified, souza2014passivity} to Lie groups to achieve trajectory tracking (Problem \ref{problem:setpoint_reg}) based on the learned port-Hamiltonian dynamics. We further derive a tracking controller specifically for learned Hamiltonian dynamics on the $SE(3)$ manifold in Sec. \ref{subsec:SE3_dyn_learning}. In the remainder of the paper, we omit the subscript $\bftheta$ in the learned model for readability.

\subsection{IDA-PBC Control Design for Trajectory Tracking}
Consider a desired regulation point $(\frakq^*, \frakp^*) \in \sfT^*\sfG$ that the system should be stabilized to. The Hamiltonian function $\mathcal{H}(\frakq, \frakp)$, representing the total energy of the system, generally does not have a minimum at $(\frakq^*, \frakp^*)$. An IDA-PBC controller \cite{van2014port,ortega2002stabilization,wang2008modified} is designed to inject additional energy $\calH_a(\frakq, \frakp)$ such that the desired total energy:
\begin{equation}
\mathcal{H}_d(\frakq, \frakp) = \mathcal{H}(\frakq, \frakp) + \mathcal{H}_a(\frakq, \frakp)
\end{equation}
achieves its minimum at $(\frakq^*, \frakp^*)$. In other words, the closed-loop system obtained by applying the controller to the port-Hamiltonian dynamics in \eqref{eq:port_Hal_dyn} should have the form:
\begin{equation}
\label{eq:desired_port_Hal_dyn}
\begin{bmatrix}
\dot{\frakq} \\
\dot{\frakp} \\
\end{bmatrix}
= (\mathcal{J}_d(\frakq, \frakp) - \mathcal{R}_d(\frakq, \frakp))
\begin{bmatrix}
\frac{\partial \mathcal{H}_d}{\partial \frakq} \\
\frac{\partial \mathcal{H}_d}{\partial \frakp} 
\end{bmatrix}.
\end{equation}
to ensure that $(\frakq^*, \frakp^*)$ is an equilibrium. The control input $\bfu$ should be chosen so that \eqref{eq:port_Hal_dyn} and \eqref{eq:desired_port_Hal_dyn} are equal. This matching equation design does not directly apply to trajectory tracking, especially for underactuated systems \cite{souza2014passivity, wang2008modified}.

Consider a desired trajectory $(\frakq^*(t), \frakp^*(t))$ that the system should track.
Let $(\frakq_e(t), \frakp_e(t))$ denote the error in the generalized coordinates and momentum, respectively, where $\frakq_e = (\frakq^*)^{-1}\frakq \in \sfG$ and $\frakp_e = \frakp - \frakp^* \in \sfT_{\frakq_e}^*\sfG$. For trajectory tracking, the desired total energy $\calH_d(\frakq_e, \frakp_e)$ is defined in terms of the error state, with desired closed-loop dynamics:
\begin{equation}\label{eq:desired_port_Hal_dyn_tracking}
\begin{bmatrix}
\dot{\frakq}_e \\
\dot{\frakp}_e \\
\end{bmatrix}
= (\mathcal{J}_d(\frakq_e, \frakp_e) - \mathcal{R}_d(\frakq_e, \frakp_e))
\begin{bmatrix} 
\frac{\partial \mathcal{H}_d}{\partial \frakq_e} \\
\frac{\partial \mathcal{H}_d}{\partial \frakp_e} 
\end{bmatrix}.
\end{equation}
Matching \eqref{eq:desired_port_Hal_dyn_tracking} with \eqref{eq:port_Hal_dyn} leads to the following requirement for the control input:
\begin{align}
\mathbf\calG(\frakq, \frakp)\bfu = (&\mathcal{J}_d(\frakq_e, \frakp_e) - \mathcal{R}_d(\frakq_e, \frakp_e))
\begin{bmatrix}
\frac{\partial \mathcal{H}_d}{\partial \frakq_e} \\
\frac{\partial \mathcal{H}_d}{\partial \frakp_e} 
\end{bmatrix} \label{eq:matching_eqn_tracking}\\
- (&\mathbf\calJ(\frakq, \frakp) - \mathbf\calR(\frakq, \frakp))
\begin{bmatrix}
\frac{\partial \mathcal{H}}{\partial \frakq} \\
\frac{\partial \mathcal{H}}{\partial \frakp} 
\end{bmatrix} + \begin{bmatrix}
\dot{\frakq} \\
\dot{\frakp}
\end{bmatrix} - \begin{bmatrix}
\dot{\frakq}_e \\
\dot{\frakp}_e
\end{bmatrix}. \notag
\end{align}
The control input can be obtained from \eqref{eq:matching_eqn_tracking} as the sum $\bfu = \bfu_{ES} + \bfu_{DI}$ of an energy-shaping component $\bfu_{ES}$ and a damping-injection component $\bfu_{DI}$:
\begin{subequations} \label{eq:control_policy_lie_group}
    \begin{align}
        \bfu_{ES} &= \mathbf\calG^{\dagger}(\frakq, \frakp)\left((\mathcal{J}_d(\frakq_e, \frakp_e))
\begin{bmatrix}
\frac{\partial \mathcal{H}_d}{\partial \frakq_e} \\
\frac{\partial \mathcal{H}_d}{\partial \frakp_e} 
\end{bmatrix}  - \begin{bmatrix}
\dot{\frakq}_e \\
\dot{\frakp}_e
\end{bmatrix} \right.  \label{eq:general_u_ES_lie}\\
& \left. \qquad - (\mathbf\calJ(\frakq, \frakp) - \mathbf\calR(\frakq, \frakp))
\begin{bmatrix}
\frac{\partial \mathcal{H}}{\partial \frakq} \\
\frac{\partial \mathcal{H}}{\partial \frakp} 
\end{bmatrix} + \begin{bmatrix}
\dot{\frakq} \\
\dot{\frakp}
\end{bmatrix} \right), \notag \\
\bfu_{DI} &= -\mathbf\calG^{\dagger}(\frakq, \frakp)\mathcal{R}_d(\frakq_e,\frakp_e)\begin{bmatrix}
\frac{\partial \mathcal{H}_d}{\partial \frakq_e} \\
\frac{\partial \mathcal{H}_d}{\partial \frakp_e} 
\end{bmatrix} , \label{eq:general_u_DI_lie}
    \end{align}
\end{subequations}
where $\mathbf\calG^{\dagger}(\frakq, \frakp) = \left(\mathbf\calG^{\top}(\frakq, \frakp)\mathbf\calG(\frakq, \frakp)\right)^{-1}\mathbf\calG^{\top}(\frakq, \frakp)$ is the pseudo-inverse of $\mathbf\calG(\frakq, \frakp)$. The control input $\bfu_{ES}$ exists as long as the desired interconnection matrix $\mathcal{J}_d$, dissipation matrix $\mathcal{R}_d$, and total energy $\mathcal{H}_d$ satisfy the following matching condition for all $(\frakq,\frakp)\in\sfT^*\sfG$ and $(\frakq_e,\frakp_e)\in\sfT^*\sfG$:
\begin{align}\label{eq:matching_condition}
\calG^{\dagger}(\frakq,\frakp)&\left(\mathcal{J}_d(\frakq_e, \frakp_e) - \mathcal{R}_d(\frakq_e, \frakp_e))
\begin{bmatrix}
\frac{\partial \mathcal{H}_d}{\partial \frakq_e} \\
\frac{\partial \mathcal{H}_d}{\partial \frakp_e} 
\end{bmatrix} \right.\\
- &\left.(\calJ(\frakq, \frakp) - \calR(\frakq, \frakp))
\begin{bmatrix}
\frac{\partial \mathcal{H}}{\partial \frakq} \\
\frac{\partial \mathcal{H}}{\partial \frakp} 
\end{bmatrix} + \begin{bmatrix}
\dot{\frakq} \\
\dot{\frakp}
\end{bmatrix} - \begin{bmatrix}
\dot{\frakq}_e \\
\dot{\frakp}_e
\end{bmatrix}\right) = 0. \notag
\end{align}
where $\calG^{\perp}(\frakq, \frakp)$ is a maximal-rank left annihilator of $\calG(\frakq, \frakp)$, i.e., $\calG^{\perp}(\frakq, \frakp)\calG(\frakq, \frakp) = \bf0$.

\subsection{Tracking Control for Port-Hamiltonian Dynamics on the SE(3) Manifold}
Consider a desired state trajectory $\bfx^*(t) = (\frakq^*(t), \bfzeta^*(t))$ that the system should track where $\frakq^*(t) \in 
SE(3)$ is the desired pose and $\bfzeta^*(t) = \begin{bmatrix}\bfv^*(t)^\top & \bfomega^*(t)^\top\end{bmatrix}^\top$ is the desired generalized velocity expressed in the desired frame. Let $\frakp^* = \bfM \begin{bmatrix} \bfR^\top \bfR^*\bfv^* \\ \bfR^\top \bfR^* \bfomega^*\end{bmatrix}$ denote the desired momentum, defined based on \eqref{eq:momenta_Mtwist} with the desired velocity expressed in the body frame. Let $\bfp_e = \bfp - \bfp^*$ and $\bfR_e = \bfR^{*\top} \bfR = \begin{bmatrix}
\bfr_{e1} & \bfr_{e2} & \bfr_{e3}
\end{bmatrix}^\top$ be the position error and rotation error between the current orientation $\bfR$ and the desired one $\bfR^*$, respectively. The vectorized error $\frakq_e$ in the generalized coordinates is:
\begin{equation}\label{eq:qe_def}
\frakq_e = \begin{bmatrix} \bfp_e^{*\top} &
\bfr_{e1}^\top & \bfr_{e2}^\top & \bfr_{e3}^\top \end{bmatrix}^\top.
\end{equation}

The error in the generalized momenta is $\frakp_e = \frakp - \frakp^*$, described in the body frame. The desired total energy is defined in terms of the error state as: 
\begin{equation}
\label{eq:desired_Hamil_tracking}
\calH_d(\frakq_e, \frakp_e) = \frac{1}{2}\frakp_e^\top \bfM_d^{-1}(\frakq_e)\frakp_e + V_d(\frakq_e),
\end{equation}
where $\bfM_d(\frakq_e)$ and $V_d(\frakq_e)$ are the desired generalized mass and potential energy. 

Choosing the following desired inter-connection matrix and dissipation matrix:
\begin{equation}\label{eq:JdRdChoice}
\mathcal{J}_d(\frakq_e, \frakp_e) = \begin{bmatrix}
\bf0 & \bfJ_1 \\
-\bfJ_1^\top & \bfJ_2\end{bmatrix}, \quad
\mathcal{R}_d(\frakq_e, \frakp_e) = \begin{bmatrix}
\bf0 & \bf0 \\
\bf0 & \bfK_\bfd 
\end{bmatrix},
\end{equation}
and plugging $\calJ(\frakq, \frakp)$ and $\calR(\frakq, \frakp)$ from \eqref{eq:SE3_PH_J} into the matching equations in \eqref{eq:matching_eqn_tracking}, leads to: 
\begin{eqnarray}
\bf0 &=& \bfJ_1 \frac{\partial \calH_d}{\partial \frakp_e} - \frakq^{\times} \frac{\partial \calH}{\partial \frakp} + \dot{\frakq} - \dot{\frakq}_e, \label{eq:tracking_cond1} \\
\bfB(\frakq)\bfu &=& \frakq^{\times\top}\frac{\partial \calH}{\partial \frakq} - \bfJ_1^\top \frac{\partial \calH_d}{\partial \frakq_e} + \bfJ_2\frac{\partial \calH_d}{\partial \frakp_e} - \frakp^{\times}\frac{\partial \calH}{\partial \frakp} \nonumber \\
&& - \bfK_\bfd\frac{\partial \calH_d}{\partial \frakp_e} + \bfD(\frakq, \frakp)\frac{\partial \calH}{\partial \frakp} + \dot{\frakp} - \dot{\frakp}_e. \label{eq:tracking_cond2} 
\end{eqnarray}
Assuming $\bfM_d(\frakq_e) = \bfM(\frakq)$, \eqref{eq:tracking_cond1} is satisfied if we choose $\bfJ_1 =\begin{bmatrix}
	\bfR^\top\!\!\!\! & \bf0 & \bf0 & \bf0 \\
	\bf0 & \hat{\bfr}_{e1}^\top & \hat{\bfr}_{e2}^\top & \hat{\bfr}_{e3}^\top
	\end{bmatrix}^\top$. 
Indeed, we have $\dot{\frakq} = \frakq^\times \frac{\partial \calH}{\partial \frakp}$ (from (11) and (19)) and $$\frac{\partial \calH_d}{\partial \frakp_e} = \bfM^{-1}_d \frakp_e = \begin{bmatrix} \bfv_e \\
			    \bfomega_e \end{bmatrix} = \begin{bmatrix} \bfv - \bfR^\top \bfR^*\bfv^* \\
			   \bfomega - \bfR^\top \bfR^*\bfomega^* \end{bmatrix}.$$

The error dynamics becomes:
\begin{equation}
\dot{\frakq}_e = \begin{bmatrix} \dot{\bfp} - \dot{\bfp}^* \\
\dot{\bfr}_{e1} \\ \dot{\bfr}_{e2} \\ \dot{\bfr}_{e3} \end{bmatrix} = \begin{bmatrix}
\bfR^\top\!\!\!\! & \bf0 & \bf0 & \bf0 \\
\bf0 & \hat{\bfr}_{e1}^\top & \hat{\bfr}_{e2}^\top & \hat{\bfr}_{e3}^\top
\end{bmatrix}^\top \begin{bmatrix} \bfv_e \\
\bfomega_e \end{bmatrix} = \bfJ_1 \frac{\partial \calH_d}{\partial \frakp_e},
\end{equation}
since $\dot{\bfR}_e = \frac{d}{dt} (\bfR_e) = \bfR_e \widehat{\bfomega}_e$ as shown in \cite[Sec. III-A]{lee2010geometric} and $\dot{\bfp} - \dot{\bfp}^*~=~\bfR\bfv - \bfR^* \bfv^* = \bfR \bfv_e$.

	
The desired control input can be obtained from \eqref{eq:tracking_cond2} as $\bfu = \bfu_{ES} + \bfu_{DI}$ with:
\begin{subequations}\label{eq:control_policy_SE3}
\begin{align}
\bfu_{ES} &= \bfB^{\dagger}(\frakq)\left(\frakq^{\times\top}\frac{\partial \calH}{\partial \frakq} - \mathbf\bfJ_1^{\top}\frac{\partial \calH_d}{\partial \frakq_e} + \bfJ_2\frac{\partial \calH_d}{\partial \frakp_e} \right. \label{eq:general_u_ES} \\
& \left.  \qquad- \frakp^{\times}\frac{\partial \calH}{\partial \frakp} + \bfD(\frakq, \frakp)\frac{\partial \calH}{\partial \frakp} + \dot{\frakp} - \dot{\frakp}_e \right), \notag \\
\bfu_{DI} &= -\bfB^{\dagger}(\frakq)\bfK_\bfd\frac{\partial \calH_d}{\partial \frakp_e}, \label{eq:general_u_DI}
\end{align}
\end{subequations}
where $\bfB^{\dagger}(\frakq) = \left(\bfB^{\top}(\frakq)\bfB(\frakq)\right)^{-1}\bfB^{\top}(\frakq)$ is the pseudo-inverse of $\bfB(\frakq)$. 
\NEW{The matching condition \eqref{eq:matching_condition} becomes:
\begin{equation}
\begin{aligned}
    \bfB^\perp(\frakq)\left(\frakq^{\times\top}\frac{\partial \calH}{\partial \frakq} - \bfJ_1^{\top}\frac{\partial \calH_d}{\partial \frakq_e} + \bfJ_2\frac{\partial \calH_d}{\partial \frakp_e}\right. & \\
- \left.\frakp^{\times}\frac{\partial \calH}{\partial \frakp} + \dot{\frakp} - \dot{\frakp}_e \right) &= 0. \label{eq:matching_condition_se3}
\end{aligned}
\end{equation}}
In this paper, we reshape the open-loop Hamiltonian $\mathcal{H}(\frakq, \frakp)$ into the following desired total energy $\mathcal{H}_d(\frakq_e, \frakp_e)$, minimized along the desired trajectory:
\begin{align} \label{eq:desired_ham_SE3}
\mathcal{H}_d&(\frakq_e, \frakp_e) = \frac{1}{2}(\bfp - \bfp^*)^\top\bfK_\bfp(\bfp - \bfp^*)  \\
&+ \frac{1}{2} \tr(\bfK_{\bfR}(\bfI - \bfR^{*\top}\bfR)) + \frac{1}{2}(\frakp-\frakp^*)^\top\bfM^{-1}(\frakq)(\frakp-\frakp^*),\notag
\end{align}
where $\bfK_\bfp, \bfK_\bfR \succ 0$ are positive-definite matrices.

For an $SE(3)$ rigid-body system with constant generalized mass matrix $\bfM_d = \bfM$ and $\bfJ_{2} = 0$, which is a common choice, the energy-shaping term in \eqref{eq:general_u_ES} and the damping-injection term in \eqref{eq:general_u_DI} simplify as:
\begin{align}
\label{eq:idapbc_pose_twist_tracking}
\bfu_{ES}(\frakq, \frakp) &= \bfB^{\dagger}(\frakq)\left(\frakq^{\times\top} \frac{\partial V}{\partial \frakq} - \prl{\frakp^{\times}- \bfD(\frakq, \frakp)}\bfM^{-1}\frakp \right. \notag\\
&\qquad\qquad\qquad\qquad\qquad \left. - \bfe(\frakq,\frakq^*) + \dot{\frakp}^* \vphantom{\frac{\partial V}{\partial \frakq}}\right), \notag\\
\bfu_{DI}(\frakq, \frakp) &= -\bfB^{\dagger}(\frakq) \bfK_\bfd  \bfM^{-1}(\frakp-\frakp^*),
\end{align}
where the generalized coordinate error between $\frakq$ and $\frakq^*$ is:
\begin{equation}
\label{eq:coordinate-error}
\begin{aligned}
\bfe(\frakq,\frakq^*) := \bfJ_1^{\top}\frac{\partial  V_d}{\partial \frakq_e} = \begin{bmatrix} \bfR^\top\bfK_\bfp(\bfp - \bfp^*) \\ \frac{1}{2}\prl{\bfK_{\bfR}\bfR^{*\top}\bfR-\bfR^\top\bfR^{*}\bfK_{\bfR}^\top}^{\vee}\end{bmatrix},
\end{aligned}
\end{equation}
and the derivative of the desired momentum is:
\begin{equation}
    \dot{\frakp}^* = \bfM \begin{bmatrix} \bfR^\top \ddot{\bfp}^* - \hat{\bfomega}\bfR^\top \dot{\bfp}^* \\
\bfR^\top \bfR^* \dot{\bfomega}^* - \hat{\bfomega}_e\bfR^\top\bfR^* \bfomega^* \end{bmatrix}.
\end{equation}
By expanding the terms in  \eqref{eq:idapbc_pose_twist_tracking}, we have: 
\begin{eqnarray}
\frakp^{\times}\bfM^{-1}\frakp &=& \frakp^{\times} \bfzeta = \begin{bmatrix}
\hat{\frakp}_\bfv \bfomega \\
\hat{\frakp}_{\bfomega} \bfomega + \hat{\frakp}_\bfv \bfv
\end{bmatrix}, \\
\bfM^{-1}(\frakp-\frakp^*) &=& \begin{bmatrix}
\bfv - \bfR^{\top}\dot{\bfp}^*  \\
\bfomega - \bfR^\top \bfR^* \bfomega^*
\end{bmatrix}, \\
\frakq^{\times\top} \frac{\partial \calV}{\partial \frakq} &=& \begin{bmatrix}
\bfR^{\top} \frac{\partial \calV(\frakq)}{\partial \bfp} \\
\sum_{i = 1}^3 \hat{\bfr}_i \frac{\partial \calV(\frakq)}{\partial \bfr_i}
\end{bmatrix}.
\end{eqnarray}

\begin{theorem} \label{thm:stability_analysis}
\NEW{Consider a port-Hamiltonian system on the $SE(3)$ manifold with dynamics \eqref{eq:portham_dyn_SE3}. Assume that the matching condition \eqref{eq:matching_condition_se3} is satisfied, the desired momentum's derivative $\dot{\frakp}^*$ is bounded, and the matrices $\bfK_\bfp$, $\bfK_\bfR$, and $\bfK_\bfd$ are positive-definite. The control policy in \eqref{eq:control_policy_SE3} leads to closed-loop error dynamics in \eqref{eq:desired_port_Hal_dyn_tracking}, \eqref{eq:JdRdChoice}. The tracking error $(\frakq_e, \frakp_e) = \prl{(\bfp_e, \bfR_e), \frakp_e}$ asymptotically stabilizes to $\prl{(\bf0, \bfI), \bf0}$ with Lyapunov function given by the desired Hamiltonian $\calH_d(\frakq_e,\frakp_e)$ in \eqref{eq:desired_ham_SE3}.}
\end{theorem}
\begin{proof}
\NEW{
Since the matching condition is satisfied and the desired momentum's derivative $\dot{\frakp}^*$ is bounded, the control policy in \eqref{eq:control_policy_SE3}, \eqref{eq:general_u_DI} exists and achieves the desired closed-loop error dynamics:
\begin{equation}
\begin{bmatrix}
\dot{\frakq}_e \\
\dot{\frakp}_e \\
\end{bmatrix}
= \begin{bmatrix}
\bf0 & \bfJ_1 \\
-\bfJ_1^\top & \bfJ_2 - \bfK_d\end{bmatrix}
\begin{bmatrix}
\frac{\partial \mathcal{H}_d}{\partial \frakq_e} \\
\frac{\partial \mathcal{H}_d}{\partial \frakp_e} 
\end{bmatrix}.
\end{equation}
We have $\tr\prl{\bfI - \bfR^{*\top} \bfR} \geq 0$, as all entries in $\bfR_e \in SO(3)$ are less than $1$.
Since $\bfM$, $\bfK_\bfp$ and $\bfK_\bfR$ are positive-definite matrices, the desired Hamiltonian $\calH_d$ is positive-definite, and achieves minimum value $0$ only at $\frakq_e = (\bf0, \bfI)$ and $\frakp_e = \bf0$, i.e., no position, rotation and momentum errors. The time derivative of $\calH_d(\frakq_e,\frakp_e)$ can be computed as:
\begin{equation}
\begin{aligned}
\dot{\calH}_d(\frakq_e,\frakp_e) &= \frac{\partial \mathcal{H}_d}{\partial \frakq_e}^\top\dot{\frakq}_e + \frac{\partial \mathcal{H}_d}{\partial \frakp_e} ^\top \dot{\frakp}_e\\
              & = -\frakp_e^\top \bfM^{-1}(\frakq)\bfK_d \bfM^{-1}(\frakq) \frakp_e.
\end{aligned}
\end{equation}
As $\bfK_d$ and $\bfM(\frakq)$ are positive-definite, we have $\dot{\calH}_d(\frakq_e,\frakp_e) \leq 0$ for all $(\frakq_e, \frakp_e)$ and equality holds at $\prl{(\bf0, \bfI), \bf0}$. By LaSalle's invariance principle \cite{khalil2002nonlinear}, the tracking errors $(\frakq_e, \frakp_e)$ asymptotically converge to $\prl{(\bf0, \bfI), \bf0}$.
}
\end{proof}

\NEW{Without requiring a priori knowledge of the system parameters, the control design in \eqref{eq:control_policy_SE3} offers a unified control approach for $SE(3)$ Hamiltonian systems that achieves trajectory tracking, if permissible by the system's degree of underactuation. Thus, our control design solves Problem \ref{problem:setpoint_reg} for rigid-body robot systems, such as UGVs, UAVs, and UUVs, with tracking performance guaranteed by Theorem \ref{thm:stability_analysis}.}


%% file: tex/ExperimentalResults.tex
\section{Evaluation}
\label{sec:experimental_results}

We verify the effectiveness of our port-Hamiltonian neural ODE network for dynamics learning and control on matrix Lie groups using a simulated pendulum, a simulated Crazyflie quadrotor, and a real PX4 quadrotor platform, whose states evolve on the $SE(3)$ manifold. The implementation details for the experiments are provided in Appendix \ref{subsec:implement_details}.

\subsection{Pendulum}
\label{subsec:pendulum_so3}
\NEWW{In this section, we verify our port-Hamiltonian dynamics learning and control approach on the $SO(3)$ manifold}. We consider a pendulum with the following dynamics:
\begin{equation}
\label{eq:pend_gt_q_dyn}
\ddot{\varphi} = -15\sin{\varphi} + 3u - 0.2\dot{\varphi},
\end{equation}
where $\varphi$ is the angle of the pendulum with respect to its vertically-downward position and $u$ is a scalar control input. The ground-truth mass, potential energy, friction coefficient, and the input gain are: $m = 1/3$, $\calV(\varphi) = 5(1-\cos{\varphi})$, $D(\varphi) = 0.2/3$, and $B(\varphi) = 1$, respectively.
We collected data of the form $\{(\cos{\varphi}, \sin{\varphi}, \dot{\varphi})\}$ from an OpenAI Gym environment, provided by \cite{zhong2020symplectic}, with the dynamics in \eqref{eq:pend_gt_q_dyn}. To illustrate our Lie group neural ODE learning, we represent the angle $\varphi$ as a rotation matrix:
%
\begin{equation}
\bfR = \begin{bmatrix}
\cos{\varphi} & - \sin{\varphi} & 0 \\
\sin{\varphi} & \cos{\varphi} & 0 \\
0 & 0 & 1 
\end{bmatrix},
\end{equation}
representing the pendulum orientation. We let $\bfomega = [0, 0, \dot{\varphi}]$ and remove position $\bfp$ and linear velocity $\bfv$ from the Hamiltonian dynamics in \eqref{eq:portham_dyn_SE3}, restricting the system to the $SO(3)$ manifold with generalized coordinates $\mathbf\frakq = [\bfr_1^\top \quad \bfr_2^\top \quad \bfr_3^\top]^\top$.

As described in Sec. \ref{subsec:SE3_dyn_learning}, control inputs $\bfu^{(i)}$ were sampled randomly and applied to the pendulum for five time intervals of $0.05s$, forming a dataset $\mathcal{D} = \left\{t_{0:N}^{(i)},\mathbf\frakq_{0:N}^{(i)}, \bfomega_{0:N}^{(i)}, \bfu^{(i)})\right\}_{i =1}^D$ with $N = 5$ and $D=5120$. We trained an $SO(3)$ port-Hamiltonian neural ODE network as described in Sec. \ref{subsec:SE3_dyn_learning} for $5000$ iterations without any nominal model, i.e., $\bfM^{-1}_{\bfomega 0}(\mathbf\frakq) = \bf0$, $\bfD_{\bfomega 0}(\frakq, \frakp) = \bf0$, $V_0(\frakq) = \bf0$ and $\bfB_0(\frakq) = \bf0$.

As noted in \cite{zhong2020symplectic, duong21hamiltonian}, since the generalized momenta $\mathbf\frakp$ are not available in the dataset, the dynamics of $\mathbf\frakq$ in \eqref{eq:pend_gt_q_dyn} do not change if $\mathbf\frakp$ is scaled by a factor $\beta > 0$. This is also true in our formulation as scaling $\mathbf\frakp$ leaves the dynamics of $\mathbf\frakq$ in \eqref{eq:portham_dyn_SE3} unchanged. To emphasize this scale-invariance, let $\bfM_\beta(\mathbf\frakq) = \beta\bfM(\mathbf\frakq)$, $\calV_\beta(\mathbf\frakq) = \beta \calV(\mathbf\frakq)$, $\bfD_{\beta}(\mathbf\frakq, \mathbf\frakp) = \beta \bfD(\mathbf\frakq, \mathbf\frakp)$ $\bfB_\beta(\mathbf\frakq) = \beta\bfB(\mathbf\frakq)$, and:
\begin{equation}
\begin{aligned}
\mathbf\frakp_{\beta} &= \bfM_{\beta}(\mathbf\frakq)\bfomega = \beta \mathbf\frakp,  \quad\qquad \dot{\mathbf\frakp}_\beta = \beta \dot{\mathbf\frakp}, \\
\mathcal{H}_{\beta}(\mathbf\frakq,\mathbf\frakp) &= \frac{1}{2}\mathbf\frakp_{\beta}^\top \bfM_{\beta}^{-1}(\mathbf\frakq) \mathbf\frakp_{\beta} + V_{\beta}(\mathbf\frakq)=\beta\mathcal{H}(\mathbf\frakq,\mathbf\frakp), \\
\frac{\partial{\mathcal{H}_{\beta}(\mathbf\frakq,\mathbf\frakp)}}{\partial \mathbf\frakp_\beta} &= \bfM^{-1}_{\beta}(\mathbf\frakq) \mathbf\frakp_{\beta} = \frac{\partial{H(\mathbf\frakq,\mathbf\frakp)}}{\partial \mathbf\frakp},
\end{aligned}
\end{equation}
guaranteeing that the equations of motions \eqref{eq:portham_dyn_SE3} still hold.

\begin{figure}[t]
\begin{subfigure}[t]{0.23\textwidth}
        \centering
        \includegraphics[width=\textwidth]{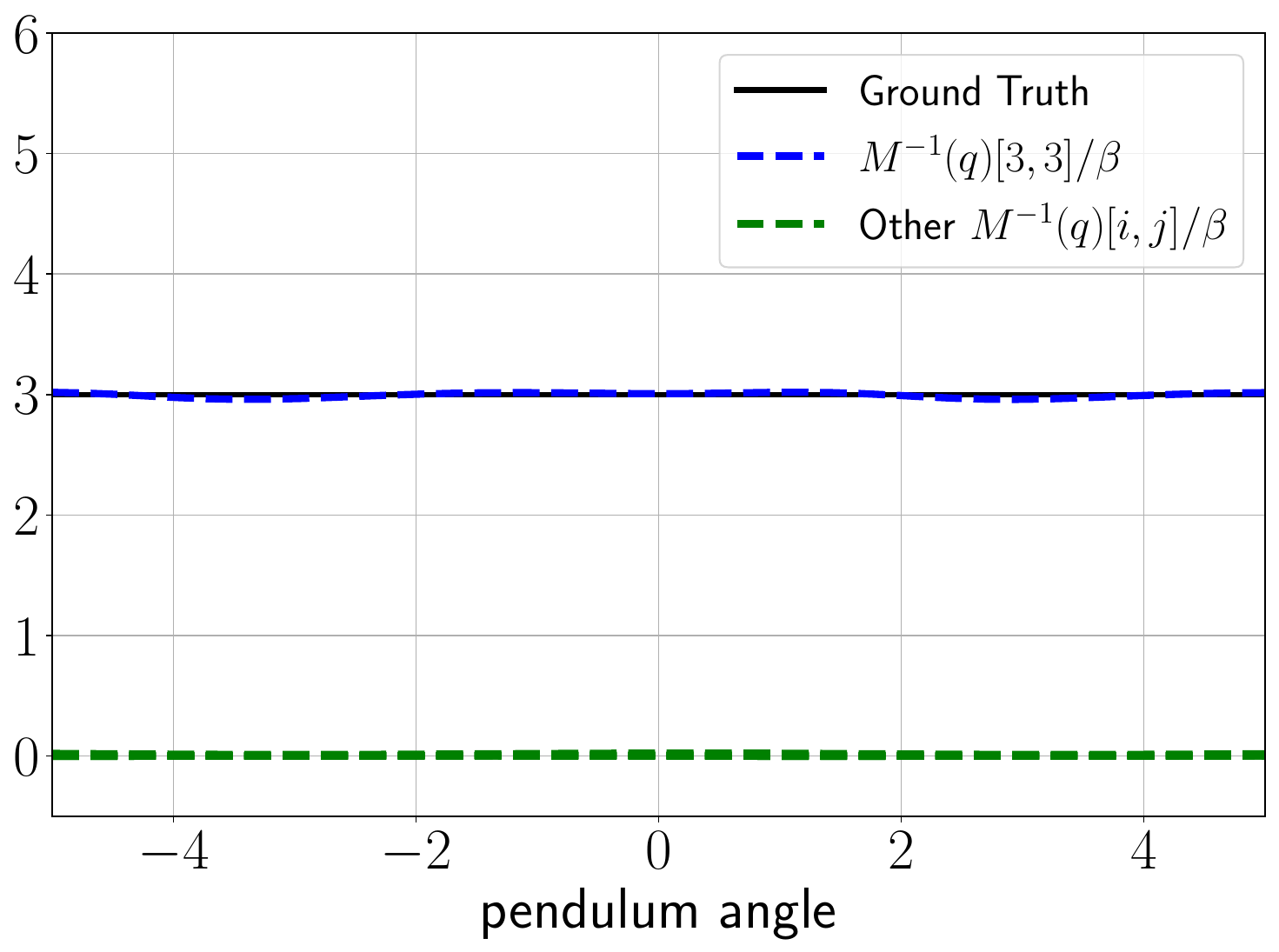}
        \caption{$\bfM^{-1}(q)/\beta$ versus $\varphi$.}
        \label{fig:pend_M_x_all}
\end{subfigure}%
\hfill%
\begin{subfigure}[t]{0.25\textwidth}
        \centering
\includegraphics[width=\textwidth]{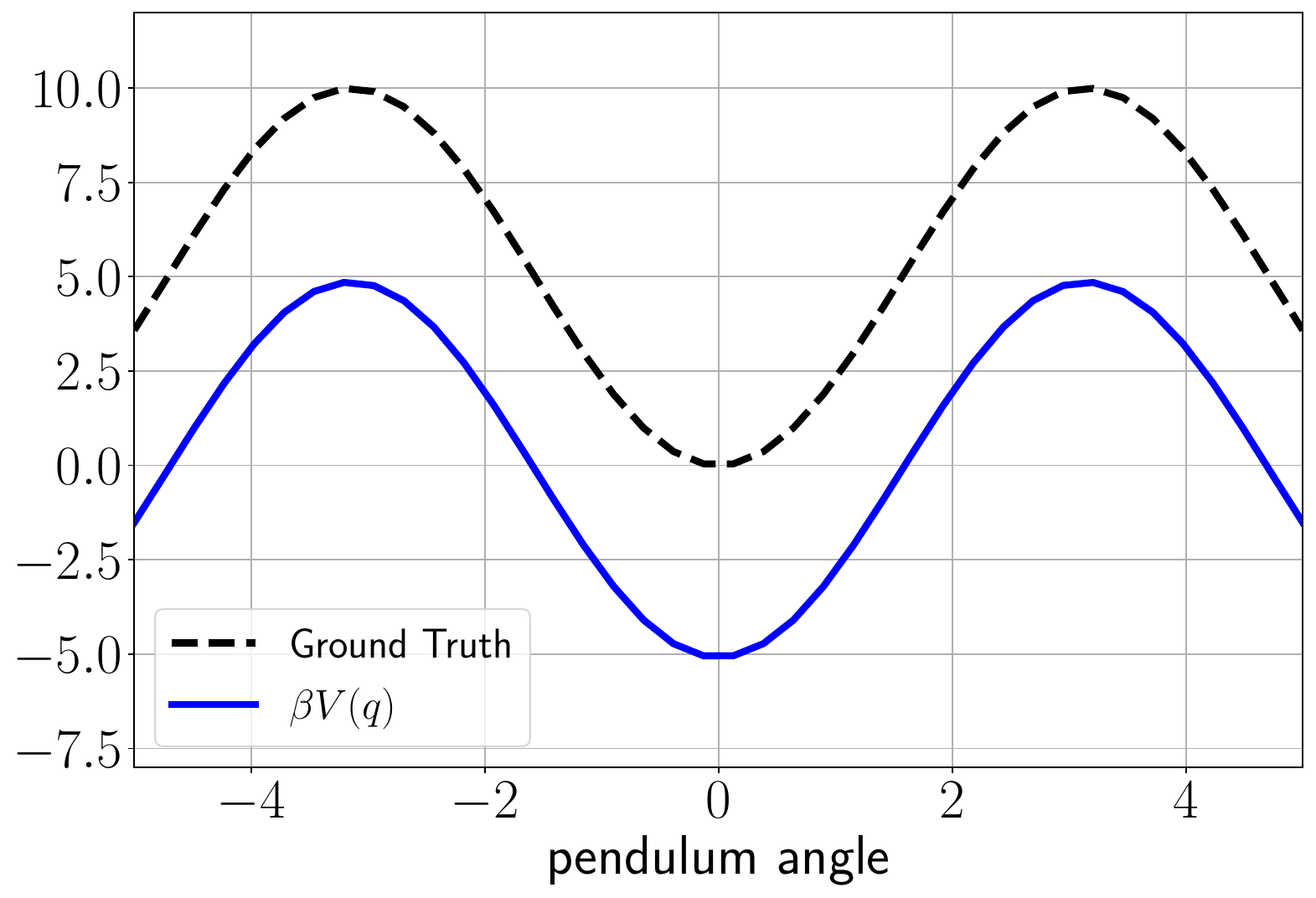}%
        \caption{$\beta \calV(\frakq)$ versus $\varphi$.}
        \label{fig:pend_Vx}
\end{subfigure}%

\begin{subfigure}[t]{0.23\textwidth}
        \centering
        \includegraphics[width=\textwidth]{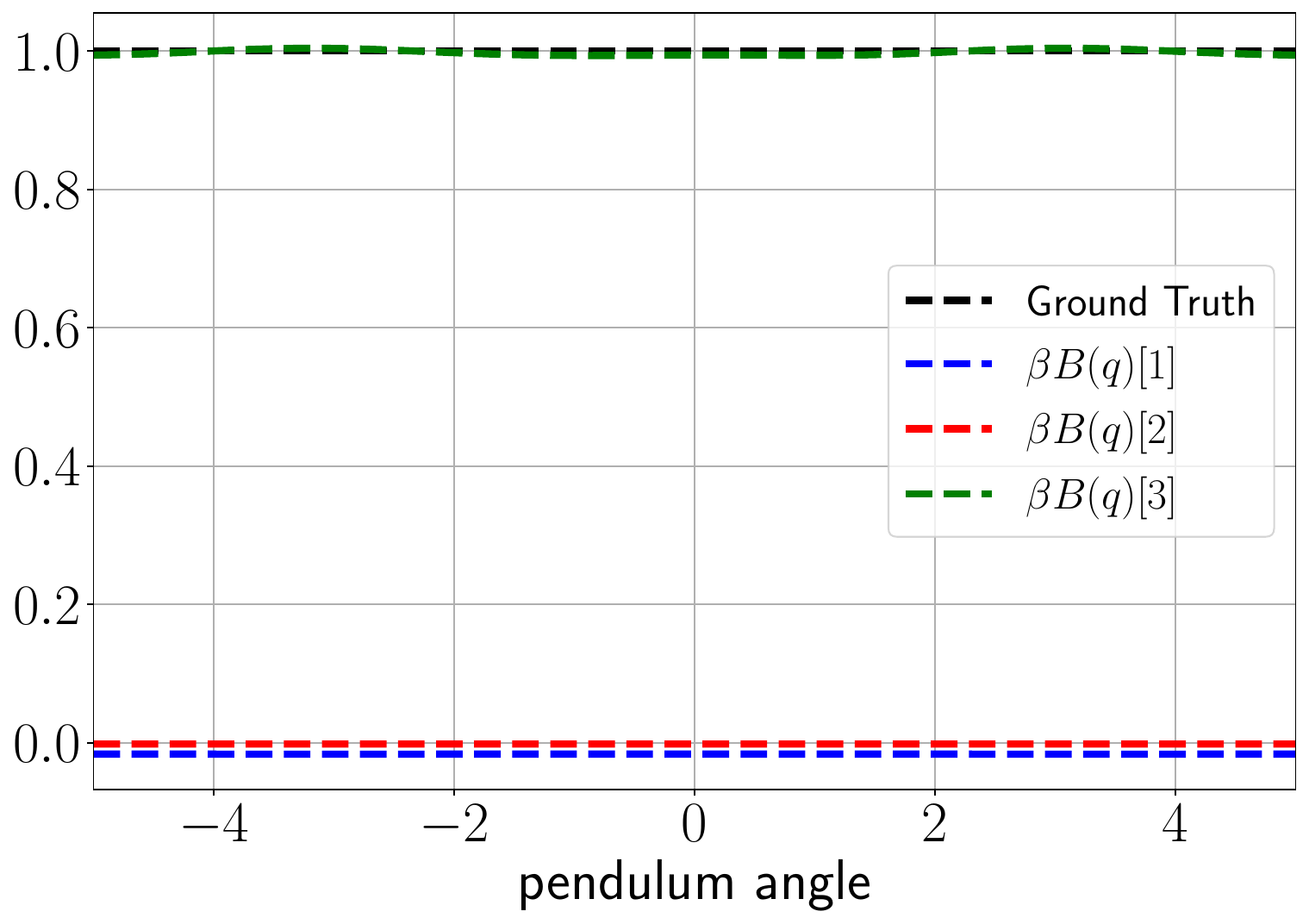}
        \caption{$\beta\bfB(\frakq)$ versus $\varphi$.}
        \label{fig:pend_gx}
\end{subfigure}%
\hfill%
\begin{subfigure}[t]{0.24\textwidth}
        \centering
        \includegraphics[width=\textwidth]{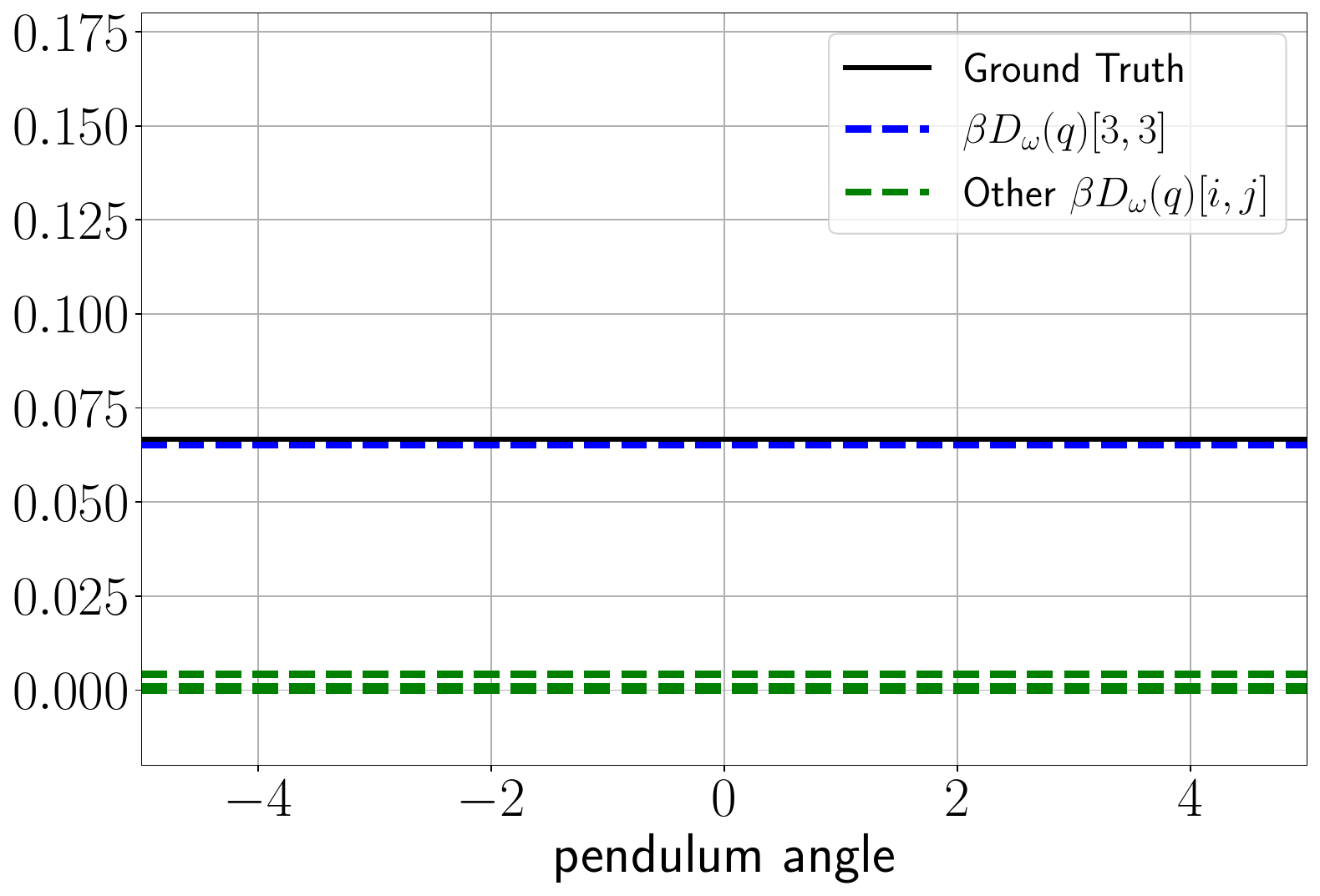}
        \caption{$\beta\bfD_\omega(\frakq)$ versus $\varphi$.}
        \label{fig:pend_D_omega}
\end{subfigure}%

\begin{subfigure}[t]{0.23\textwidth}
        \centering
\includegraphics[width=\textwidth]{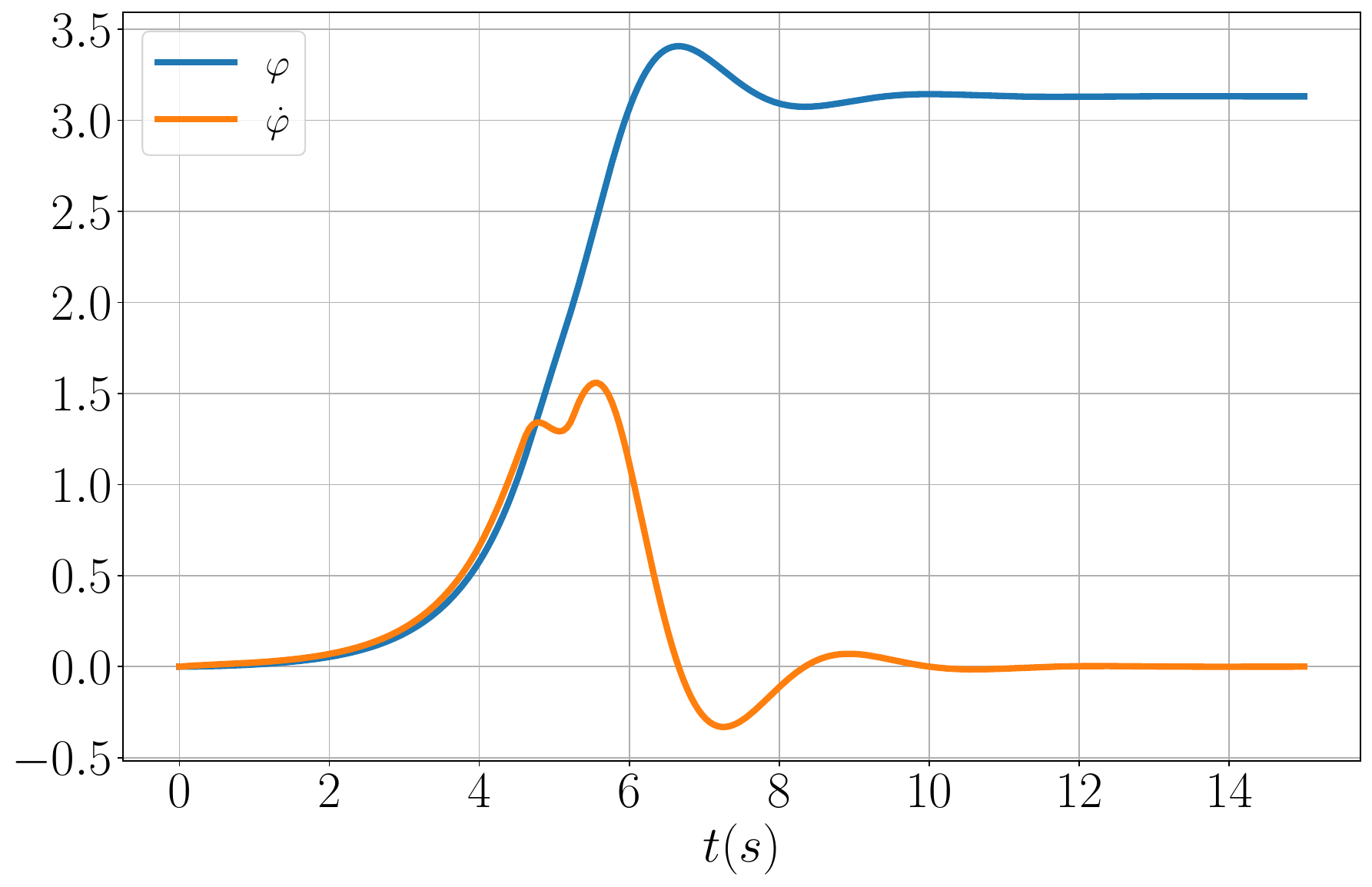}%
        \caption{Angle $\varphi$ and velocity $\dot{\varphi}$.}
        \label{fig:pend_control_theta_thetadot}
\end{subfigure}%
\hfill%
\begin{subfigure}[t]{0.23\textwidth}
        \centering
\includegraphics[width=\textwidth]{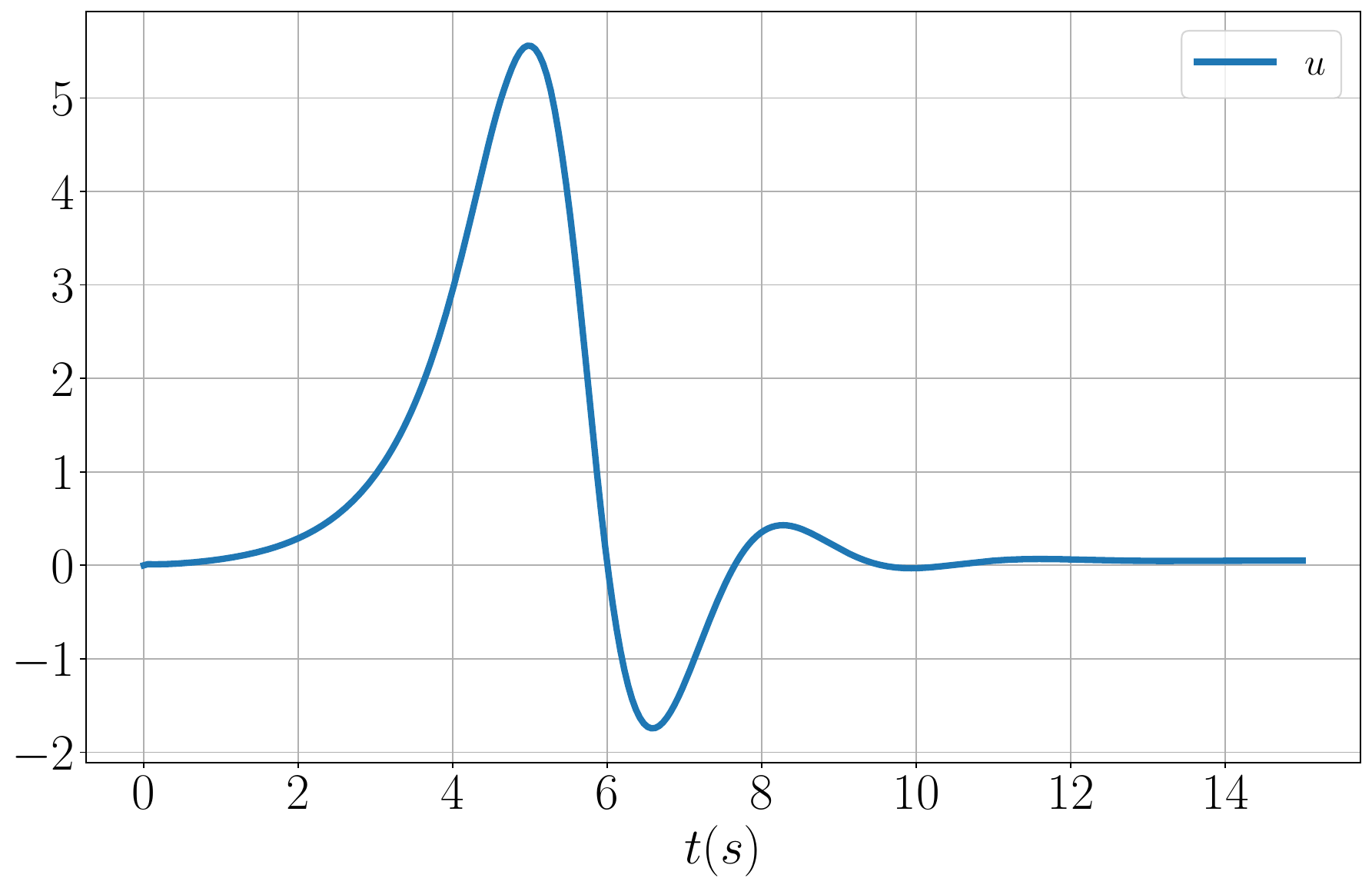}%
        \caption{Control input $u$.}
        \label{fig:pend_control_input}
\end{subfigure}%
\caption{Pendulum dynamics estimation using an $SO(3)$ port-Hamiltonian neural ODE network with scale factor $\beta = 1.33$.}
\label{fig:pend_exp}
\end{figure}

\begin{figure}[t]
\begin{subfigure}[t]{0.238\textwidth}
        \centering
        \includegraphics[width=\textwidth]{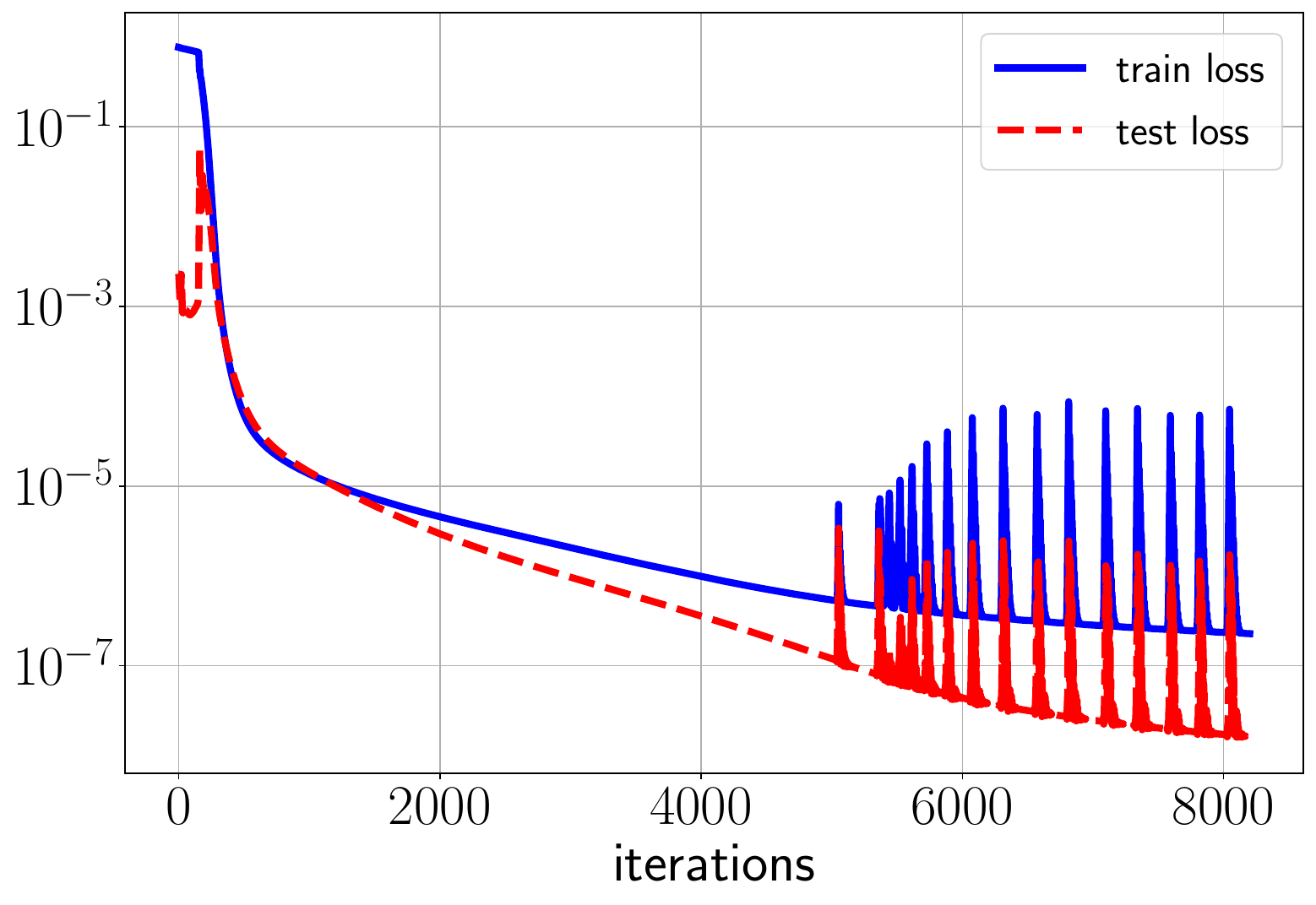}
        \caption{Loss (log scale)}
        \label{fig:car_loss}
\end{subfigure}%
\hfill
\begin{subfigure}[t]{0.245\textwidth}
        \centering
        \includegraphics[width=\textwidth]{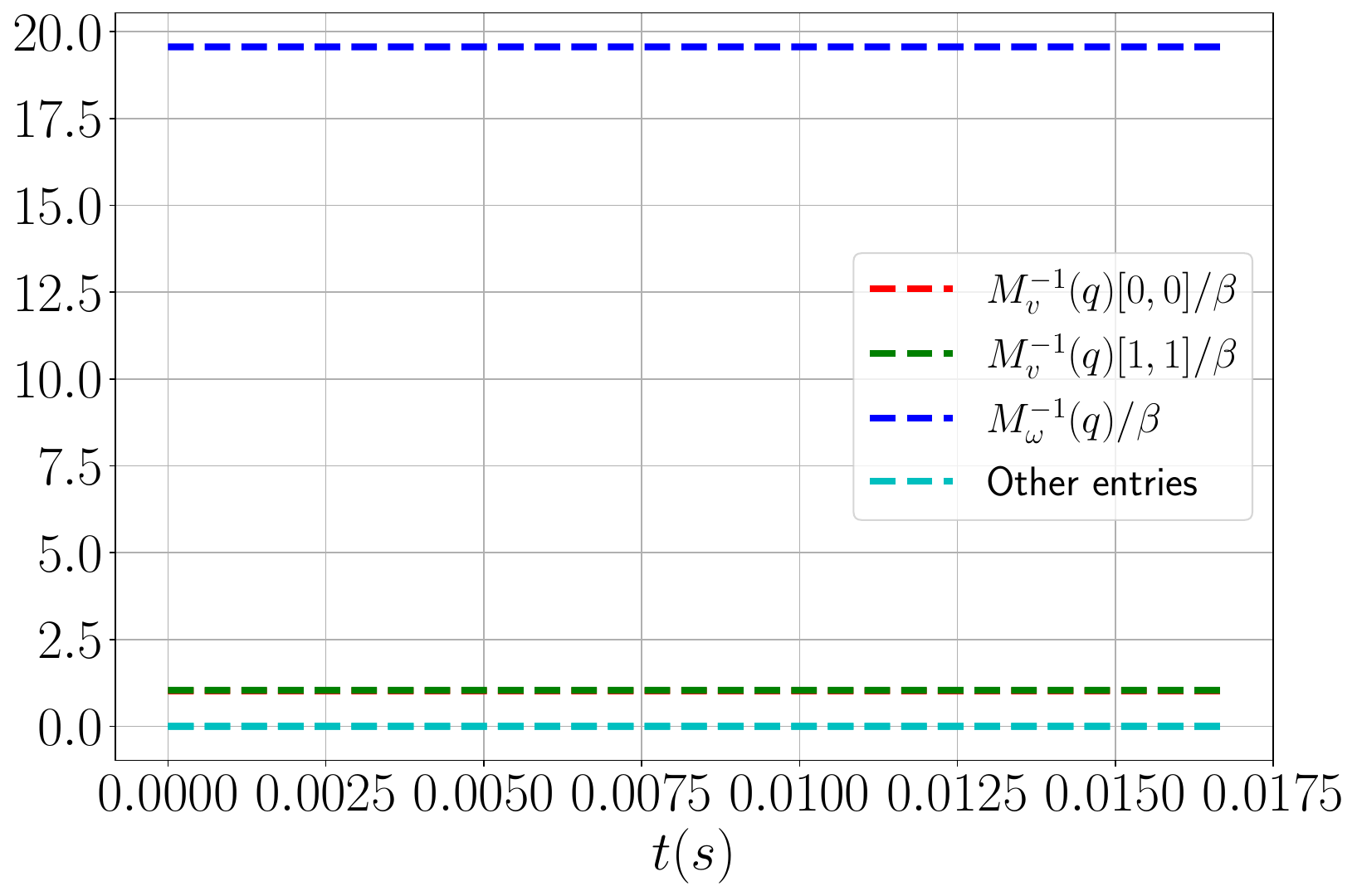}
        \caption{$\bfM_\bfv^{-1}(\frakq)$ and $\bfM_{\bfomega}^{-1}(\frakq)$}
        \label{fig:car_M_x_all}
\end{subfigure}%

\begin{subfigure}[t]{0.248\textwidth}
        \centering
\includegraphics[width=\textwidth]{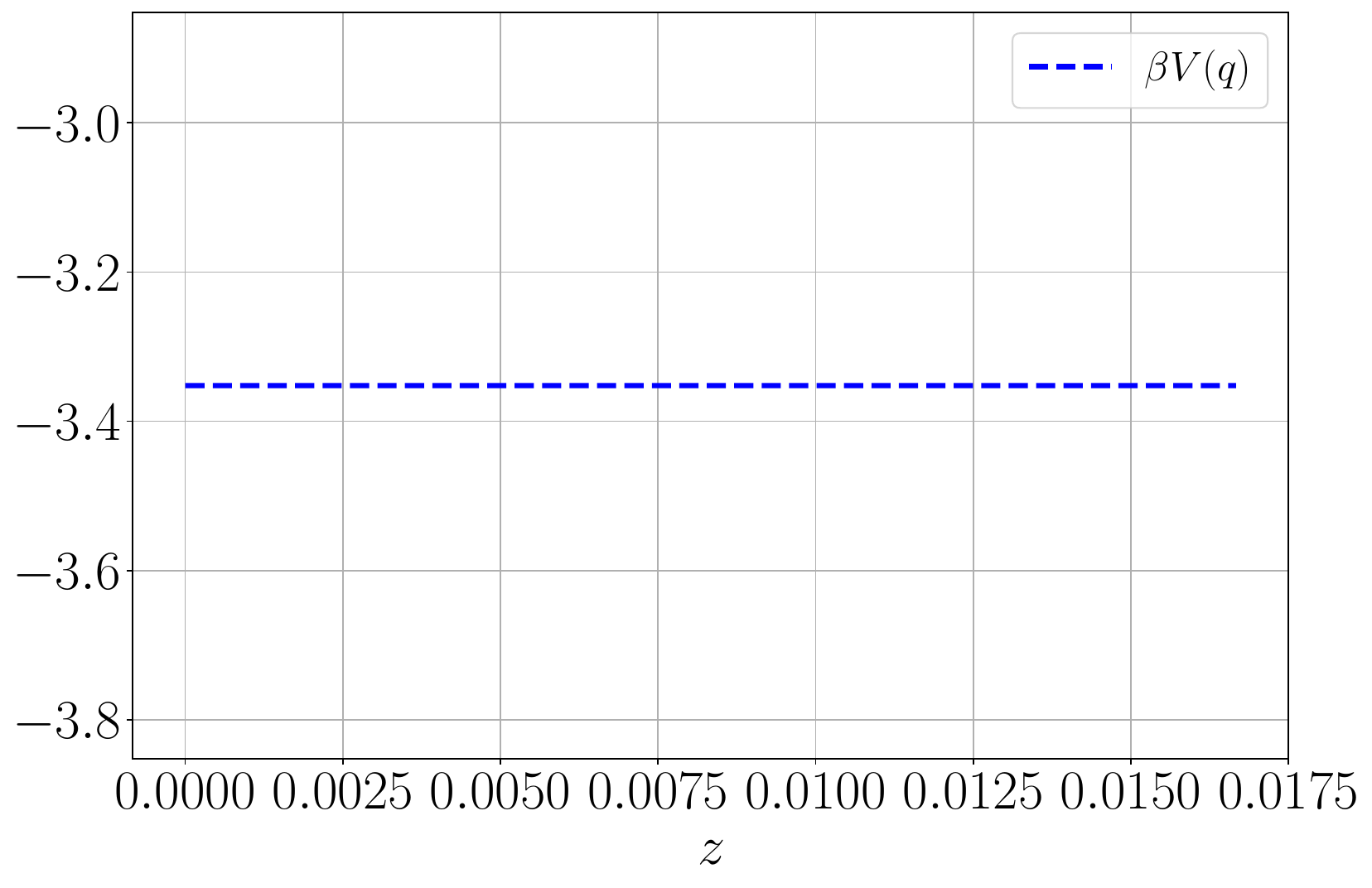}%
        \caption{$\calV(\frakq)$}
        \label{fig:car_Vx}
\end{subfigure}%
\hfill
\begin{subfigure}[t]{0.24\textwidth}
        \centering
        \includegraphics[width=\textwidth]{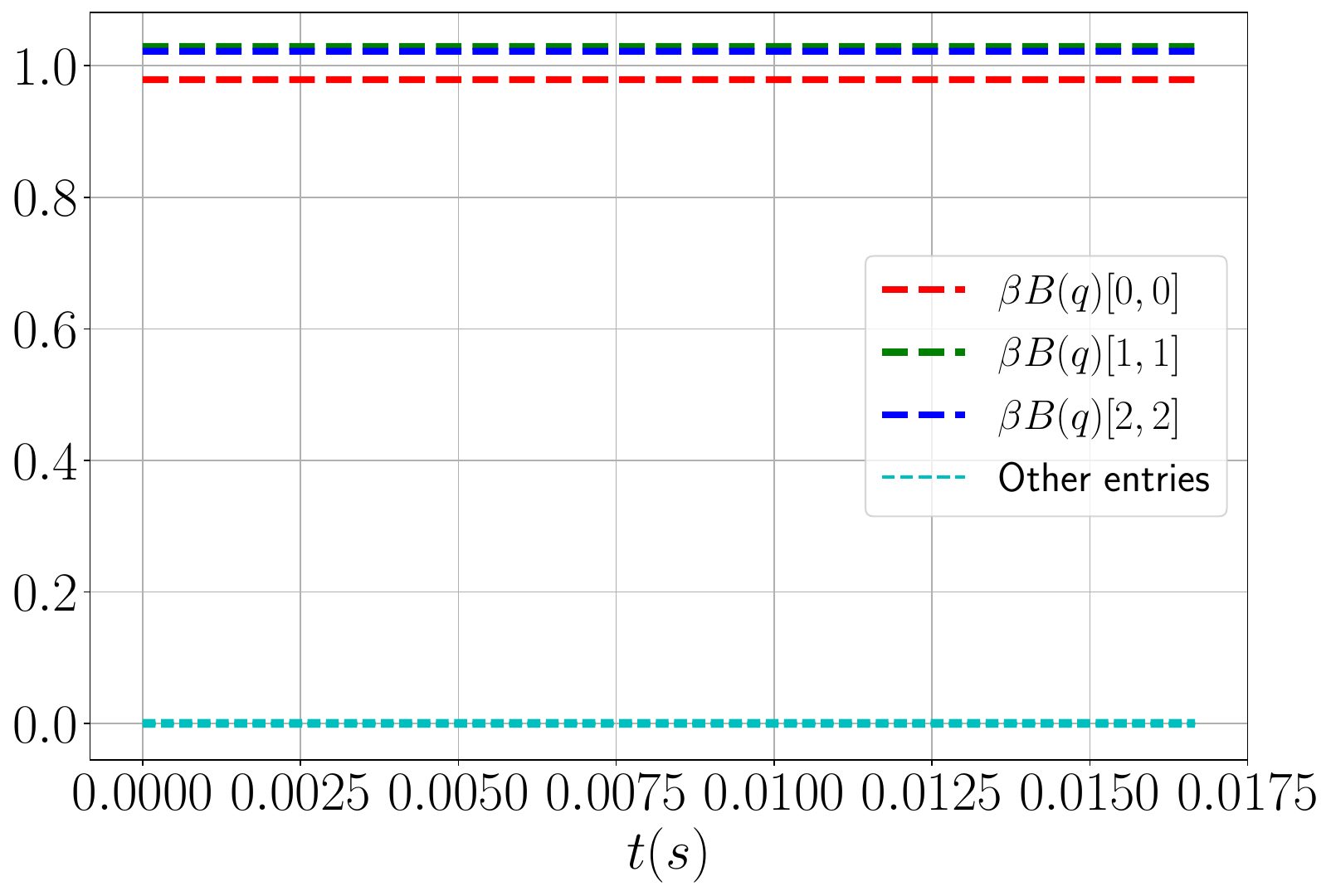}
        \caption{$\bfB_{\bfv}(\frakq)$}
        \label{fig:car_g_x_all}
\end{subfigure}%

\begin{subfigure}[t]{0.235\textwidth}
        \centering
\includegraphics[width=\textwidth]{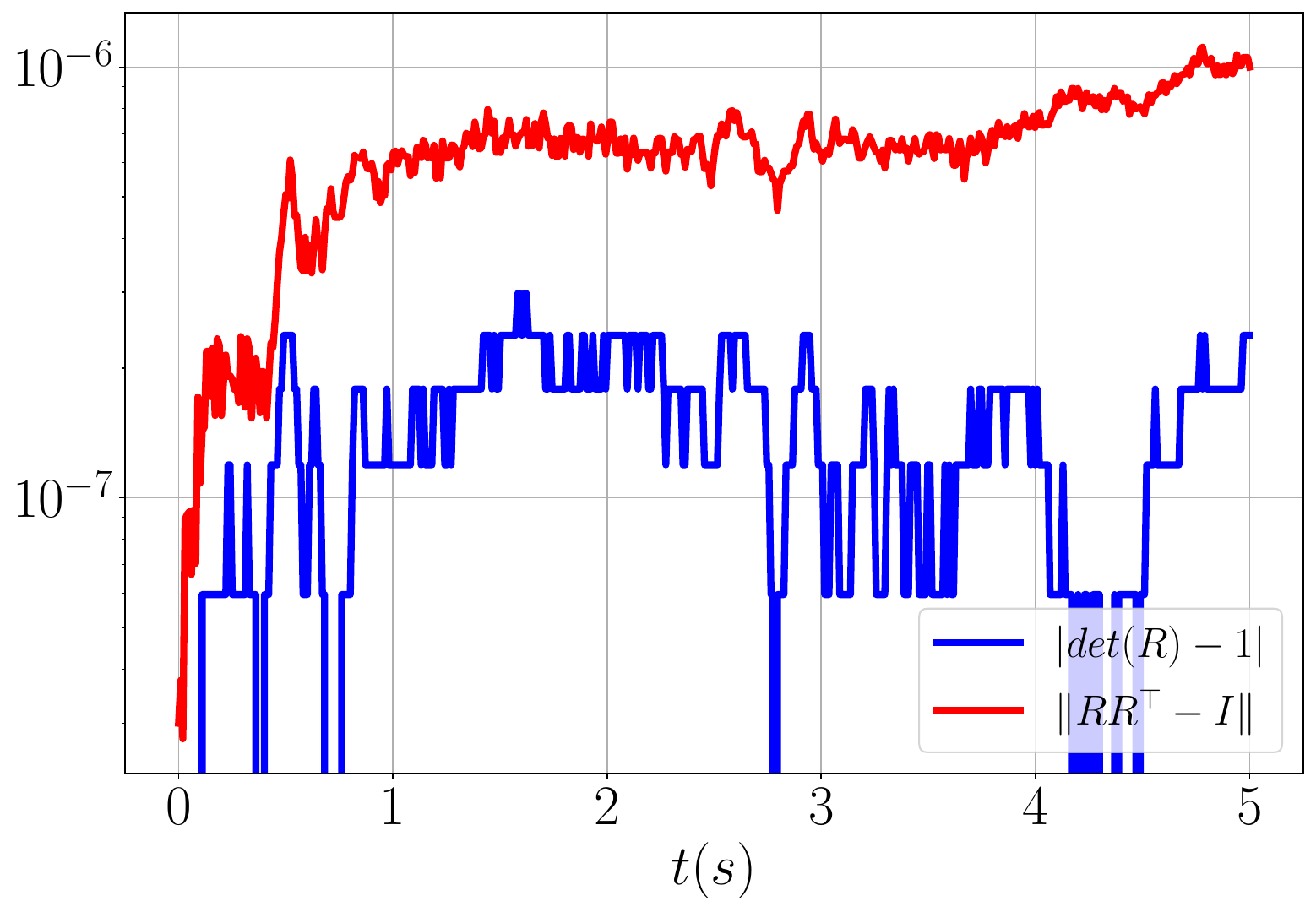}%
        \caption{SO(3) constraints.}
        \label{fig:car_so3_constraints}
\end{subfigure}%
\hfill
\begin{subfigure}[t]{0.25\textwidth}
        \centering
\includegraphics[width=\textwidth]{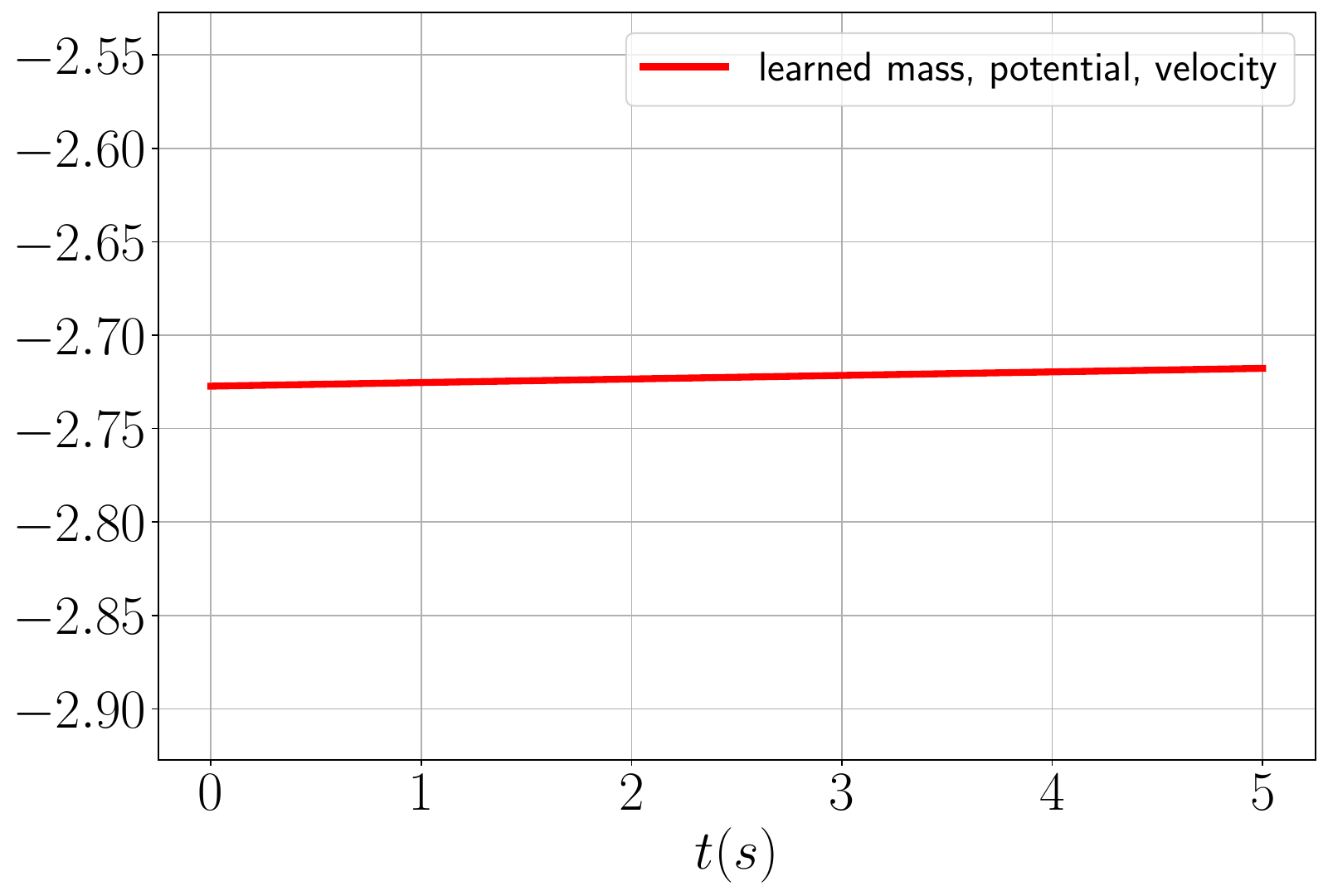}%
        \caption{Total energy.}
        \label{fig:car_total_energy}
\end{subfigure}%
\caption{\NEWW{Evaluation of $SE(2)$ port-Hamiltonian neural ODE network on a simulated omnidirectional ground vehicle with scale factor $\beta = 7.1$.}}
\label{fig:car_exp}
\end{figure}

Fig. \ref{fig:pend_exp} shows the training and testing behavior of our $SO(3)$ Hamiltonian ODE network.
Fig. \ref{fig:pend_M_x_all} and \ref{fig:pend_gx} show that the $\brl{\bfM(\mathbf\frakq)^{-1}}_{3,3}$ entry of the mass inverse and the $\brl{\bfB(\mathbf\frakq)}_3$ entry of the input matrix with \NEWW{scaling factor $\beta = 1.33$} are close to their correct values of $3$ and $1$, respectively, while the other entries are close to zero. Fig. \ref{fig:pend_Vx} indicates a constant gap between the learned and the ground-truth potential energy, which can be explained by the relativity of potential energy. 

We tested stabilization of the pendulum based on the learned dynamics to the unstable equilibrium at the upward position $\varphi = \pi$, with zero velocity.
Since the pendulum is a fully-actuated system, the energy-based controller in \eqref{eq:idapbc_pose_twist_tracking} exists and is obtained by removing the position error from the desired energy:
\begin{equation}
\mathcal{H}_d(\mathbf\frakq_e, \mathbf\frakp_e) = \frac{1}{2} \tr(\bfK_{\bfR}(\bfI - \bfR^{*\top}\bfR)) + \frac{1}{2}\frakp_{\bfomega}^\top\bfM^{-1}(\mathbf\frakq)\frakp_{\bfomega}.
\end{equation}
The controlled angle $\varphi$ and angular velocity $\dot{\varphi}$ as well as the control inputs $\bfu$ with gains $\bfK_\bfR = 2\bfI$ and $\bfK_\bfd = \bfI$ are shown over time in Fig. \ref{fig:pend_control_theta_thetadot} and \ref{fig:pend_control_input}. We can see that the controller was able to smoothly drive the pendulum from $\phi = 0$ to $\phi = \pi$, relying only on the learned dynamics.

\subsection{Omnidirectional Ground Vehicle}
\NEWW{In this section, we verify our port-Hamiltonian dynamics learning and control approach on a simulated omnidirectional ground vehicle, whose states evolve on the $SE(2)$ manifold. The ground-truth dynamics of the vehicle can be obtained from \eqref{eq:portham_dyn_SE3} by keeping only the components $x$ and $y$ in the position $\bfp$ and the yaw angle of the rotation matrix, leading to an $SO(2)$ rotation matrix:
\begin{equation}
\bfR = \begin{bmatrix}
\cos{\varphi} & - \sin{\varphi}\\
\sin{\varphi} & \cos{\varphi}
\end{bmatrix},
\end{equation}
where $\varphi$ is the vehicle's yaw angle. 
The vehicle moves on a flat ground with potential energy $\calV(\frakq) = c$, where $c$ is a constant, and has ground-truth mass $\bfM_\bfv(\frakq) = \bfI \in \bbS_{\succ0}^{2 \times 2}$ and inertia $\bfM_{\bfomega}(\frakq) = 0.05 \in \bbR_{>0}$. It is fully-actuated with control input $\bfu = [f_x, f_y, \tau_\varphi]$, where $f_x$ and $f_y$ are forces along the $x$ and $y$ axes of the body frame and $\tau_\varphi$ is the yaw torque, generated by the motors.

\begin{figure}[t]
\centering
\begin{subfigure}[ht]{0.5\textwidth}
        \centering
\includegraphics[width=\textwidth]{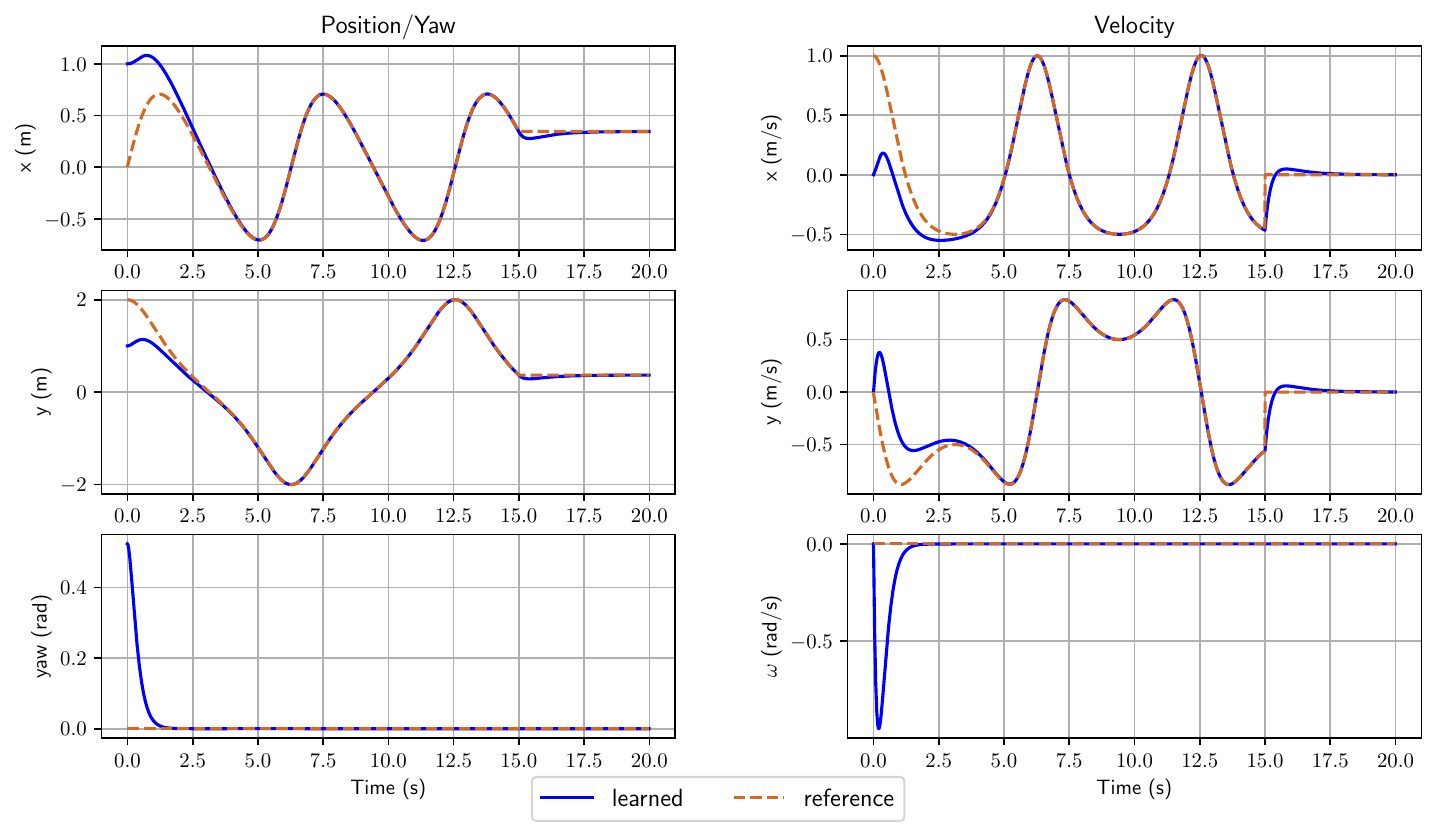}%
\caption{Tracking performance with a lemniscate trajectory.}
\label{fig:car_tracking_results}
\end{subfigure}%

\begin{subfigure}[ht]{0.5\textwidth}
        \centering
\includegraphics[width=0.47\textwidth]{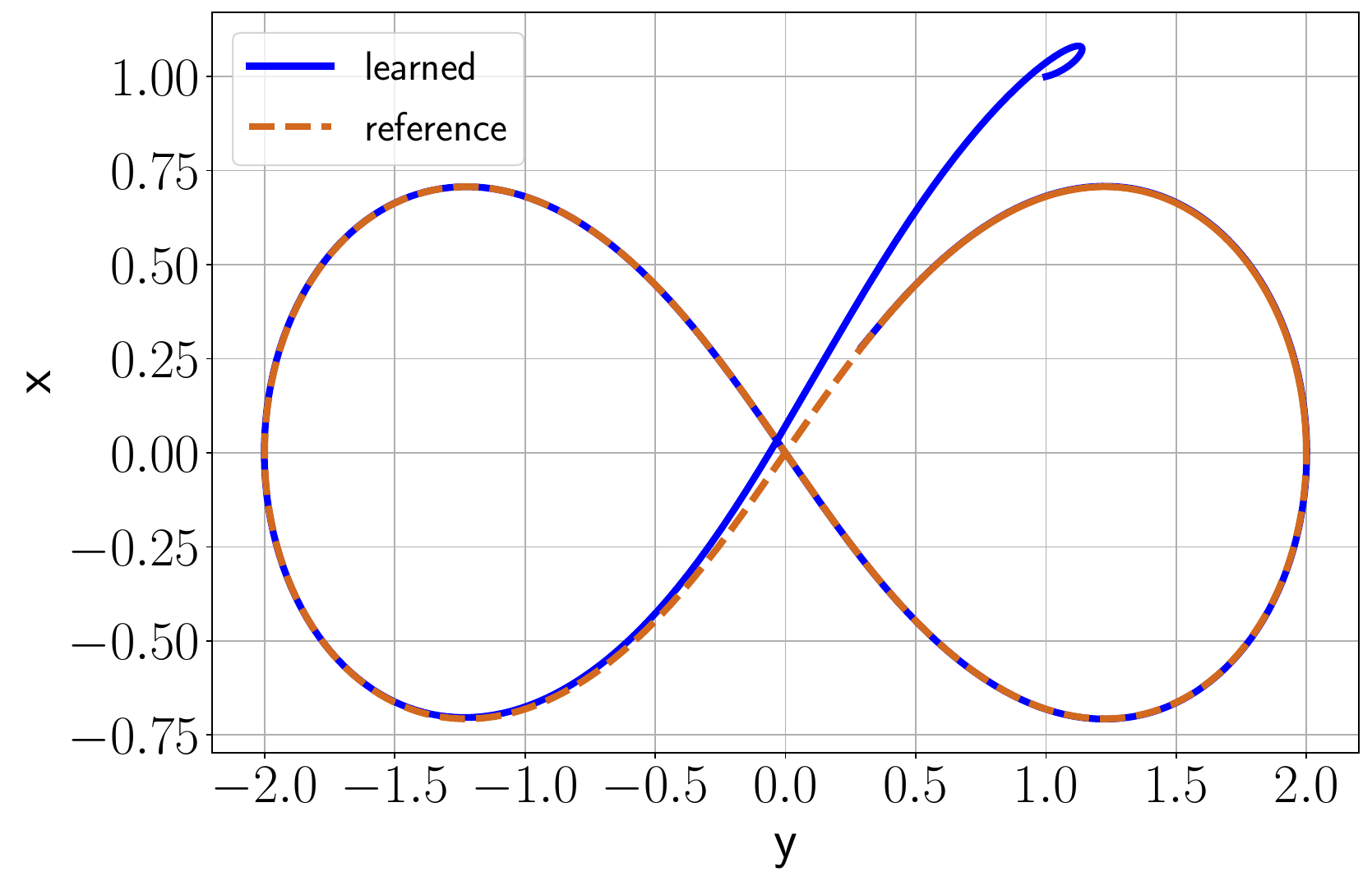}%
\hfill
\includegraphics[width=0.47\textwidth]{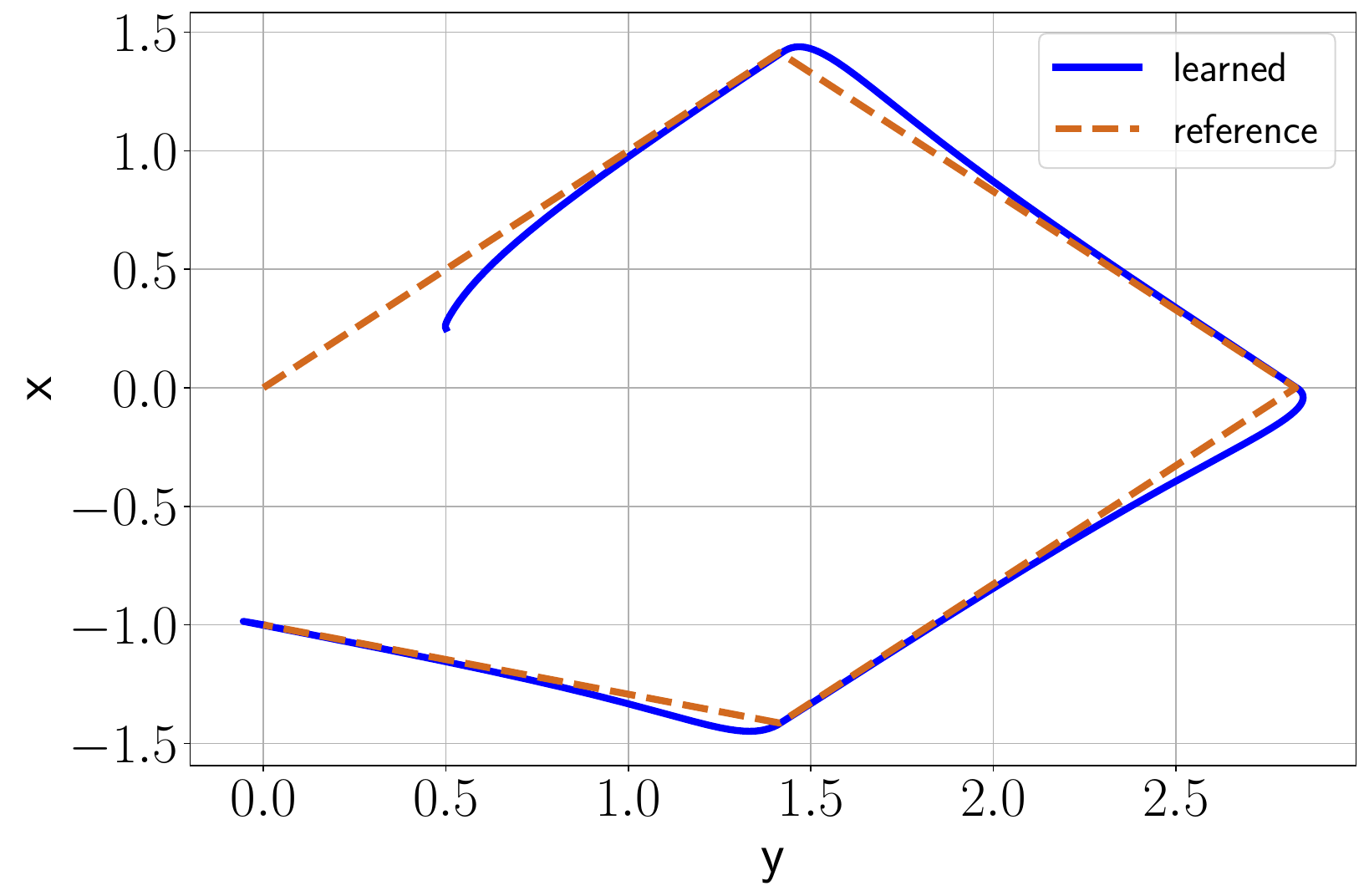}
        \caption{Tracking lemniscate (left) and piecewise-linear (right) trajectories.}
        \label{fig:car_trajviz}
\end{subfigure}%
\caption{\NEWW{Trajectories (blue) of omnidirectional ground vehicle tracking desired trajectories (orange) with our learned $SE(2)$ port-Hamiltonian dynamics and IDA-PBC control design.}}
\label{fig:car_traj_viz}
\end{figure}

\begin{figure}[t]
\begin{subfigure}[t]{0.239\textwidth}
        \centering
        \includegraphics[width=\textwidth]{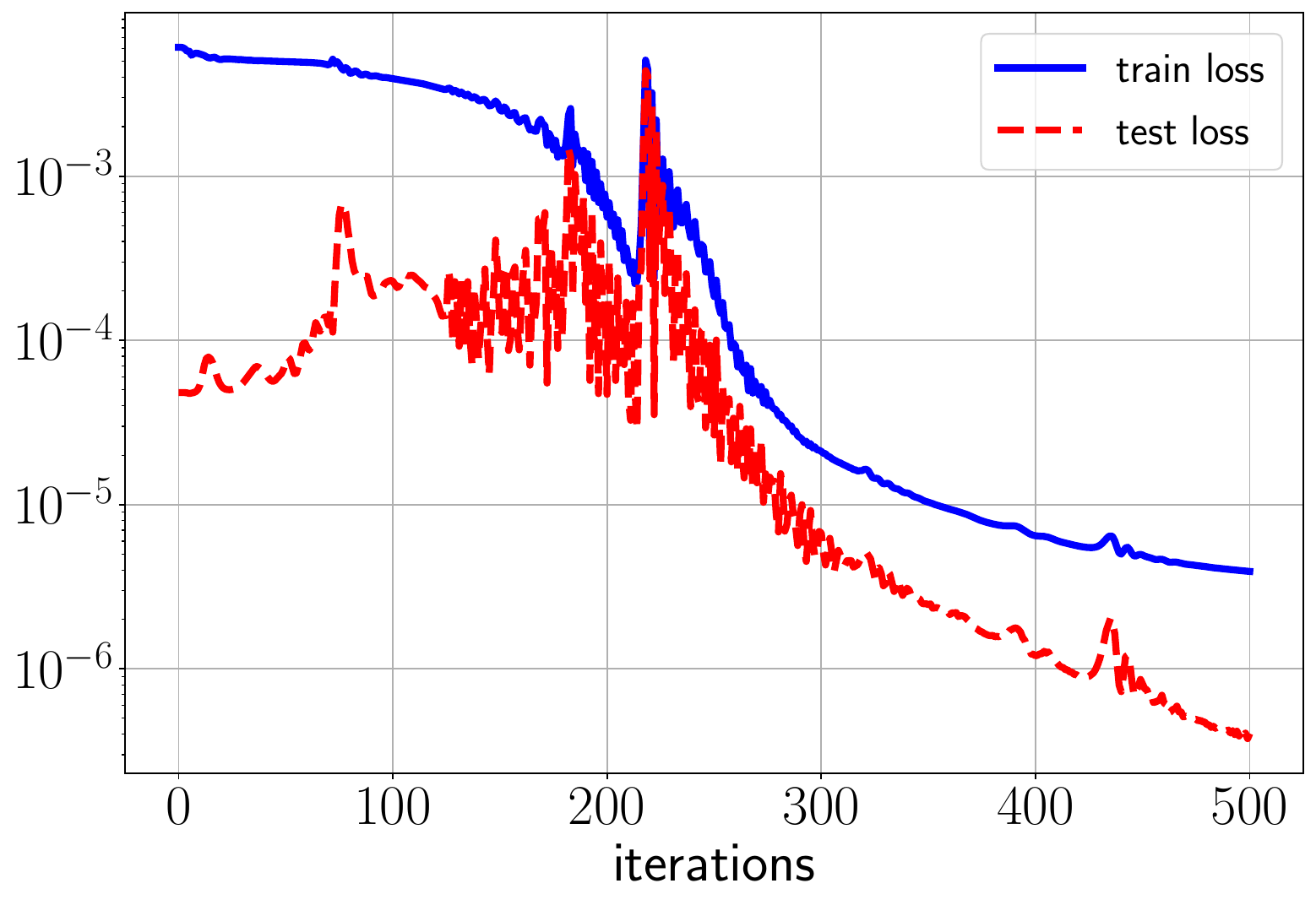}
        \caption{Loss (log scale)}
        \label{fig:pybullet_loss}
\end{subfigure}%
\hfill
\begin{subfigure}[t]{0.241\textwidth}
        \centering
        \includegraphics[width=\textwidth]{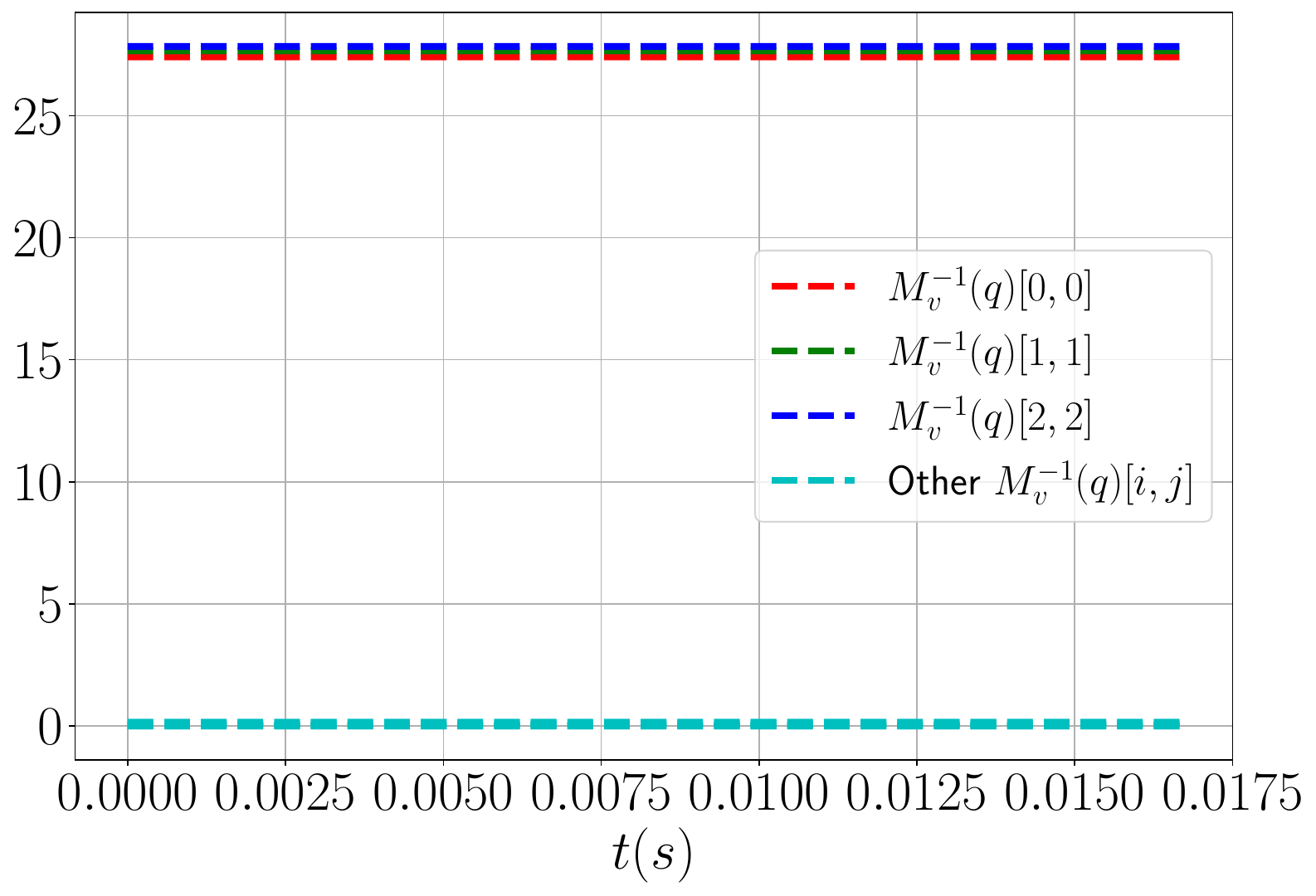}
        \caption{$\bfM_\bfv^{-1}(\frakq)$}
        \label{fig:pybullet_M1_x_all}
\end{subfigure}%

\begin{subfigure}[t]{0.235\textwidth}
        \centering
		\includegraphics[width=\textwidth]{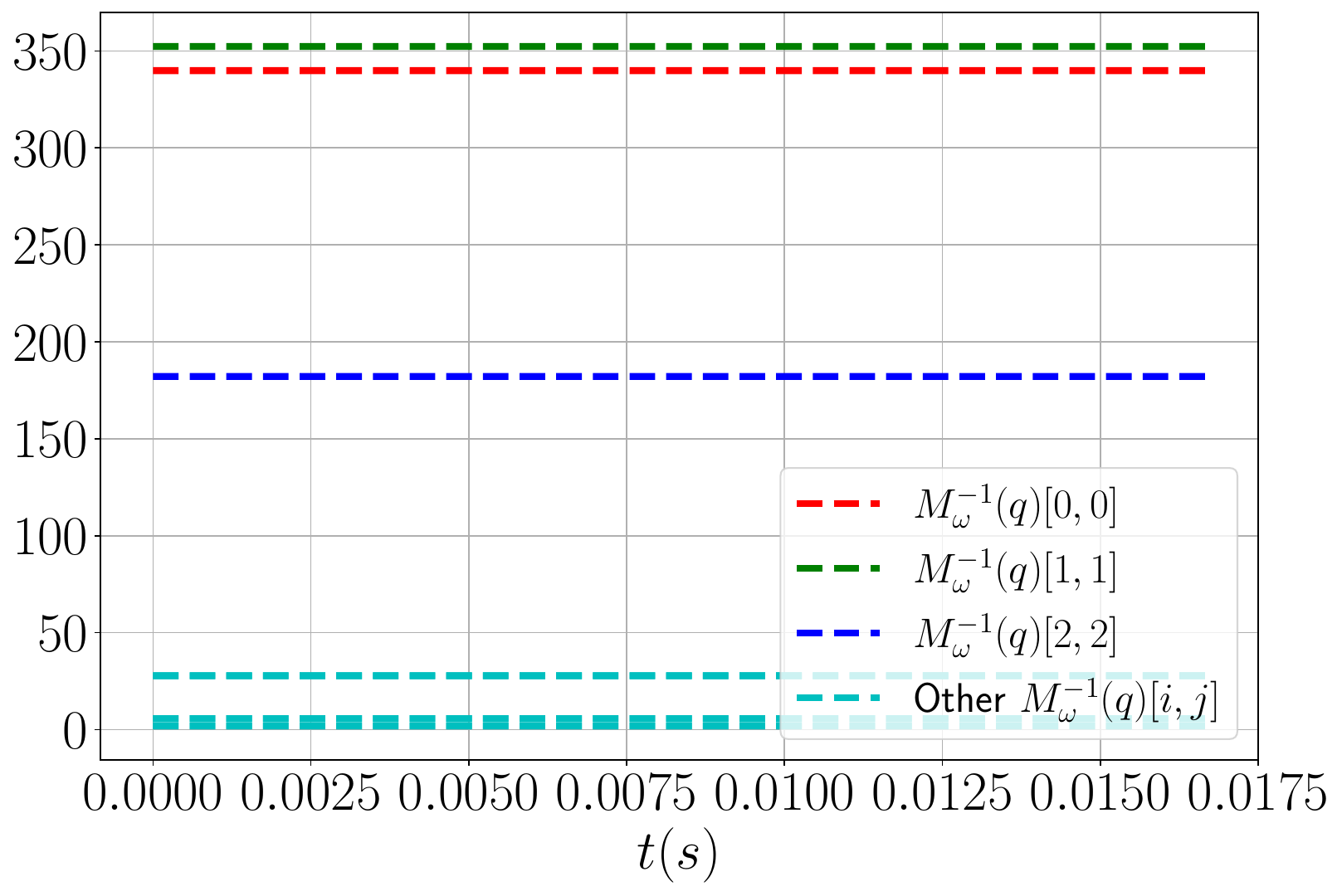}%
        \caption{$\bfM_{\bfomega}^{-1}(\frakq)$}
        \label{fig:pybullet_M2_x_all}
\end{subfigure}%
\hfill
\begin{subfigure}[t]{0.245\textwidth}
        \centering
\includegraphics[width=\textwidth]{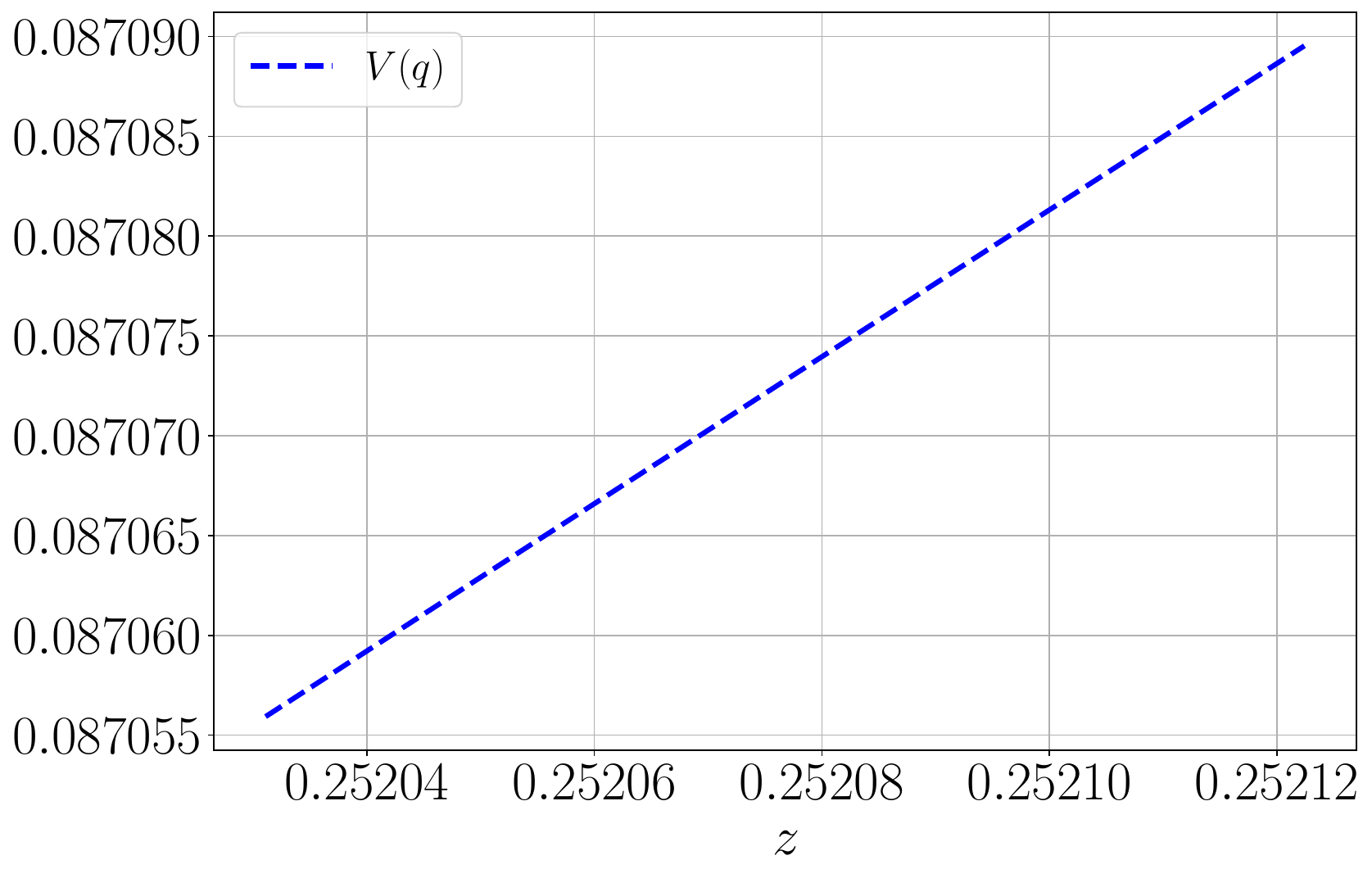}%
        \caption{$\calV(\frakq)$}
        \label{fig:pybullet_Vx}
\end{subfigure}%

\begin{subfigure}[t]{0.24\textwidth}
        \centering
        \includegraphics[width=\textwidth]{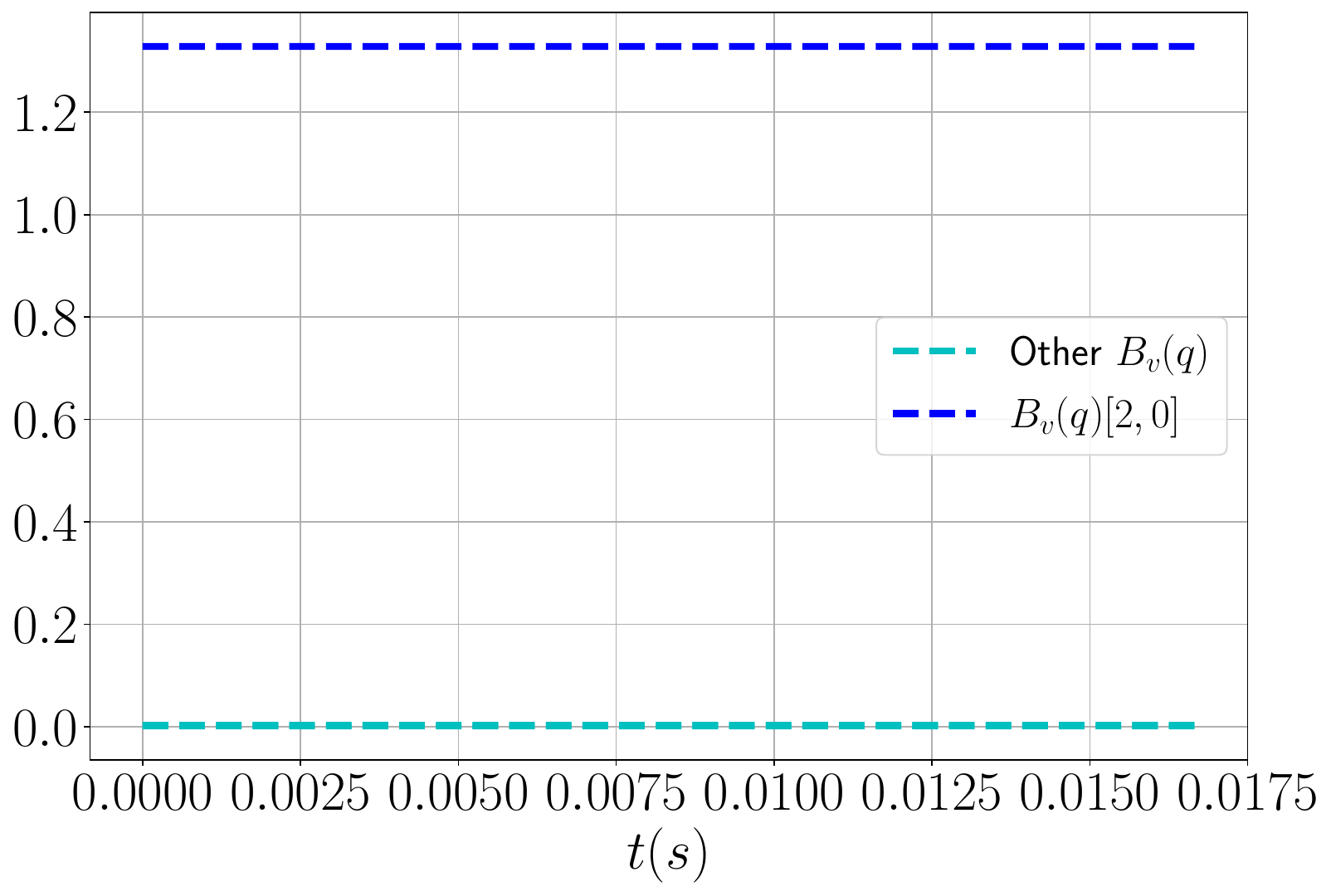}
        \caption{$\bfB_{\bfv}(\frakq)$}
        \label{fig:pybullet_g_v_x_all}
\end{subfigure}%
\hfill
\begin{subfigure}[t]{0.24\textwidth}
        \centering
        \includegraphics[width=\textwidth]{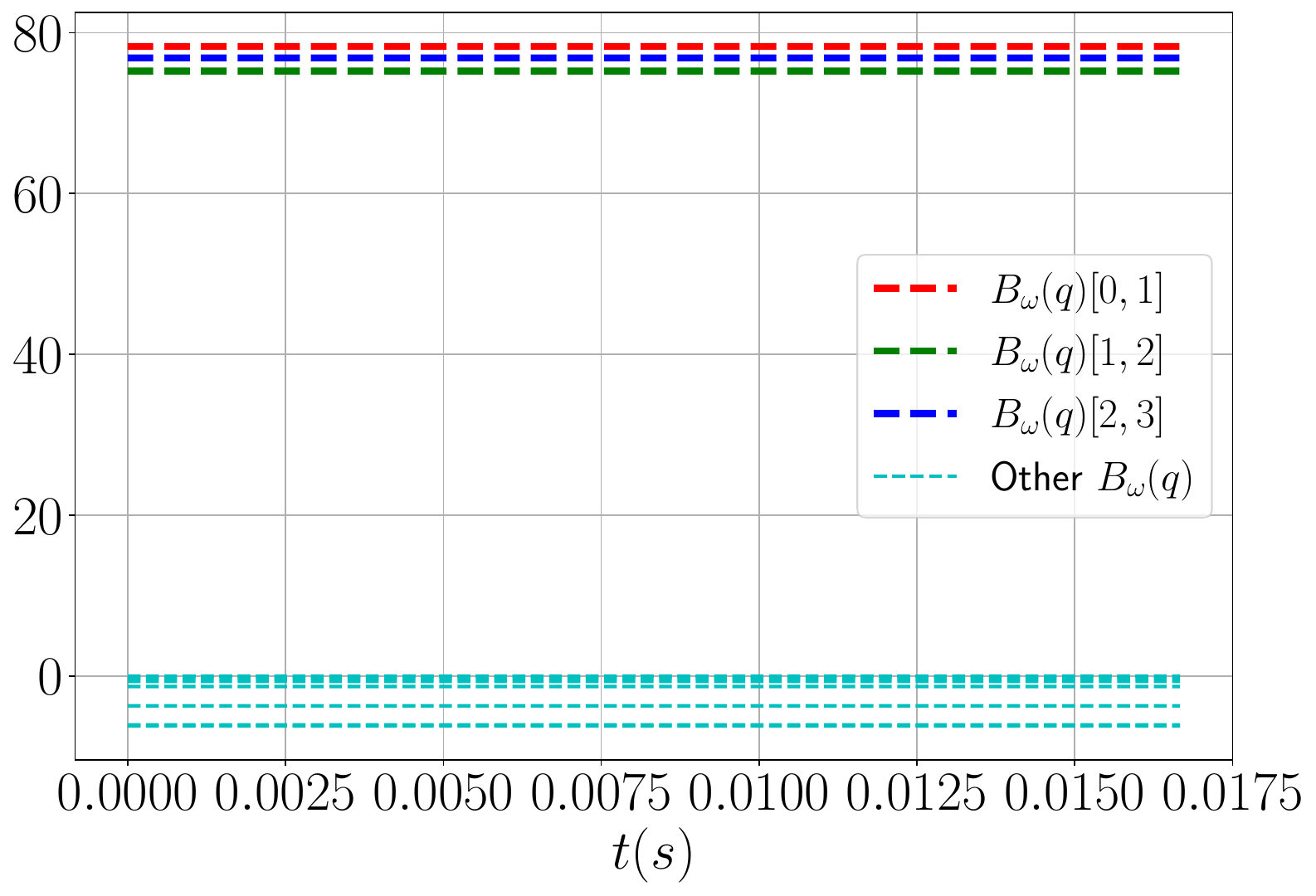}
        \caption{$\bfB_{\bfomega}(\frakq)$}
        \label{fig:pybullet_g_w_x_all}
\end{subfigure}%

\begin{subfigure}[t]{0.235\textwidth}
        \centering
\includegraphics[width=\textwidth]{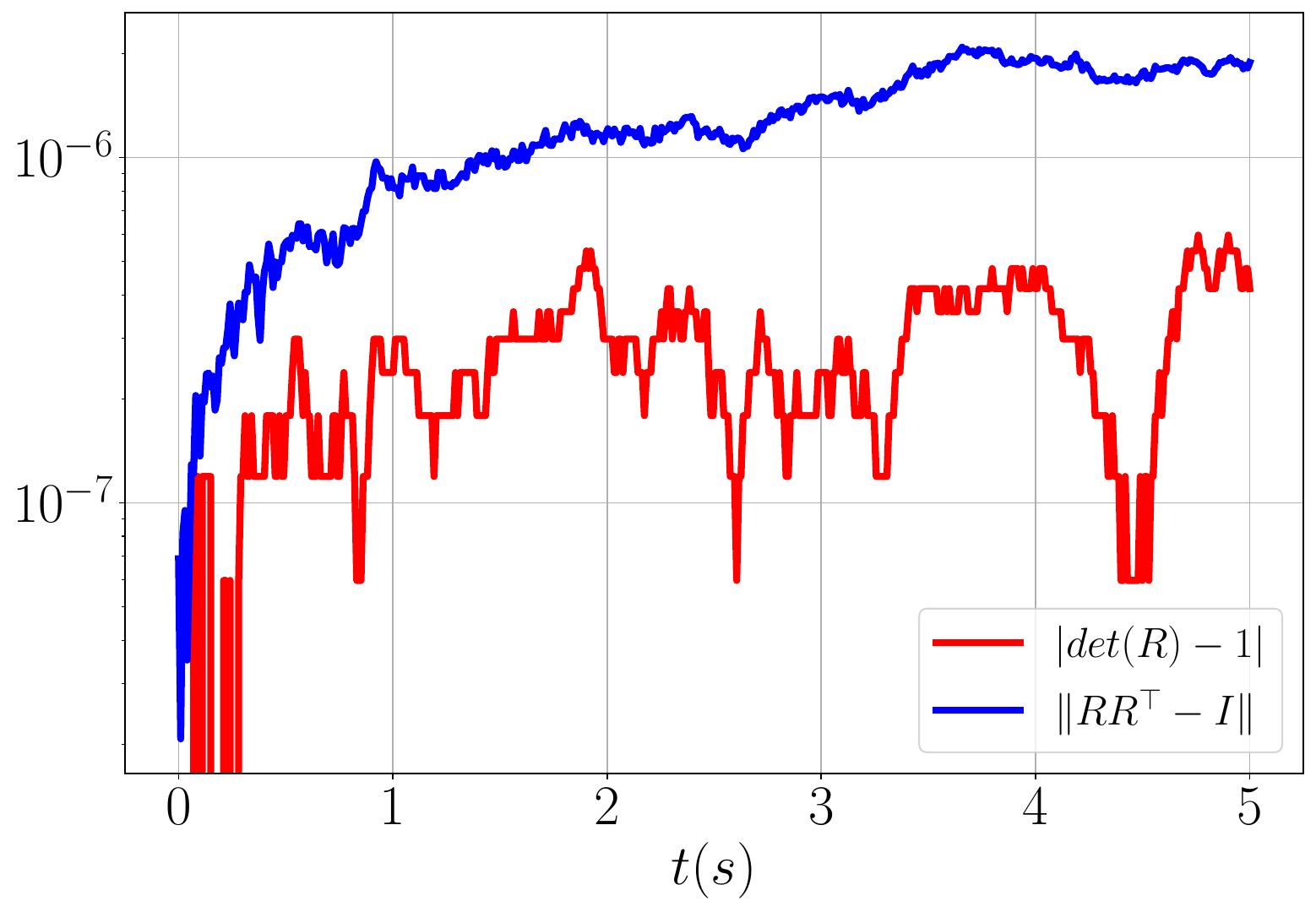}%
        \caption{SO(3) constraints.}
        \label{fig:pybullet_so3_constraints}
\end{subfigure}%
\hfill
\begin{subfigure}[t]{0.25\textwidth}
        \centering
\includegraphics[width=\textwidth]{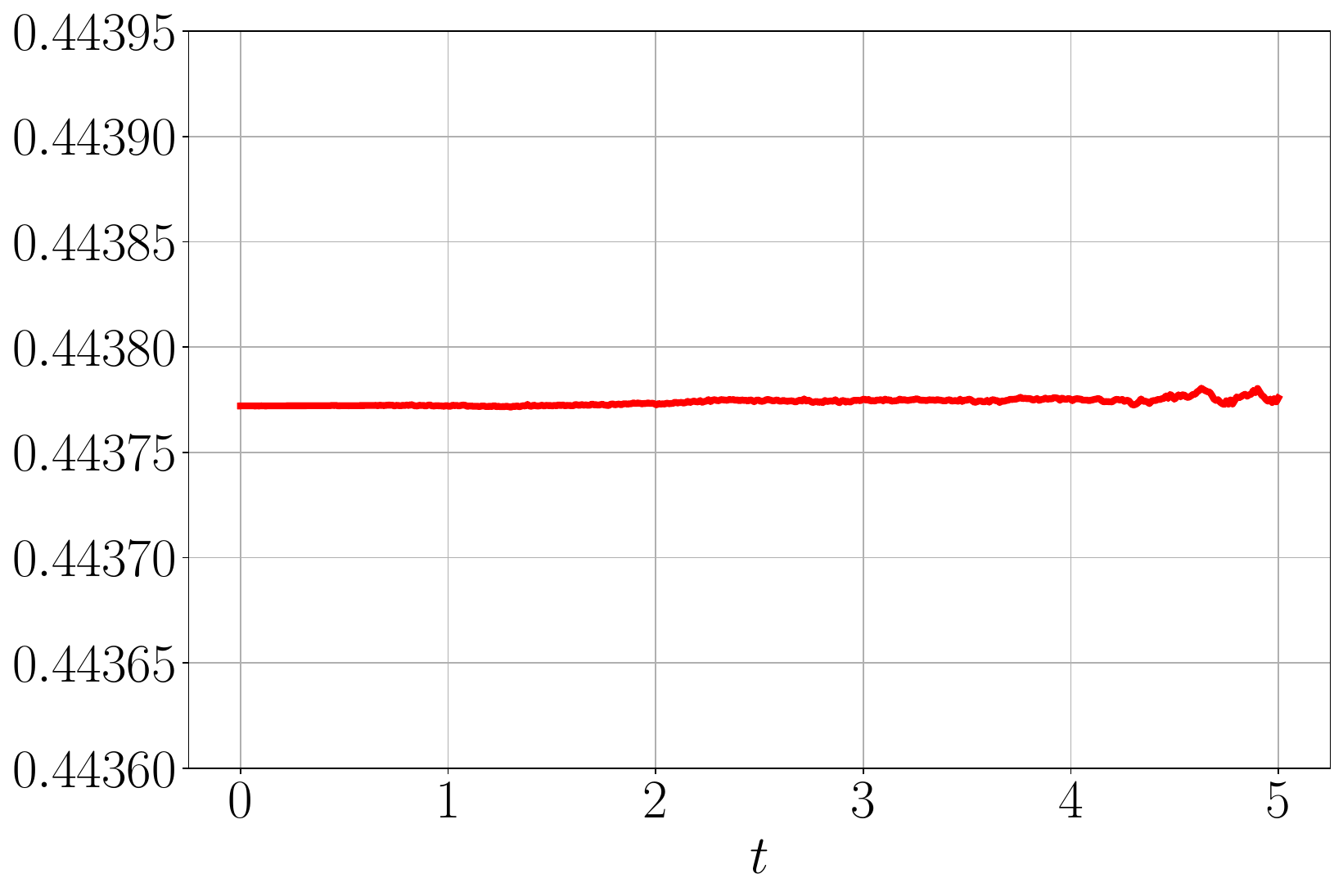}%
        \caption{Total energy.}
        \label{fig:pybullet_total_energy}
\end{subfigure}%
\caption{$SE(3)$ port-Hamiltonian neural ODE network on a Crazyflie quadrotor in the PyBullet simulator \cite{gym-pybullet-drones2020}.}
\label{fig:pybullet_exp}
\end{figure}

To collect training data, the vehicle was controlled from a random initial point to $9$ different desired positions and yaw angles using a PID controller, providing $9$ one-second trajectories. The trajectories were used to generate a dataset $\mathcal{D} = \{t_{0:N}^{(i)},\mathbf\frakq_{0:N}^{(i)}, \bfzeta_{0:N}^{(i)}, \bfu^{(i)})\}_{i=1}^D$ with $N = 5$ and $D = 432$. The $SE(2)$ port-Hamiltonian ODE network was formulated, as described in Sec. \ref{subsec:SE3_dyn_learning}, by ignoring the $z$ component of the position $\bfp$, and the pitch and roll components of the rotation $\bfR$. The model was trained for $8000$ iterations without a nominal model, i.e., $\bfM^{-1}_{\bfv 0}(\mathbf\frakq) = \bf0$, $\bfM^{-1}_{\bfomega 0}(\mathbf\frakq) = \bf0$, $\bfD_{\bfv 0}(\frakq, \frakp) = \bf0$, $\bfD_{\bfomega 0}(\frakq, \frakp) = \bf0$, $V_0(\frakq) = \bf0$ and $\bfB_0(\frakq) = \bf0$. We did not consider energy dissipation such as friction in the simulation, and we omitted the dissipation matrix $\bfD_{\bftheta}(\frakq, \frakp)$ in the model.

Fig. \ref{fig:car_exp} shows the training results for the $SE(2)$ Hamiltonian ODE network. Fig. \ref{fig:car_M_x_all} and \ref{fig:car_g_x_all} show that the mass inverse and the input gain matrix with scaling factor $\beta = 7.1$ are close to their correct values: $\bfM_\bfv(\mathbf\frakq)^{-1} \approx \bfI, \bfM_{\bfomega}(\mathbf\frakq)^{-1} \approx 20$ and $\bfB(\mathbf\frakq) \approx \bfI$. Fig. \ref{fig:car_Vx} shows a constant learned potential energy as expected. 

We verified our energy-based control design in Sec. \ref{sec:controller_design} by controlling the ground robot to track horizontal lemniscate and piecewise-linear trajectories. Fig. \ref{fig:car_trajviz} demonstrates that the omnidirectional ground vehicle controlled by our energy-based controller achieves successful trajectory tracking. The control gains were chosen as: $\bfK_\bfp = 0.72\bfI$, $\bfK_\bfv = 0.8\bfI, \bfK_{\bfR} = 9.1\bfI$, $\bfK_{\bfomega} = 3.6$. Fig. \ref{fig:car_tracking_results} plots the tracking errors in the position, yaw angles, linear velocity, and angular velocity.
}

\subsection{Crazyflie Quadrotor}
\label{subsec:crazieflie_quad}

In this section, we demonstrate that our $SE(3)$ dynamics learning and control approach can achieve trajectory tracking for an underactuated system. We consider a Crazyflie quadrotor, shown in Fig. \ref{fig:pybullet_crazyflie}, simulated in the physics-based simulator PyBullet \cite{gym-pybullet-drones2020}. The control input $\bfu = [f, \bftau]$ includes a thrust $f\in \mathbb{R}_{\geq 0}$ and a torque vector $\bftau \in \mathbb{R}^3$ generated by the $4$ rotors. The generalized coordinates and velocity are $\mathbf\frakq = [\bfp^\top \quad \bfr_1^\top \quad \bfr_2^\top \quad \bfr_3^\top]^\top$ and $\bfzeta = [\bfv^\top \quad \bfomega^\top]^\top$ as before. 

The quadrotor was controlled from a random starting point to $18$ different desired poses using a PID controller \cite{gym-pybullet-drones2020}, providing $18$ $2.5$-second trajectories. The trajectories were used to generate a dataset $\mathcal{D} = \{t_{0:N}^{(i)},\mathbf\frakq_{0:N}^{(i)}, \bfzeta_{0:N}^{(i)}, \bfu^{(i)})\}_{i=1}^D$ with $N = 5$ and $D = 1080$. The $SE(3)$ port-Hamiltonian ODE network was trained, as described in Sec. \ref{subsec:SE3_dyn_learning}, for $500$ iterations without a nominal model, i.e., $\bfM^{-1}_{\bfv 0}(\mathbf\frakq) = \bf0$, $\bfM^{-1}_{\bfomega 0}(\mathbf\frakq) = \bf0$, $\bfD_{\bfv 0}(\frakq, \frakp) = \bf0$, $\bfD_{\bfomega 0}(\frakq, \frakp) = \bf0$, $V_0(\frakq) = \bf0$ and $\bfB_0(\frakq) = \bf0$. We did not consider energy dissipation such as drag effect in the PyBullet simulator, and we omitted the dissipation matrix $\bfD_{\bftheta}(\frakq, \frakp)$ in the model design.

\begin{figure}[t]
\centering
\includegraphics[width=0.5\textwidth]{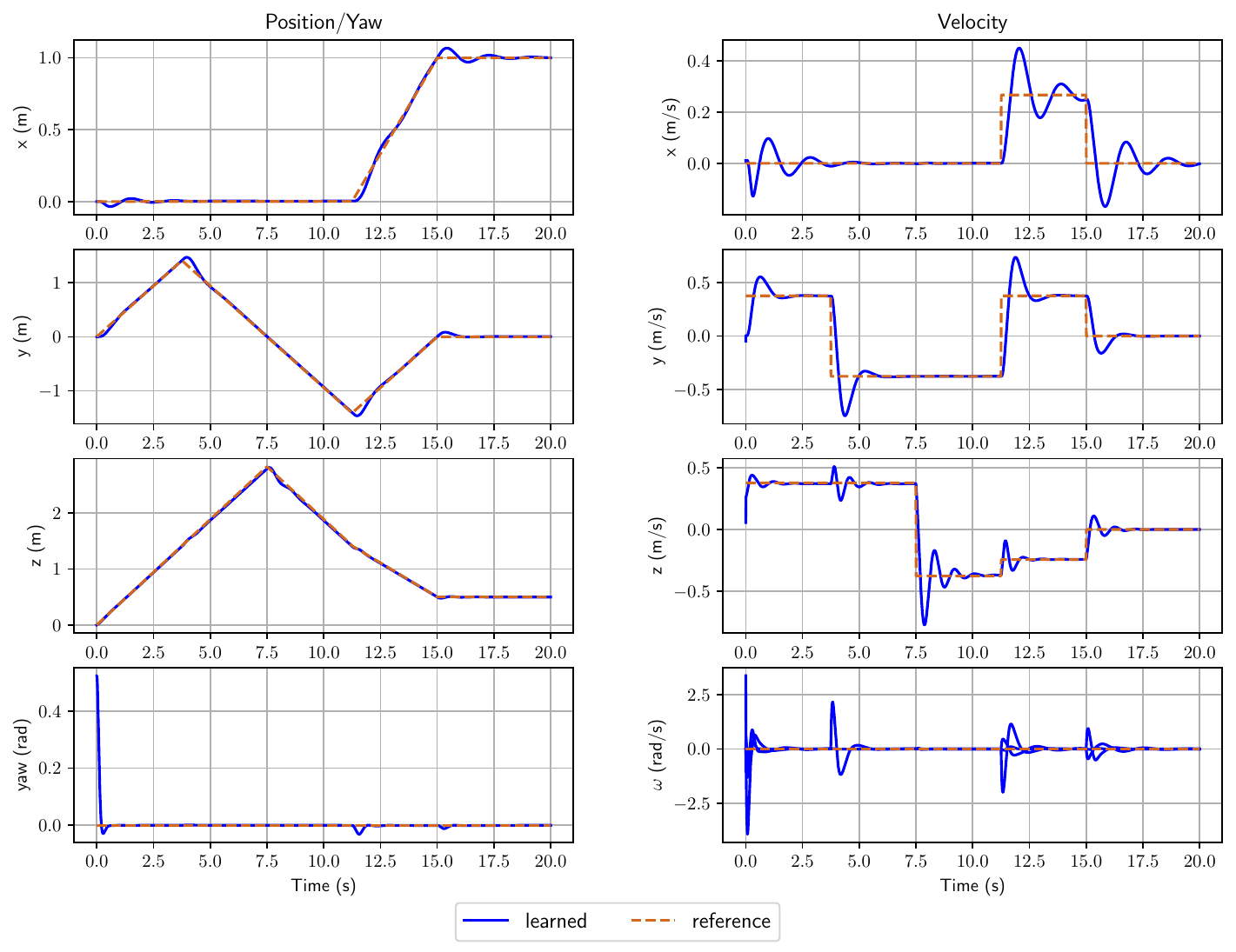}%
\caption{Crazyflie quadrotor trajectory (blue) tracking a desired diamond-shaped trajectory (orange) shown in Fig.~\ref{fig:pybullet_traj_viz}.}
\label{fig:pybullet_tracking_results}
\end{figure}

\begin{figure}[t]
\centering
\begin{subfigure}[ht]{0.25\textwidth}
        \centering
\includegraphics[width=\textwidth]{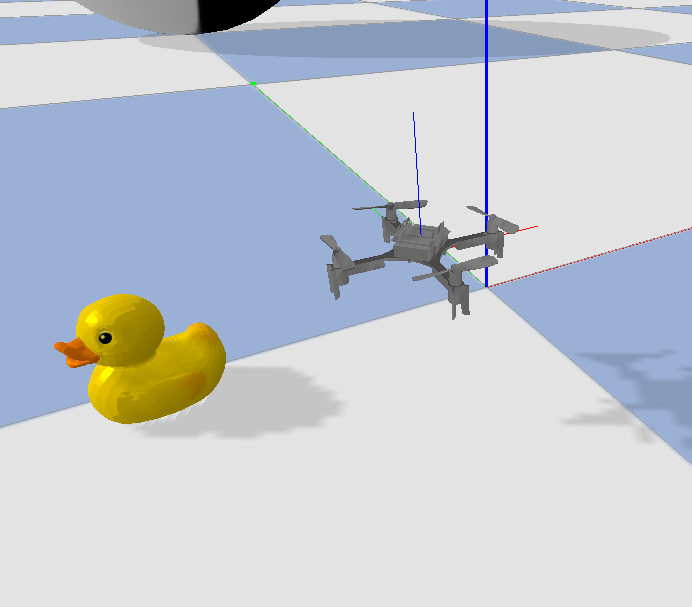}%
        \caption{Crazyflie simulator}
        \label{fig:pybullet_crazyflie}
\end{subfigure}%
\begin{subfigure}[ht]{0.25\textwidth}
        \centering
\includegraphics[width=\textwidth]{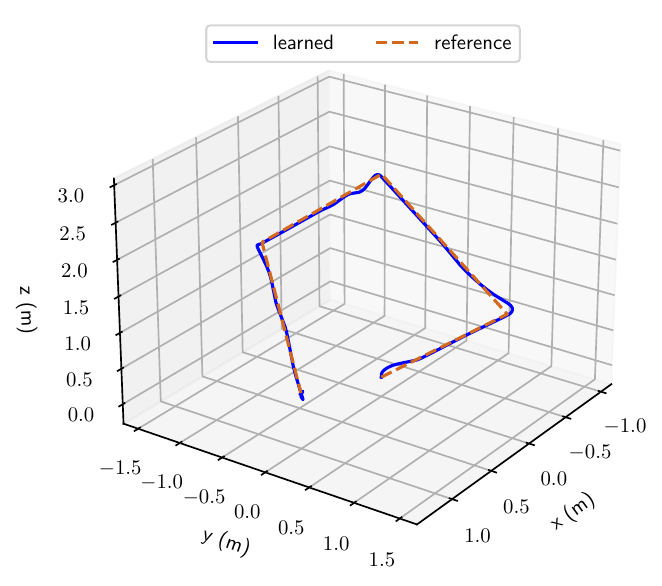}%
        \caption{Trajectory tracking}
        \label{fig:pybullet_trajviz}
\end{subfigure}%
\caption{\NEWW{Trajectory tracking experiment with a Crazyflie quadrotor in the PyBullet simulator \cite{gym-pybullet-drones2020}.}}
\label{fig:pybullet_traj_viz}
\end{figure}

\begin{figure*}[t]
\centering
\begin{subfigure}[t]{0.245\textwidth}
        \centering
\includegraphics[width=\textwidth]{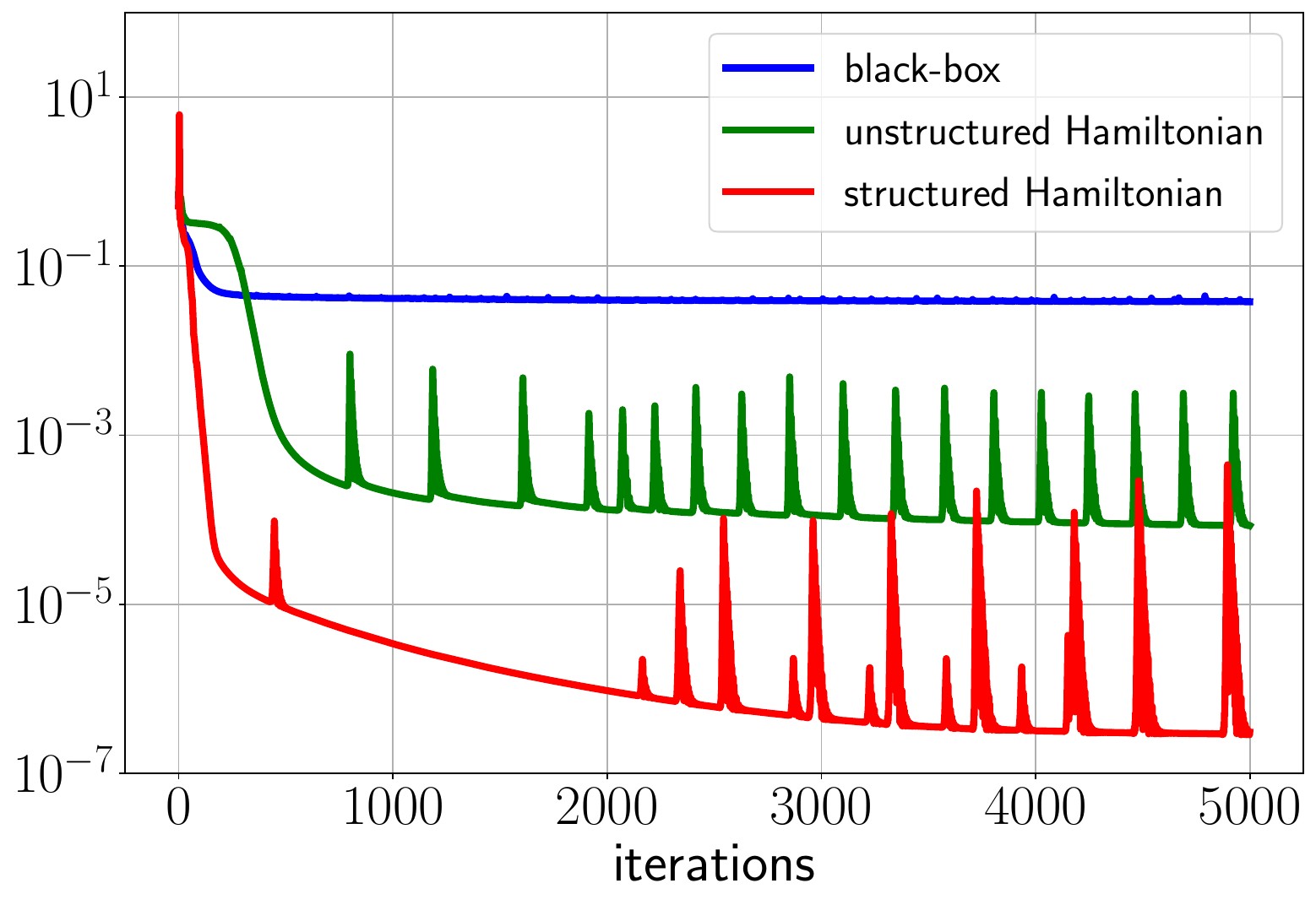}%
        \caption{Training loss}
        \label{fig:pend_training_loss}
\end{subfigure}%
\hfill
\begin{subfigure}[t]{0.25\textwidth}
        \centering
\includegraphics[width=\textwidth]{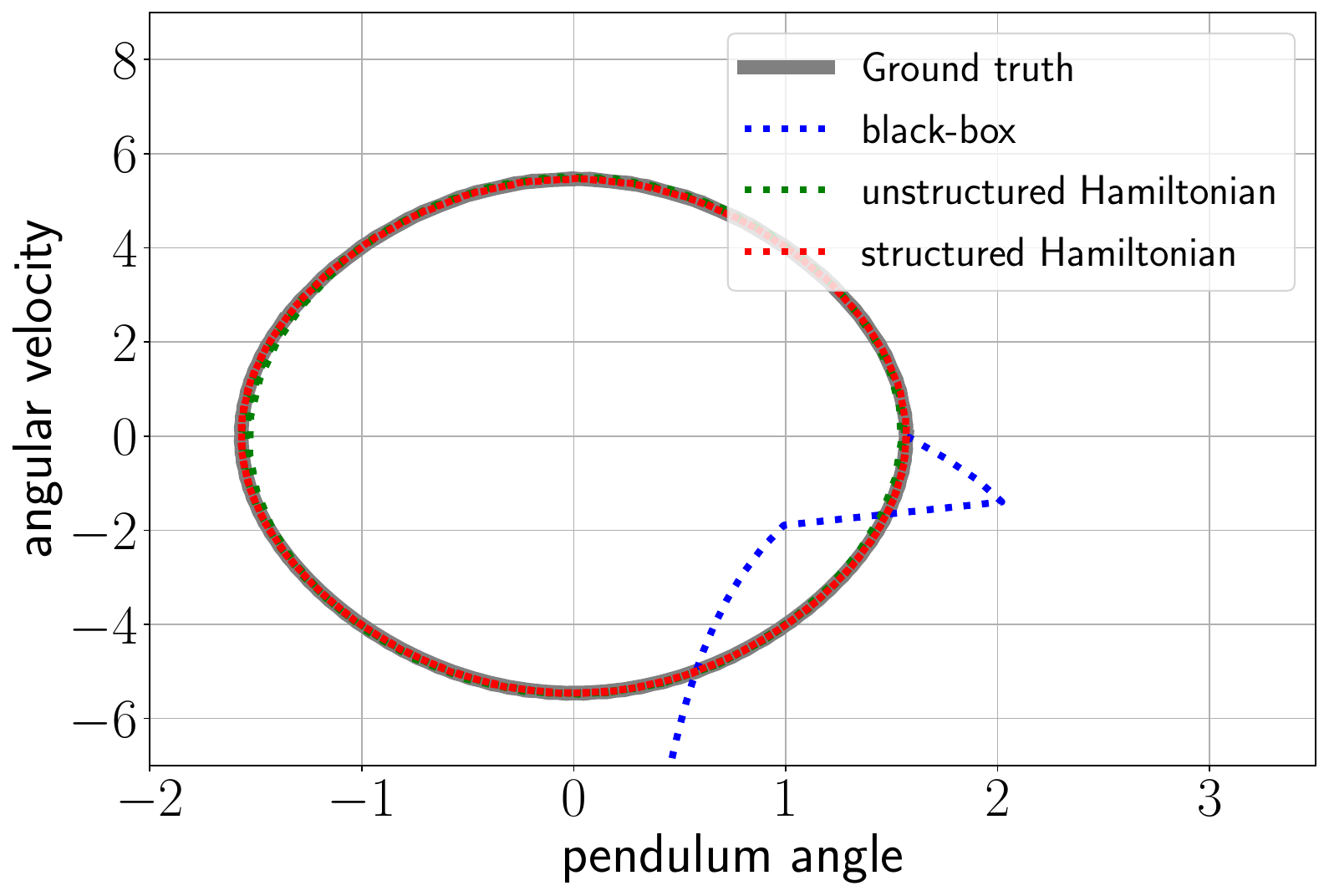}%
        \caption{Phase portraits}
        \label{fig:pend_phase_portrait}
\end{subfigure}%
\hfill
\begin{subfigure}[t]{0.24\textwidth}
        \centering
\includegraphics[width=\textwidth]{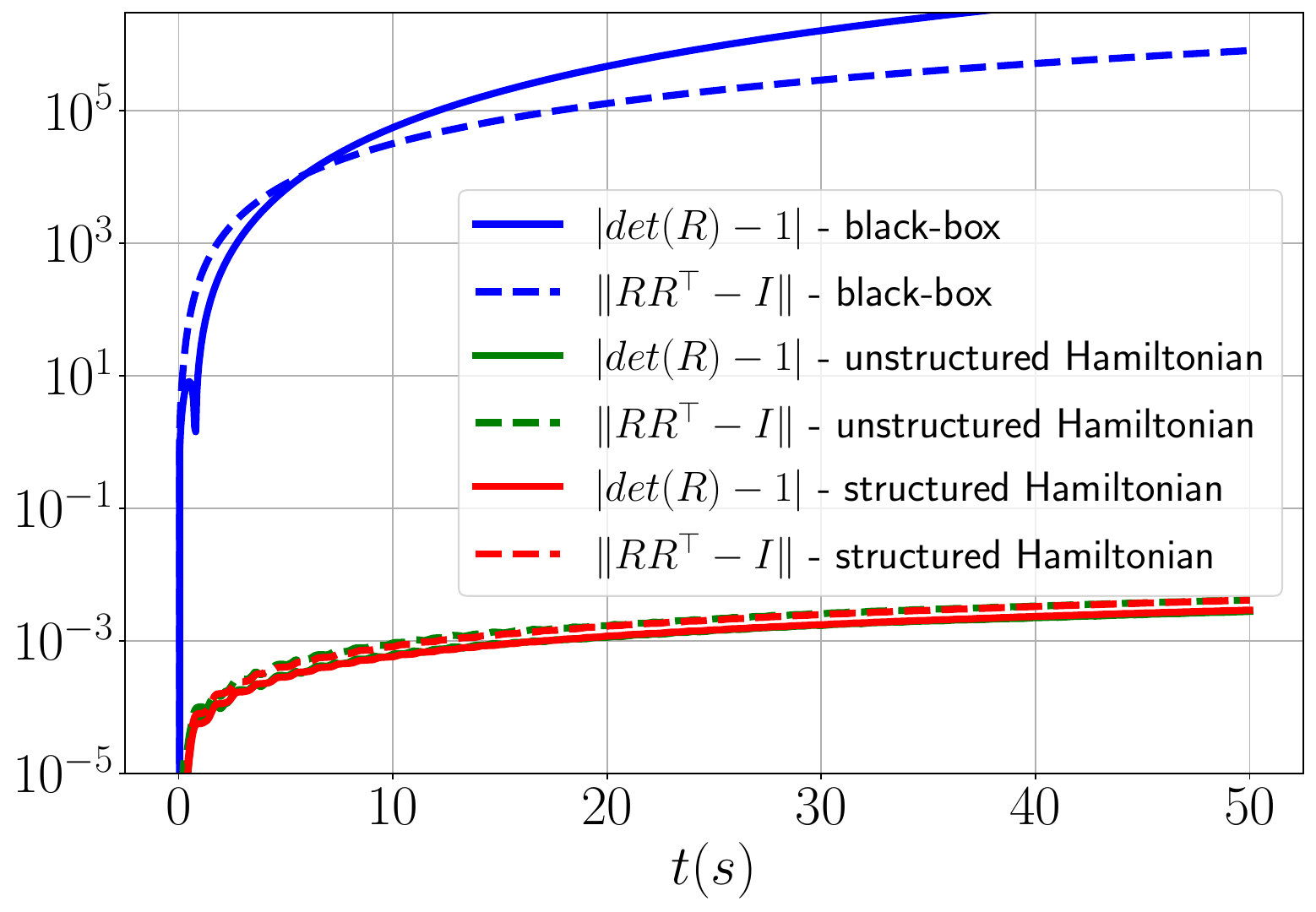}%
        \caption{SO(3) constraints}
        \label{fig:pend_so3_constraints}
\end{subfigure}%
\hfill
\begin{subfigure}[t]{0.235\textwidth}
        \centering
\includegraphics[width=\textwidth]{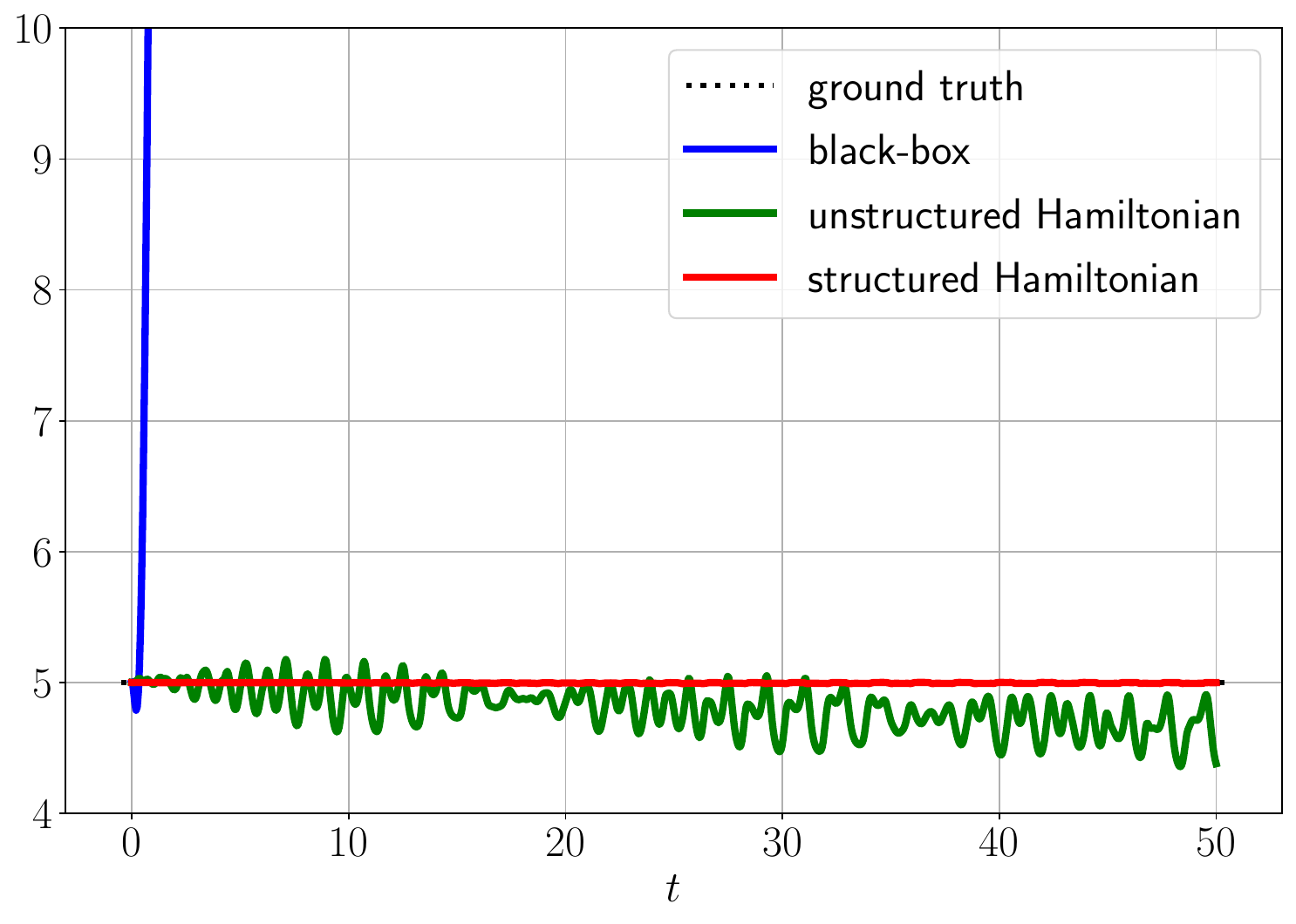}%
        \caption{Total energy}
        \label{fig:pend_total_energy}
\end{subfigure}%
\caption{Comparison of different neural network architectures to learn pendulum dynamics: 1) black-box, i.e., the dynamics function $\bff$ is represented by a multilayer perceptron network; 2) unstructured Hamiltonian, i.e., the Hamiltonian function is represented by a multilayer perceptron network instead of the sum of kinetic and potential energy as shown in Eq. \eqref{eq:hamiltonian_se3}; 3) structured Hamiltonian, i.e., the Hamiltonian function has the form of Eq. \eqref{eq:hamiltonian_se3}. Initialized at $\phi = \pi/2$, the learned pendulum dynamics are rolled out, showing that our approach with structured Hamiltonian preserves the phase portraits, $SO(3)$ constraints, and the conservation of energy better than the other models.}
\label{fig:pend_so3_comparison}
\end{figure*}

\begin{figure*}[t]
\centering
\begin{subfigure}[t]{0.26\textwidth}
        \centering
\includegraphics[width=\textwidth]{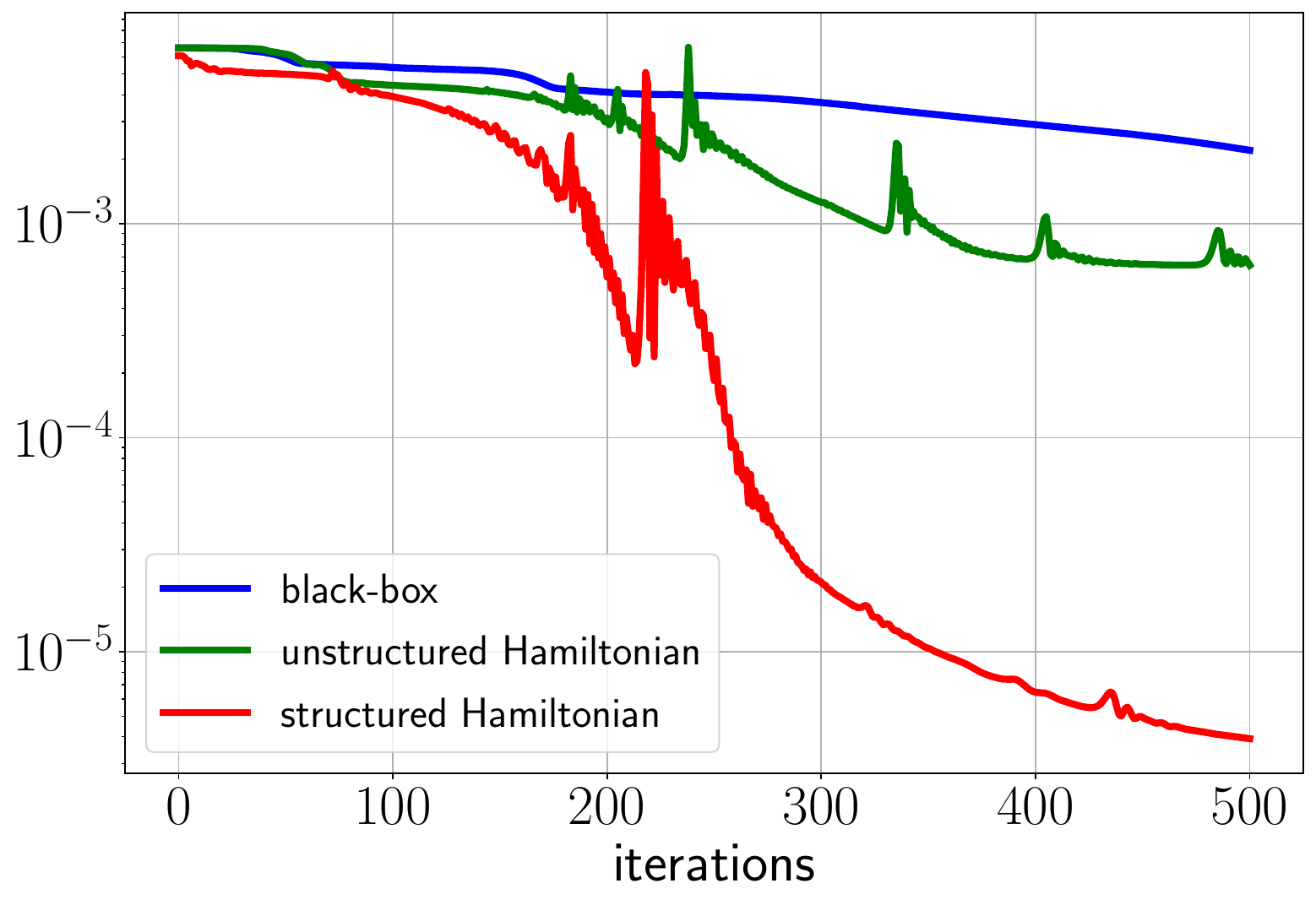}%
        \caption{Training loss}
        \label{fig:quad_loss_comparison}
\end{subfigure}%
\hfill
\begin{subfigure}[t]{0.26\textwidth}
        \centering
\includegraphics[width=\textwidth]{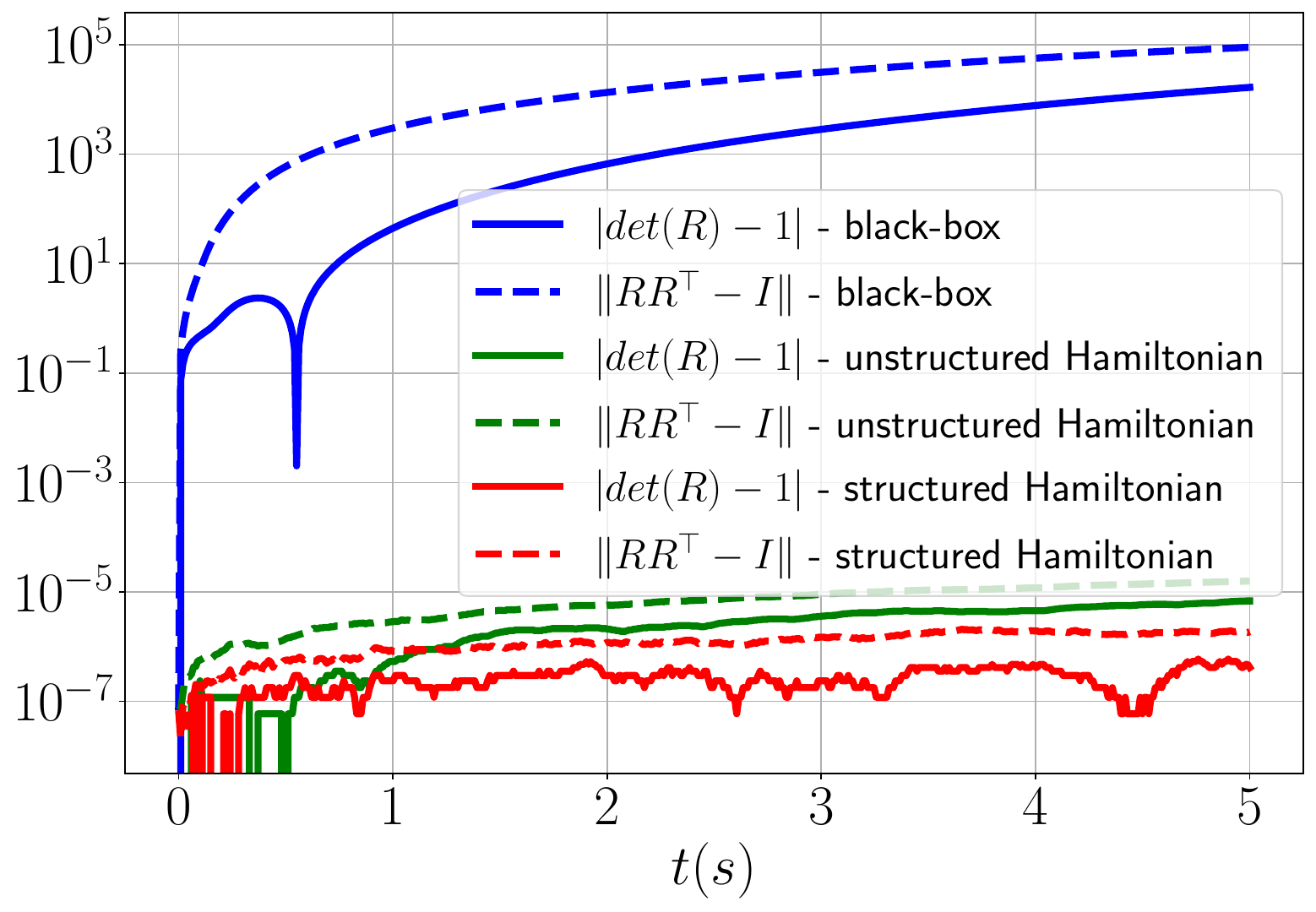}%
        \caption{SO(3) constraints}
        \label{fig:quad_se3_constraints_comparison}
\end{subfigure}%
\hfill
\begin{subfigure}[t]{0.26\textwidth}
        \centering
\includegraphics[width=\textwidth]{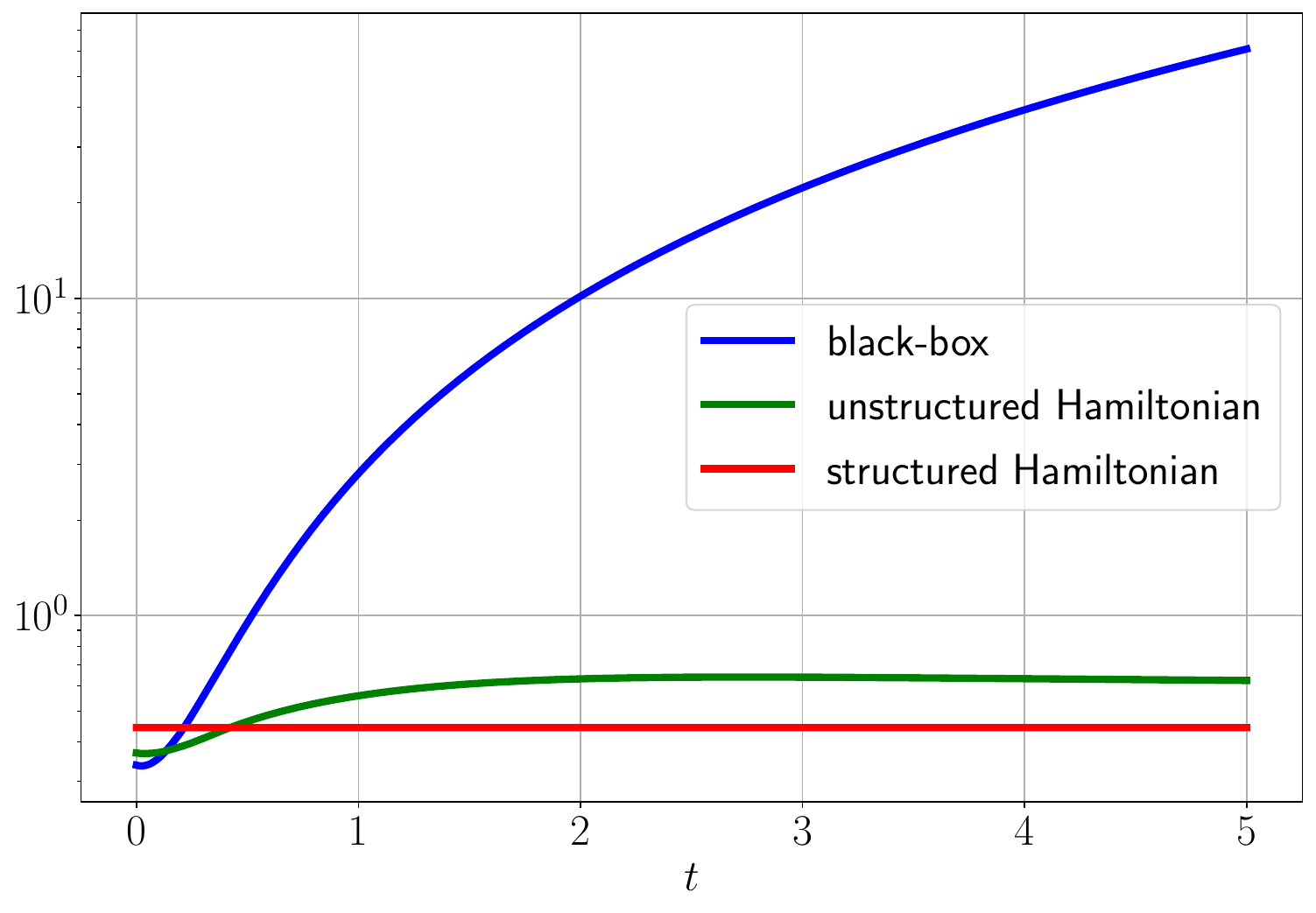}%
        \caption{Total energy}
        \label{fig:quad_hamiltonian_comparison}
\end{subfigure}%
\hfill
\begin{subfigure}[t]{0.21\textwidth}
        \centering
\includegraphics[width=\textwidth]{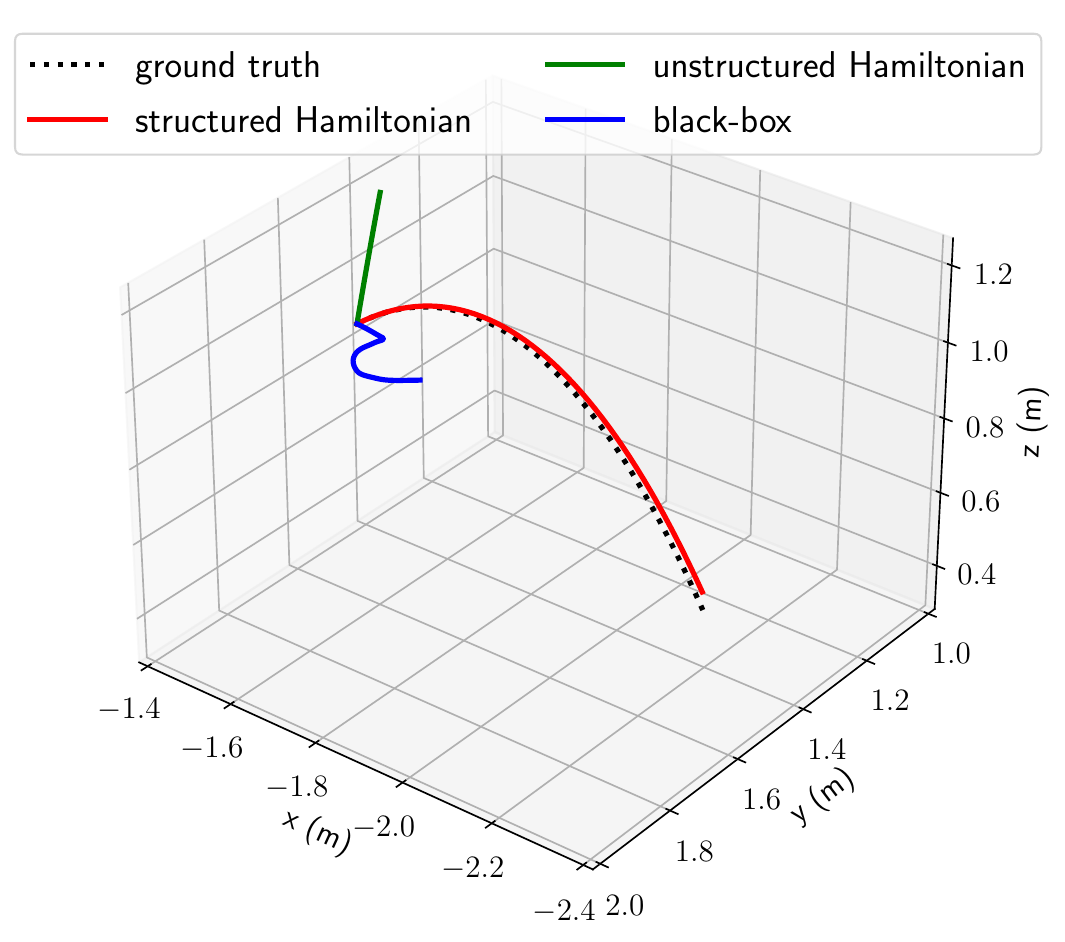}%
        \caption{Predicted trajectories}
        \label{fig:quad_predicted_traj}
\end{subfigure}%
\caption{Comparison of different neural network architectures to learn quadrotor dynamics: 1) black-box, i.e., the dynamics function $\bff$ is represented by a multilayer perceptron network; 2) unstructured Hamiltonian, i.e., the Hamiltonian function is represented by a multilayer perceptron network instead of the sum of kinetic and potential energy as shown in Eq. \eqref{eq:hamiltonian_se3}; 3) structured Hamiltonian, i.e., the Hamiltonian function has the form of Eq. \eqref{eq:hamiltonian_se3}. Initialized at a random pose and twist, the learned quadrotor dynamics are rolled out, showing that our approach with structured Hamiltonian preserves $SO(3)$ constraints and the conservation of energy better than the other models, and provides better state predictions.}
\label{fig:quad_se3_comparison}
\end{figure*}

The training and test results are shown in Fig. \ref{fig:pybullet_exp}. The learned generalized mass and inertia converged to constant diagonal matrices: $\bfM_\bfv^{-1}(\mathbf\frakq) \approx 27.5\bfI$, $\bfM_{\bfomega}^{-1}(\mathbf\frakq) \approx \text{diag}([351, 340, 181])$. The input matrix $\bfB_{\bfv}(\mathbf\frakq)$ converged to a constant matrix with $\begin{bmatrix}\bfB_{\bfv}(\mathbf\frakq)\end{bmatrix}_{2,0} \approx 1.32$ while the other entries were closed to $0$, consistent with the fact that the quadrotor thrust only affects the linear velocity along the $z$ axis in the body-fixed frame. The input matrix $\bfB_{\bfomega}(\mathbf\frakq)$ converged to $\sim 76\bfI$ as the motor torques affects all components of the angular velocity $\bfomega$. The learned potential energy $\calV(\mathbf\frakq)$ was linear in the height $z$, agreeing with the gravitational potential.

We also verified our energy-based control design in Sec. \ref{sec:controller_design} by controlling the quadrotor to track a desired trajectory based on the learned dynamics model. Given desired position $\bfp^*$ and heading $\bfpsi^*$ (yaw angle), we construct an appropriate $\bfR^*$ and $\frakp^*$ to be used with the energy-based controller in  \eqref{eq:idapbc_pose_twist_tracking}. 
By expanding the terms in  \eqref{eq:idapbc_pose_twist_tracking} and choosing the control gain $\bfK_\bfd$ of the form $\bfK_\bfd = \begin{bmatrix}
\bfK_v & \bf0\\
\bf0 & \bfK_\omega
\end{bmatrix}$, the control input can be written explicitly as
\begin{equation}
\label{eq:control_quadrotor}
\bfu = \bfB^{\dagger}(\mathbf\frakq) \begin{bmatrix}
\bftau_\bfv \\
\bftau_{\bfomega}
\end{bmatrix},
\end{equation}
where 
\begin{eqnarray}
\bftau_\bfv &=& \bfR^{\top} \frac{\partial \calV(\mathbf\frakq)}{\partial \bfp}  - \hat{\mathbf\frakp}_\bfv \bfomega - \bfR^\top\bfK_\bfp(\bfp - \bfp^*) \nonumber \\
&& - \bfK_\bfv(\bfv - \bfR^{\top}\dot{\bfp}^*)  + \bfM_1 (\bfR^\top \ddot{\bfp}^* - \hat{\bfomega}\bfR^\top \dot{\bfp}^*),  \\
\bftau_{\bfomega} &=& \sum_{i = 1}^3 \hat{\bfr}_i \frac{\partial \calV(\mathbf\frakq)}{\partial \bfr_i}- \bfK_{\bfomega}(\bfomega - \bfR^\top \bfR^* \bfomega^*)  \nonumber \\
&&   - (\hat{\mathbf\frakp}_{\bfomega} \bfomega + \hat{\mathbf\frakp}_\bfv \bfv) -\frac{1}{2}\prl{\bfK_{\bfR}\bfR^{*\top}\bfR-\bfR^\top\bfR^{*}\bfK_{\bfR}^\top}^{\vee}   \nonumber \\
&& + \bfM_2(\bfR^\top \bfR^* \dot{\bfomega}^* - \hat{\bfomega}_e\bfR^\top\bfR^* \bfomega^*).
\end{eqnarray}
Note that $\bftau_\bfv \in \mathbb{R}^3$ is the desired thrust in the body frame and depends only on the desired position $\bfp^*$ and the current pose. The desired thrust is transformed to the world frame as $\bfR\bftau_\bfv$.
Inspired by \cite{lee2010geometric}, the vector $\bfR\bftau_\bfv$ should be along the $z$ axis of the body frame, i.e., the third column $\bfr_3^*$ of the desired rotation matrix $\bfR^*$. The second column $\bfr^*_2$ of the desired rotation matrix $\bfR^*$ can be chosen so that it has the desired yaw angle $\psi^*$ and is perpendicular to $\bfr_3^*$. This can be done by projecting the second column of the yaw's rotation matrix $\bfr_2^\psi = [-\sin{\psi}, \cos{\psi}, 0]$ onto the plane perpendicular to $\bfr_3^*$. We have $\bfR^* = [\bfr_1^* \quad \bfr_2^* \quad \bfr_3^*]$ where:
\begin{equation}
\bfr_3^* = \frac{\bfR\bftau_\bfv}{\Vert \bfR\bftau_\bfv \Vert}, \bfr_1^* = \frac{\bfr_2^\psi \times \bfr^*_3}{\Vert  \bfr_2^\psi \times \bfr^*_3 \Vert}, \bfr_2^* = \bfr_3^* \times \bfr_1^*,
\end{equation}
and $\hat{\bfomega}^* = \bfR^{*\top}\dot{\bfR}^*$. The derivative $\dot{\bfR}^*$ is calculated as follows:
\begin{eqnarray}
\dot{\bfr}_3^* &=& \bfr_3^* \times \frac{\dot{\bfR\bftau_\bfv}}{\Vert \bfR\bftau_\bfv \Vert} \times \bfr_3^*,   \\
\dot{\bfr}_1^* &=& \bfr_1^* \times \frac{\dot{\bfr}_2^\psi \times \bfr_3^* + \bfr_2^\psi \times \dot{\bfr}_3^*}{\Vert \bfr_2^\psi \times \bfr_3^* \Vert} \times \bfr_1^*,  \\
\dot{\bfr}_2^* &=& \dot{\bfr}_3^* \times \bfr_1^* + \bfr_3^* \times \dot{\bfr}_1^*.
\end{eqnarray} 
Plugging $\bfR^*$ and $\bfomega^*$ back in $\bftau_{\bfomega}$, we obtain the complete control input $\bfu$ in  \eqref{eq:control_quadrotor}.


Fig. \ref{fig:pybullet_trajviz} shows qualitatively that the quadrotor controlled by our energy-based controller achieves successful trajectory tracking. The control gains were chosen as: $\bfK_\bfp = \diag([0.8, 0.8, 3.9])$, $\bfK_\bfv = 0.23\bfI, \bfK_{\bfR} = \diag([3.6, 3.6, 6.9])$, $\bfK_{\bfomega} = \diag([0.3, 0.3, 0.6])$. Fig. \ref{fig:pybullet_exp} shows quantitatively the tracking errors in the position, yaw angles, linear velocity, and angular velocity. On an Intel i9 3.1 GHz CPU with 32GB RAM, our controller's computation time in Python 3.7 was about $2.5$ ms per control input, including forward passes through the neural networks, showing that it is suitable for fast real-time applications.

\subsection{Comparison to Unstructured Neural ODE Models}

\begin{figure*}[t]
\centering
\begin{subfigure}[t]{0.33\textwidth}
        \centering
\includegraphics[width=\textwidth]{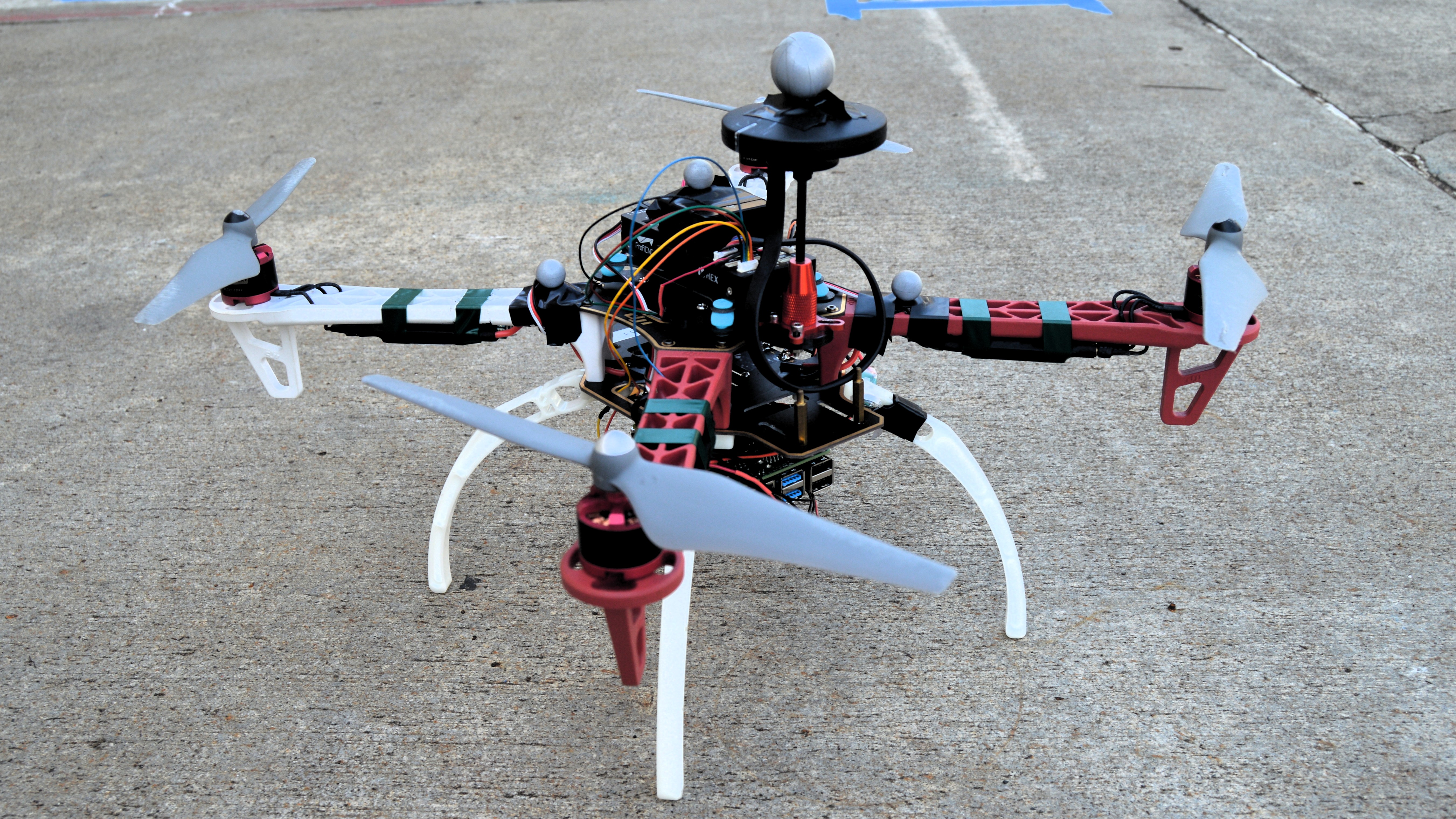}%
        \caption{RaspberryPi drone}
        \label{fig:quadrotor_barebone}
\end{subfigure}%
\hfill%
\begin{subfigure}[t]{0.33\textwidth}
        \centering
\includegraphics[width=\textwidth]{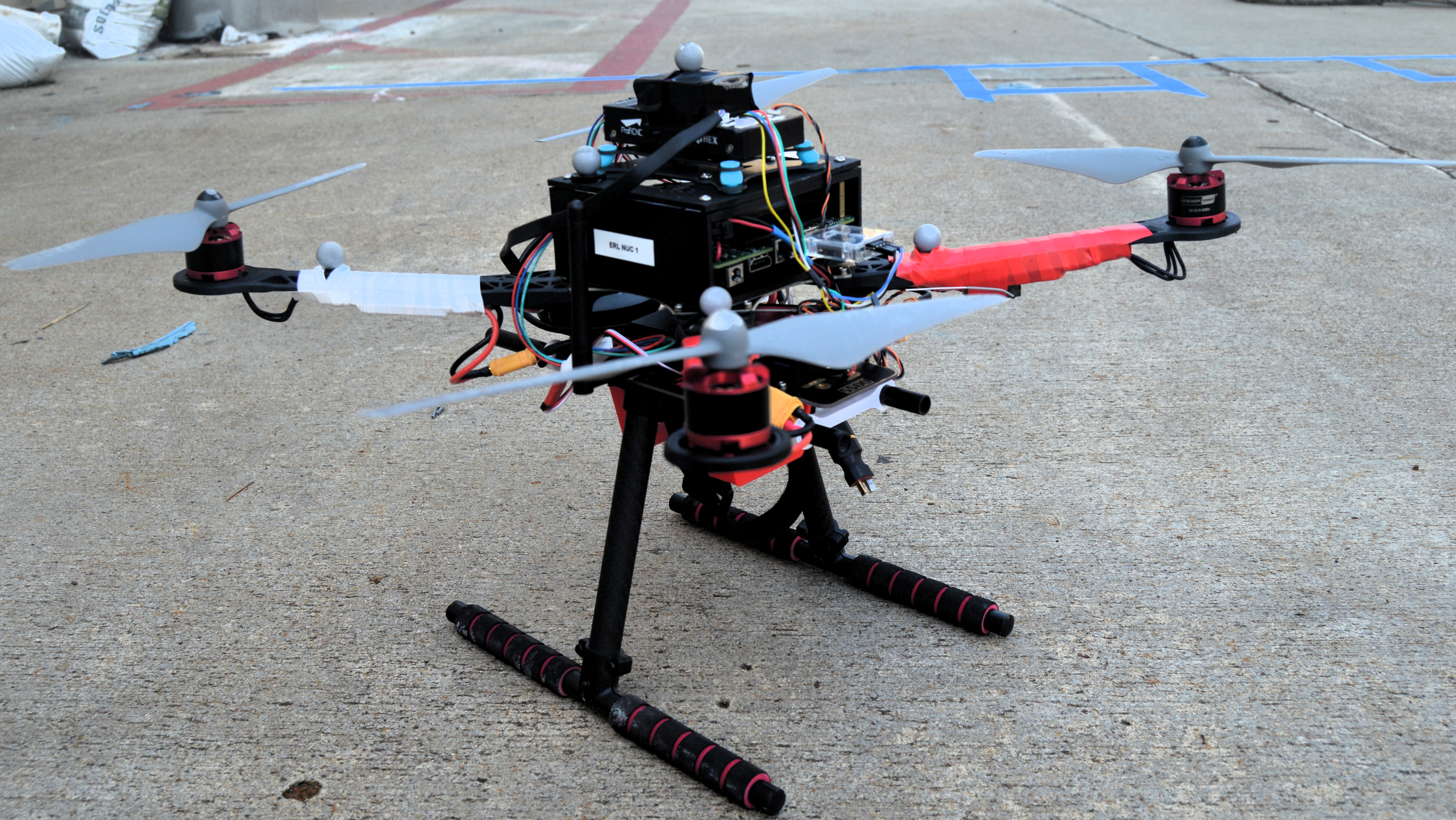}%
        \caption{Intel NUC drone}
        \label{fig:quadrotor_nuc}
\end{subfigure}%
\hfill%
\begin{subfigure}[t]{0.33\textwidth}
        \centering
\includegraphics[width=\textwidth]{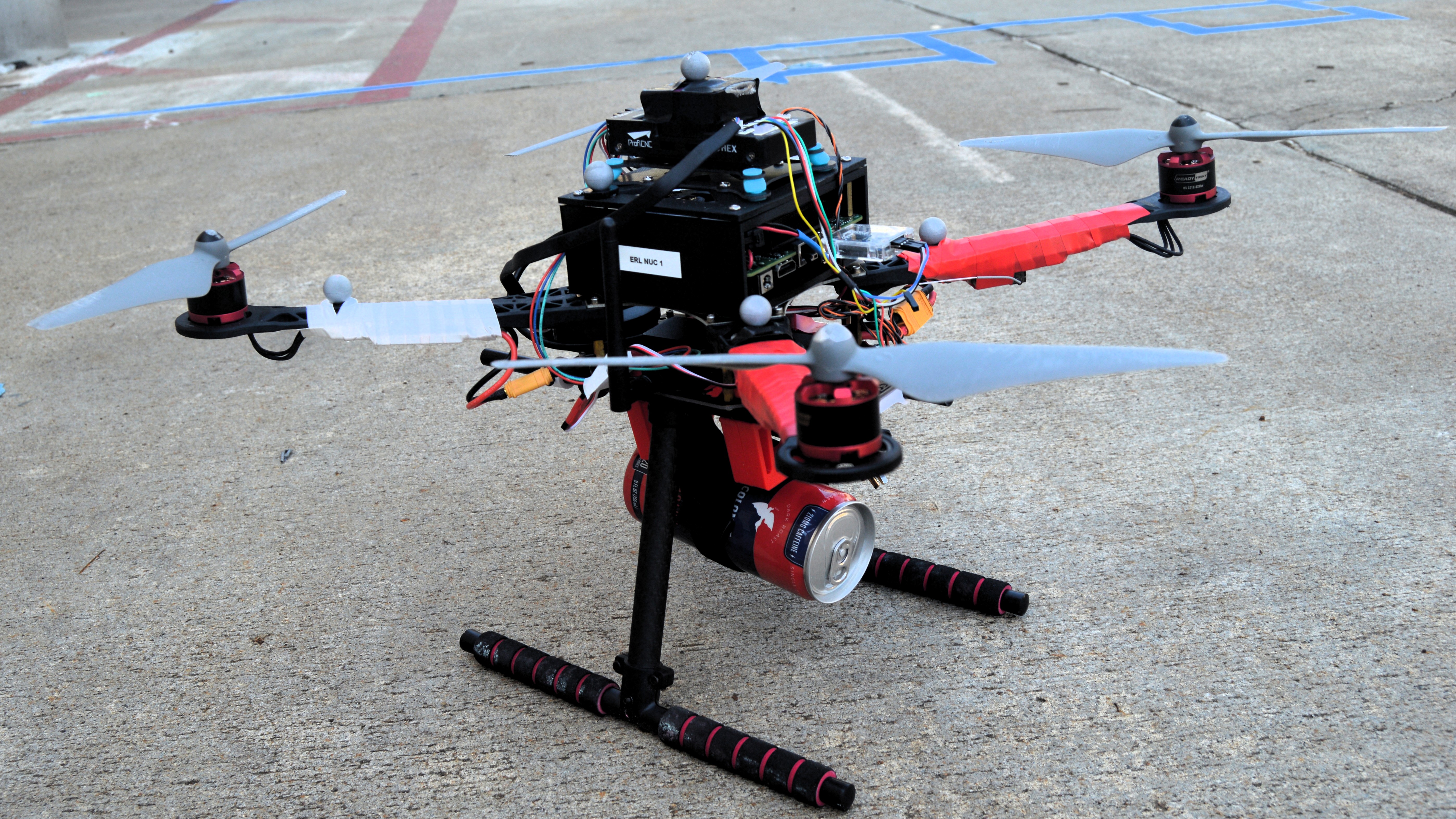}%
        \caption{Intel NUC drone with payload}
        \label{fig:quadrotor_nuc_payload}
\end{subfigure}%
\caption{Quadrotor robots used in the experiments: (a) RaspberryPi quadrotor whose mass and inertia matrix serve as nominal values for our learning framework, (b) Intel NUC quadrotor with a different frame, and (c) Intel NUC quadrotor carrying a coffee can as payload.}
\label{fig:quadrotor_photos}
\end{figure*}
\begin{figure*}[t]
\centering
\begin{subfigure}[t]{0.33\textwidth}
        \centering
\includegraphics[width=\textwidth]{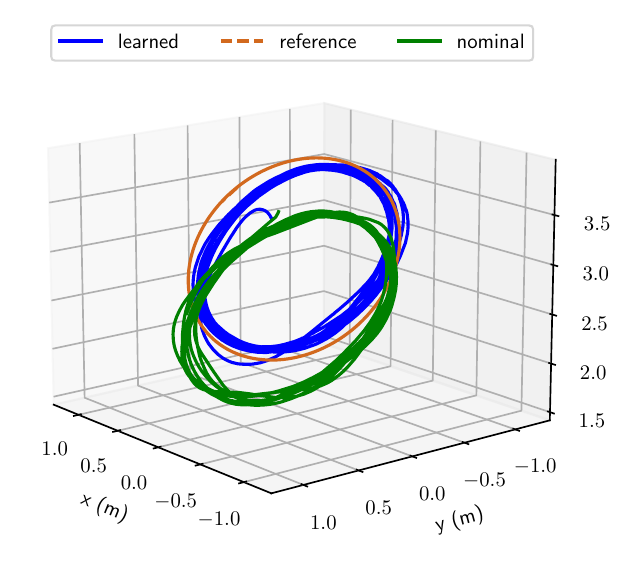}%
        \caption{Vertical circle}
        \label{fig:quadrotor_upgrade_circle_traj}
\end{subfigure}%
\hfill
\begin{subfigure}[t]{0.33\textwidth}
        \centering
\includegraphics[width=\textwidth]{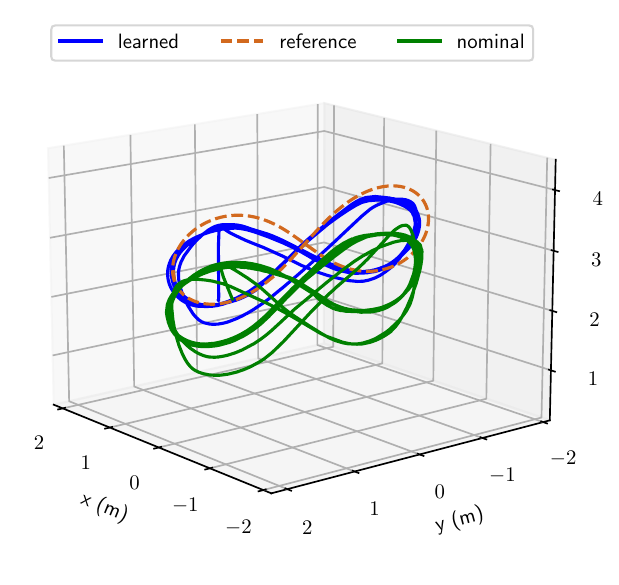}%
        \caption{Vertical lemniscate}
        \label{fig:quadrotor_upgrade_lemniscate_traj}
\end{subfigure}%
\hfill
\begin{subfigure}[t]{0.324\textwidth}
        \centering
\includegraphics[width=\textwidth]{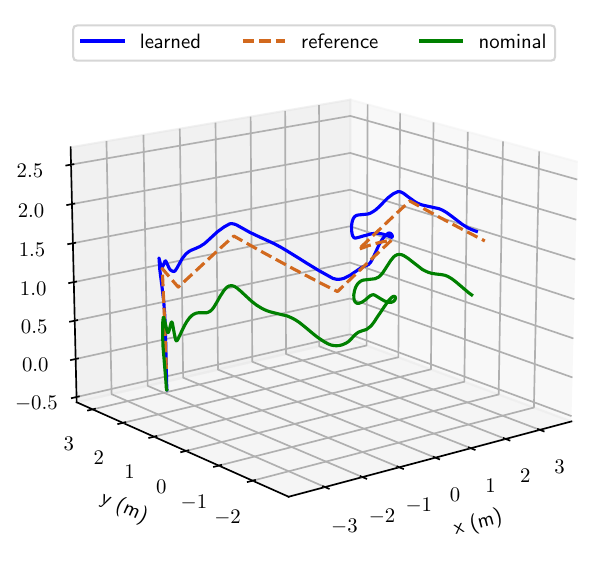}%
        \caption{Piecewise linear}
        \label{fig:quadrotor_upgrade_zigzag_traj}
\end{subfigure}%
\caption{Trajectory tracking with real quadrotors: (a) vertical circle, (b) vertical lemniscate, (c) piecewise linear trajectory.}
\label{fig:drone_upgrade_traj_3d}
\end{figure*}

\begin{figure*}[t]
\centering
\begin{subfigure}[t]{0.5\textwidth}
        \centering
\includegraphics[width=\textwidth]{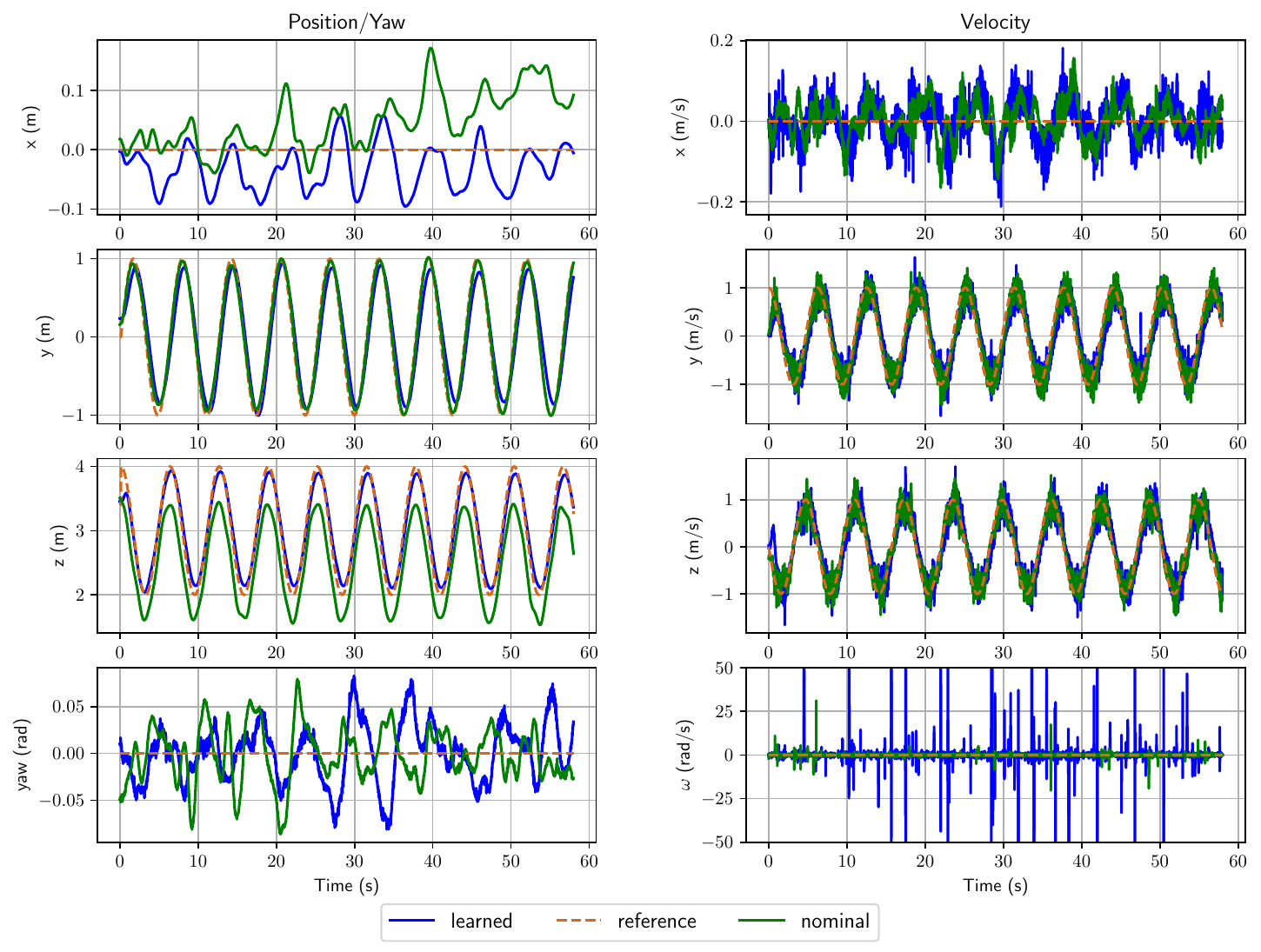}%
        \caption{Vertical circle}
        \label{fig:circle_tracking_plot}
\end{subfigure}%
\begin{subfigure}[t]{0.5\textwidth}
        \centering
\includegraphics[width=\textwidth]{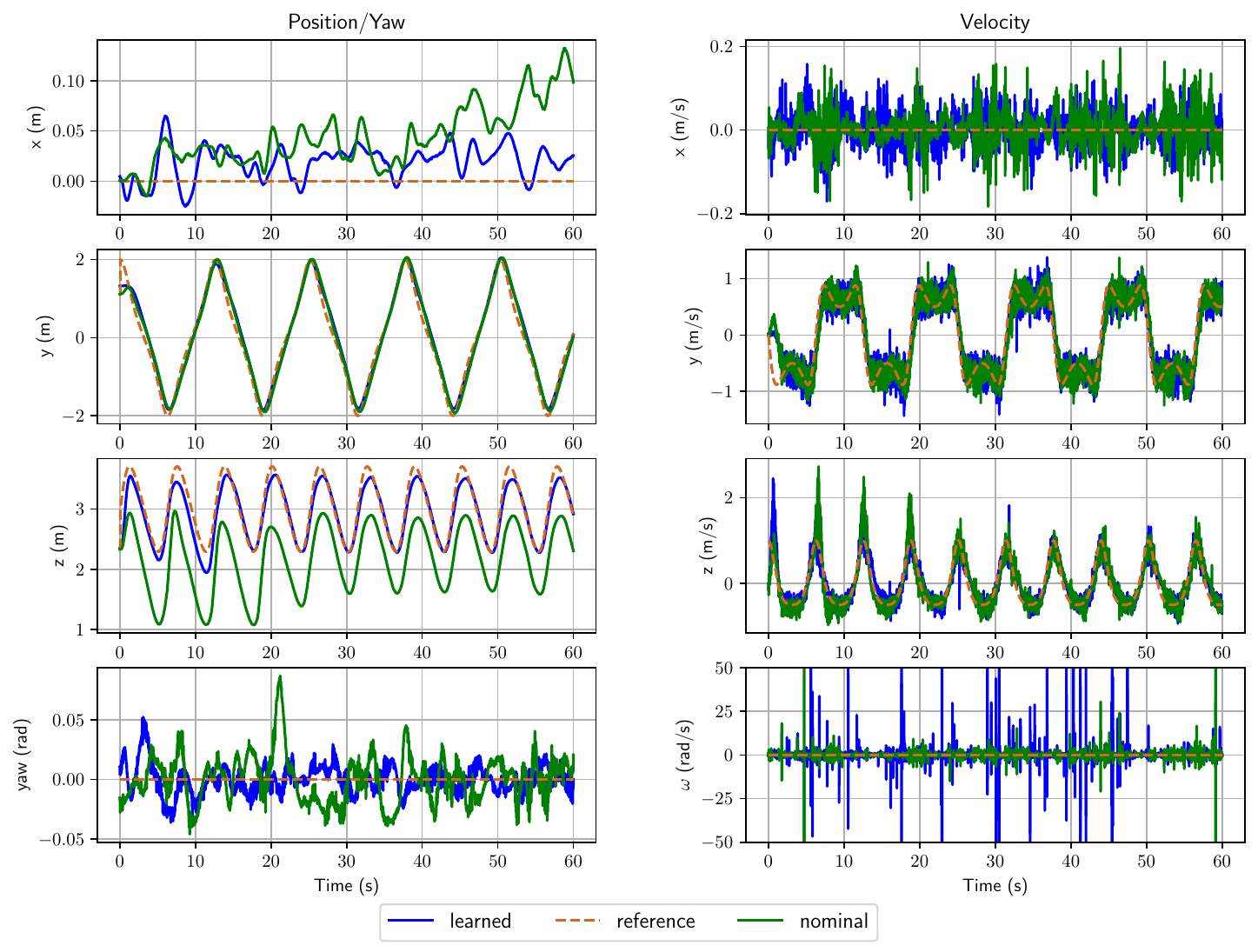}%
        \caption{Vertical lemniscate}
        \label{fig:lemniscate_tracking_plot}
\end{subfigure}%

\begin{subfigure}[t]{0.5\textwidth}
        \centering
\includegraphics[width=\textwidth]{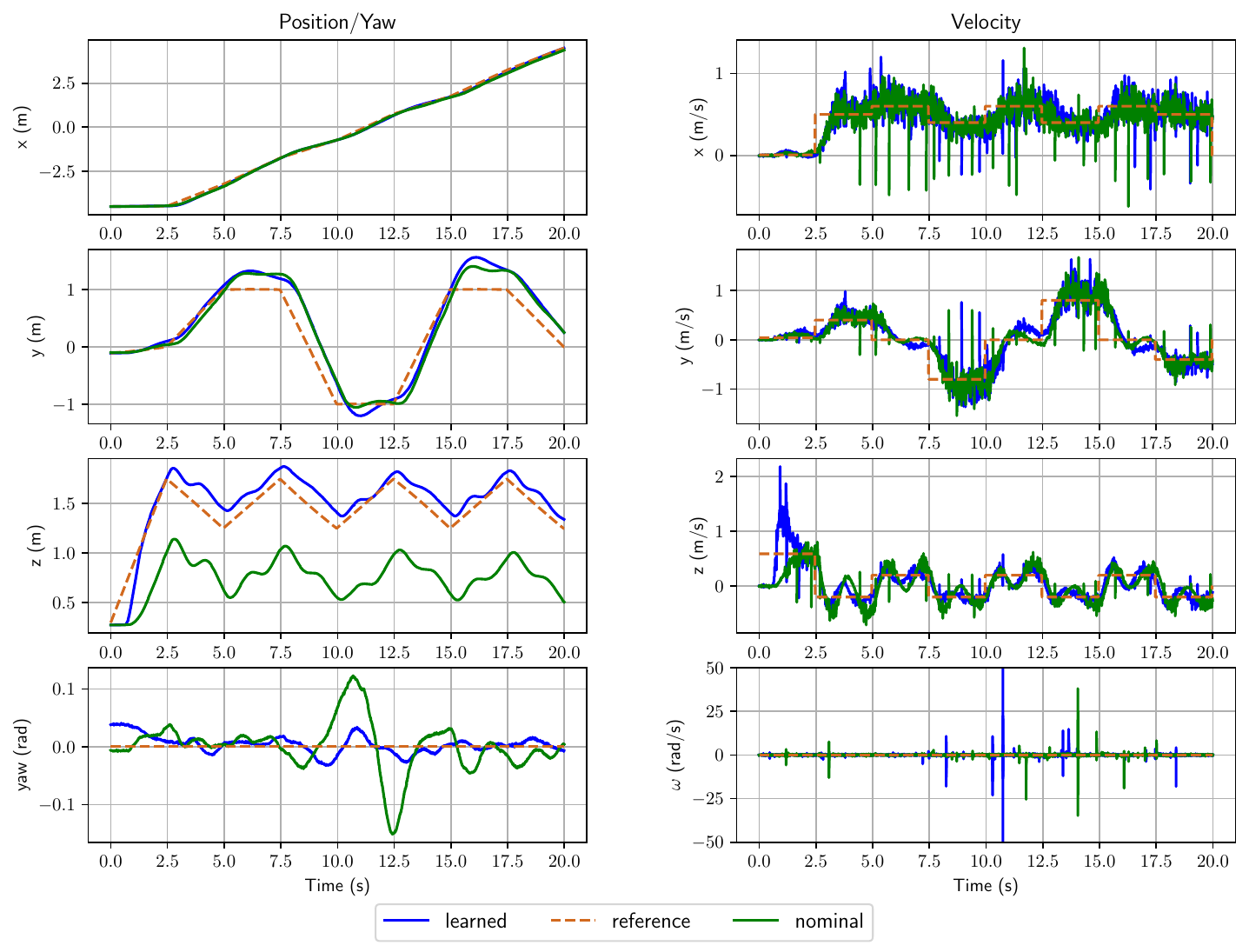}%
        \caption{Piecewise-linear}
        \label{fig:zigzage_tracking_plot}
\end{subfigure}%
\hfill
\begin{minipage}{0.48\textwidth}
\vspace{-60mm}
\begin{subfigure}[t]{0.49\textwidth}
        \centering
\includegraphics[width=\textwidth]{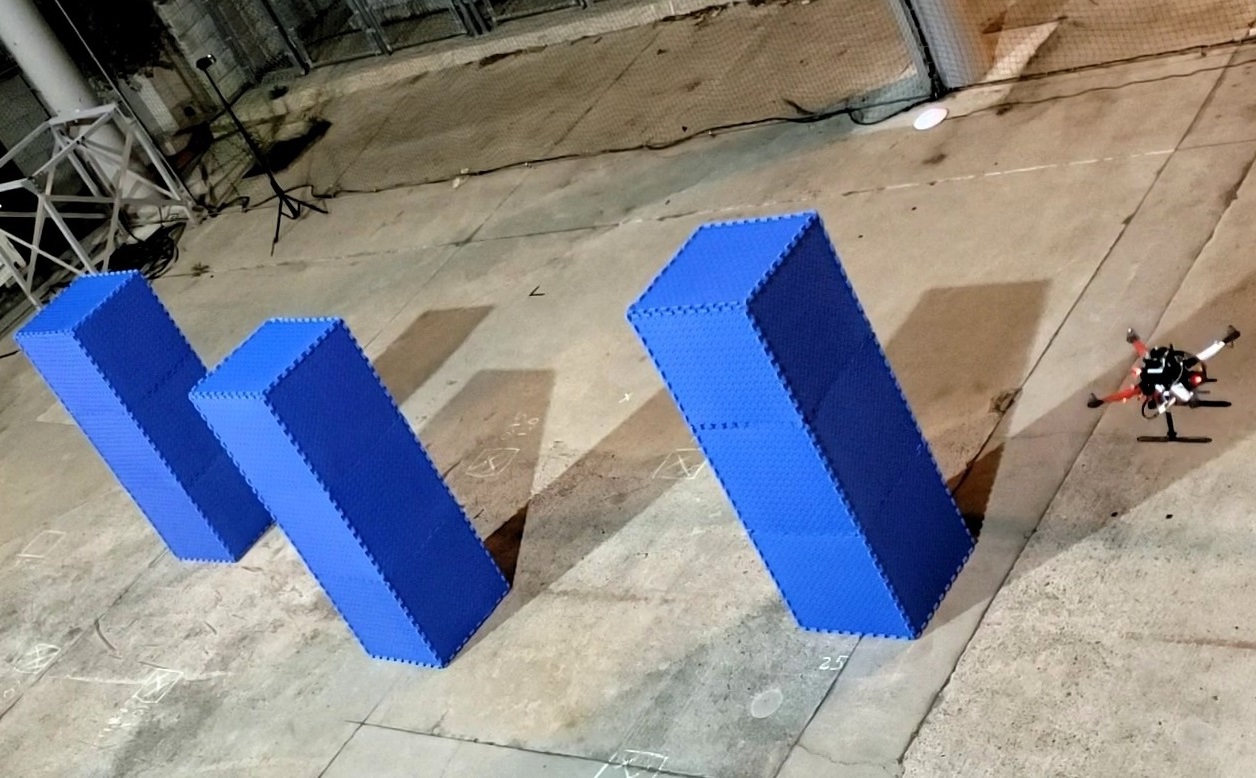}%
        \caption{Tracking at $t=0$s}
        \label{fig:zigzag_t0_snapshot}
\end{subfigure}%
\hfill
\begin{subfigure}[t]{0.49\textwidth}
        \centering
\includegraphics[width=\textwidth]{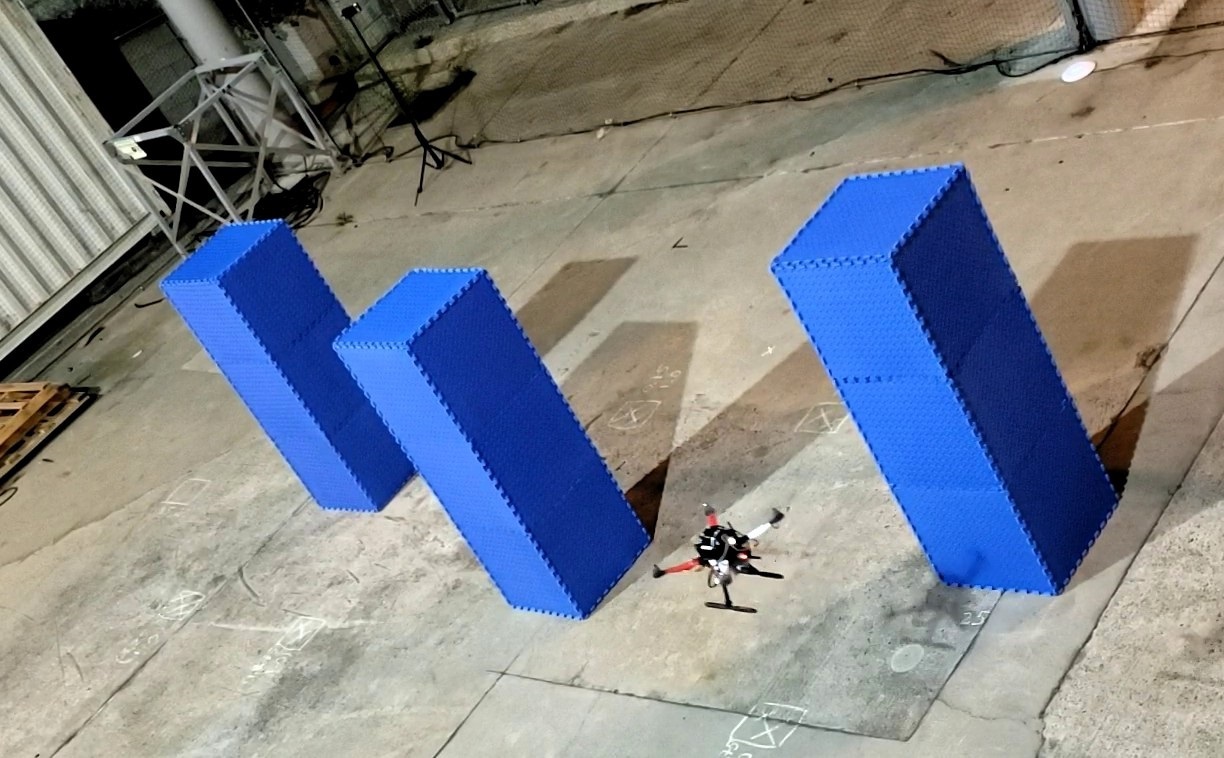}%
        \caption{Tracking at $t=3$s}
        \label{fig:zigzag_t3_snapshot}
\end{subfigure}%
\vspace{5mm}
\begin{subfigure}[t]{0.49\textwidth}
        \centering
\includegraphics[width=\textwidth]{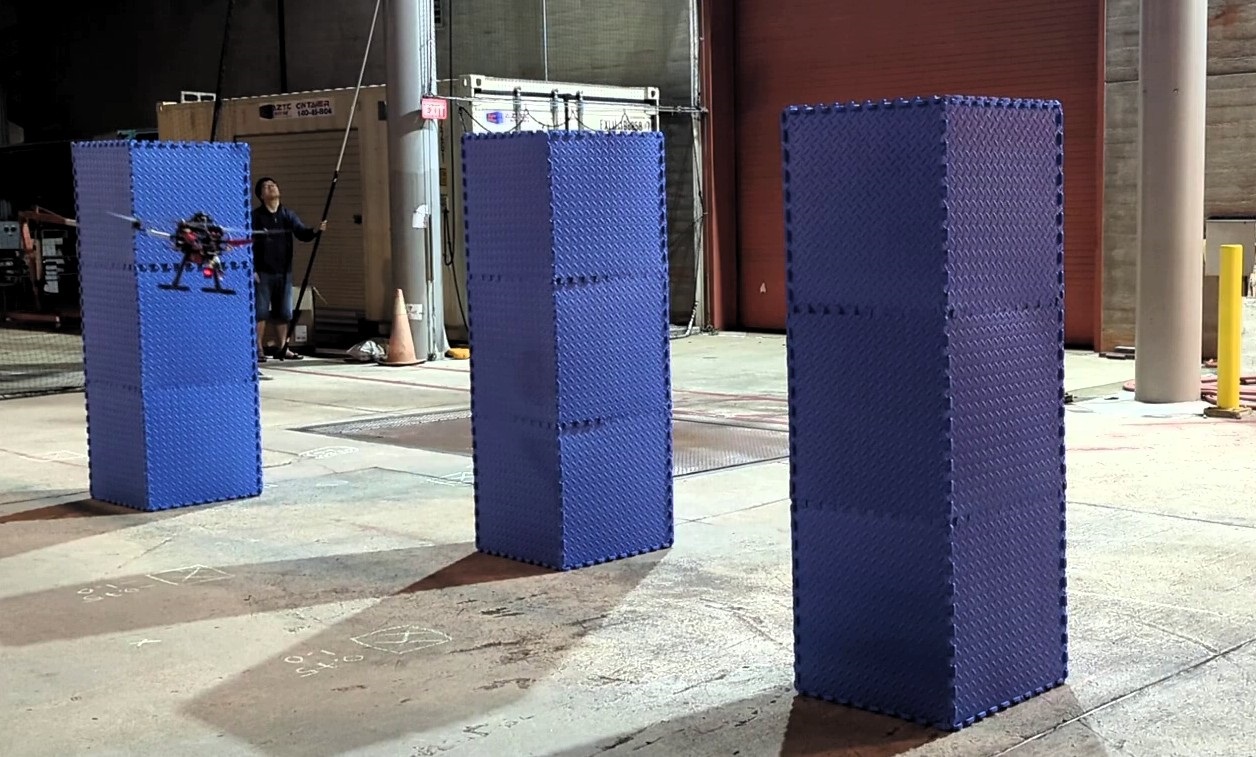}%
        \caption{Tracking at $t=9$s}
        \label{fig:zigzag_t9_snapshot}
\end{subfigure}%
\hfill
\begin{subfigure}[t]{0.49\textwidth}
        \centering
\includegraphics[width=\textwidth]{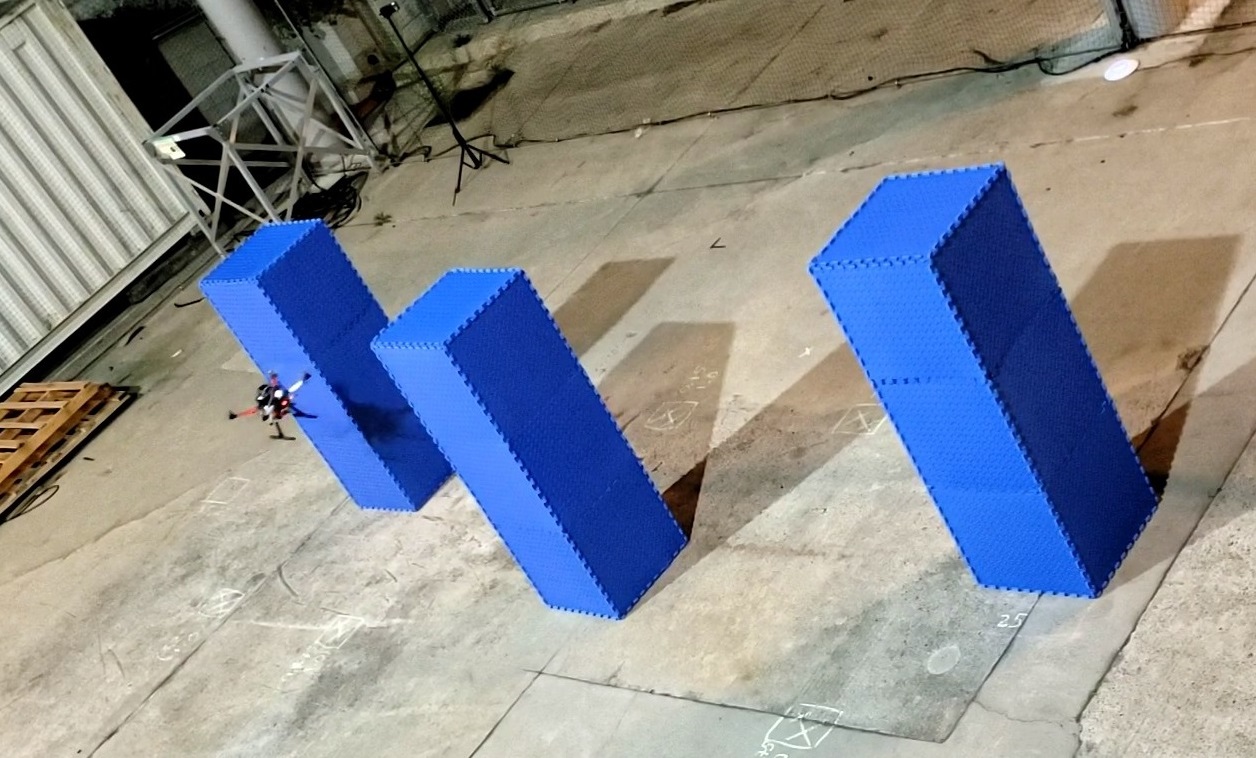}%
        \caption{Tracking at $t=12$s}
        \label{fig:zigzag_t12_snapshot}
\end{subfigure}%
\end{minipage}
\caption{Real quadrotor trajectory using our learned model and controller (blue) and using a nominal model and a geometric controller \cite{lee2010geometric} (green) tracking a desired trajectory.}
\label{fig:drone_upgrade_trajplot}
\end{figure*}

\begin{figure*}[t]
\centering
\begin{subfigure}[t]{0.33\textwidth}
        \centering
\includegraphics[width=\textwidth]{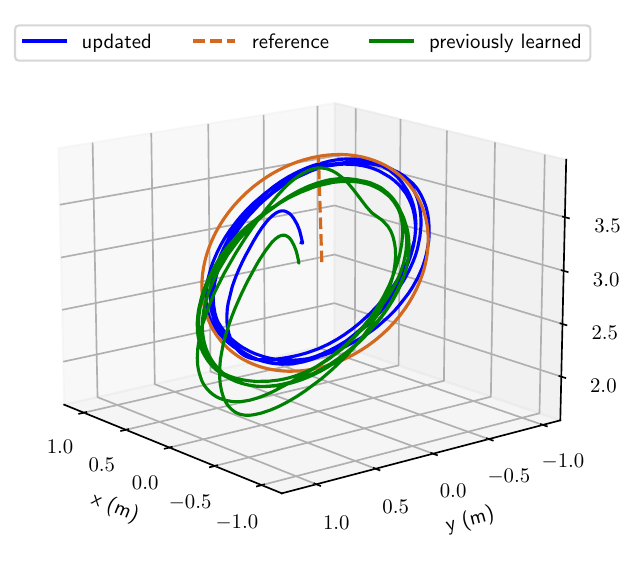}%
        \caption{Vertical circle}
        \label{fig:vertical_circle_payload}
\end{subfigure}%
\begin{subfigure}[t]{0.33\textwidth}
        \centering
\includegraphics[width=\textwidth]{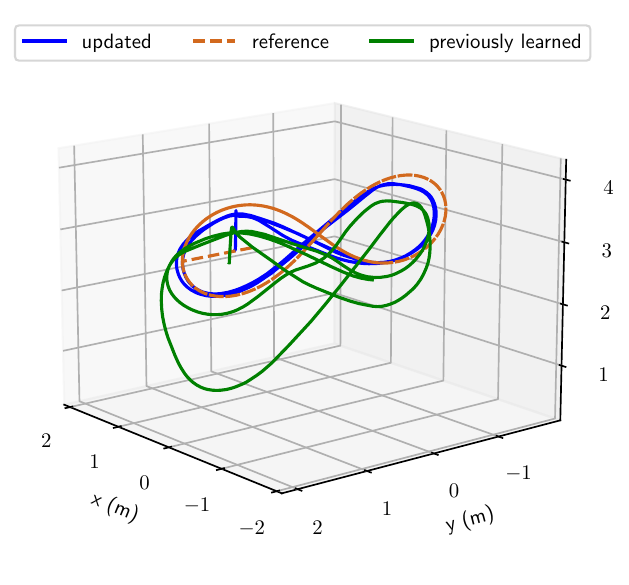}%
        \caption{Vertical lemniscate}
        \label{fig:vertical_lemniscate_payload}
\end{subfigure}%
\begin{subfigure}[t]{0.317\textwidth}
        \centering
\includegraphics[width=\textwidth]{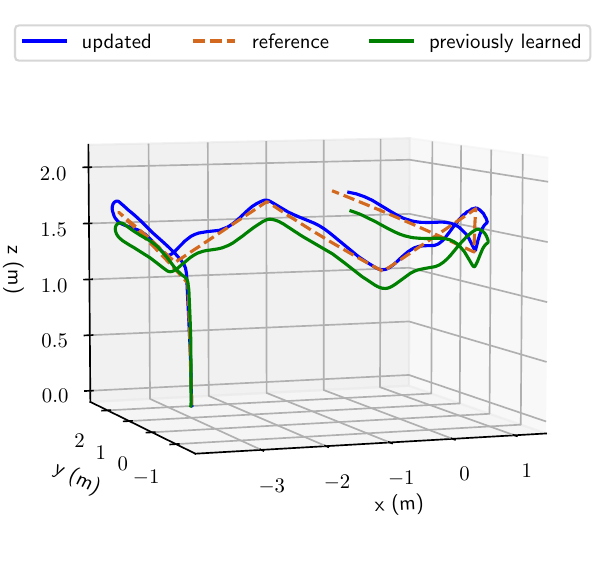}%
        \caption{Piecewise linear}
        \label{fig:vertical_zigzag_payload}
\end{subfigure}%

\caption{Quadrotor trajectory tracking experiment with extra payload: vertical circle (a), vertical lemniscate (b), piecewise linear trajectory (c).}
\label{fig:real_tracking_traj_payload}
\end{figure*}

In this section, we show the benefits of our neural ODE architecture by comparing 1) our \emph{structured Hamiltonian model}, 2) a \emph{black-box model}, i.e., the approximated dynamics $\bff$ is represented by a multilayer perceptron network, and 3) an \emph{unstructured Hamiltonian model}, i.e., the Hamiltonian function is represented by a multilayer perceptron network instead of using the structure in Eq. \eqref{eq:hamiltonian_se3}, in terms of training convergence rates, \NEWW{satisfaction} of energy conservation principle, and Lie group constraints. To verify energy conservation, we rolled out the learned dynamics and calculated the Hamiltonian via \eqref{eq:hamiltonian_se3} along the predicted trajectories for (1) the black-box model using ground-truth mass and potential energy with the predicted states; (2) the unstructured Hamiltonian model using the output of the multilayer perceptron Hamiltonian network; and (3) the structured Hamiltonian model using the learned mass and potential energy networks. We check the $SO(3)$ constraints by verifying that two quantities $|\det{\bfR} - 1|$ and $\Vert \bfR \bfR^\top - \bfI \Vert$ remain small along the predicted trajectories.

We first use a pendulum as described in Sec. \ref{subsec:pendulum_so3} without energy dissipation. The models are trained for $5000$ iterations from $512$ 0.2-second state-control trajectories and rolled out for a significantly longer horizon of $50$ seconds. Fig. \ref{fig:pend_so3_comparison} plots the training loss, the phase portraits, the SO(3) constraints and the total energy (Hamiltonian) of the learned models for a pendulum system. 
As the Hamiltonian structure is imposed in the neural ODE network architecture, our model is able to converge faster with lower loss (Fig. \ref{fig:pend_training_loss} \NEWW{and Table \ref{table:compare_net_arch}}), preserves the phase portraits for state predictions (Fig. \ref{fig:pend_phase_portrait}), \NEWW{and achieves the lowest angle prediction error in in Table \ref{table:compare_net_arch}}.
Fig. \ref{fig:pend_so3_constraints} and \NEWW{Table \ref{table:compare_net_arch}} show that the $SO(3)$ constraints are satisfied by our structured and unstructured Hamiltonian models as their values of $|\det{\bfR} - 1|$ and $\Vert \bfR \bfR^\top - \bfI \Vert$ remain small along a $50$-second trajectory rollout initialized at $\phi = \pi/2$. The constant Hamiltonian in Fig. \ref{fig:pend_total_energy} of our structured Hamiltonian \NEWW{with lowest standard derivation value in Table \ref{table:compare_net_arch}} verifies that our model obeys the energy conservation law \NEWW{with high precision, given} no control input and no energy dissipation. The Hamiltonian of the black-box model increases along the trajectory while that of the unstructured Hamiltonian model fluctuates and slightly decreases over time.

We also tested the models using the simulated Crazyflie quadtoror with the same dataset $\calD$ of $18$ trajectories as described in Sec. \ref{subsec:crazieflie_quad}. The $SE(3)$ port-Hamiltonian ODE network was trained, as described in Sec. \ref{subsec:SE3_dyn_learning}, for $500$ iterations. Our structured Hamiltonian model converges faster with significantly lower loss as seen in Fig. \ref{fig:quad_loss_comparison} \NEWW{and in Table \ref{table:compare_net_arch}}.
We verified that the predicted orientation trajectories from our learned models satisfy the $SO(3)$ constraints. Fig. \ref{fig:quad_se3_constraints_comparison} \NEWW{and Table \ref{table:compare_net_arch}} show two near-zero quantities $|\det{\bfR} - 1|$ and $\Vert \bfR \bfR^\top - \bfI \Vert$, obtained by rolling out our learned dynamics for $5$ seconds, while the learned black-box model significantly violates the constraints after a very short time. Fig. \ref{fig:quad_hamiltonian_comparison} shows a constant total energy along the predicted trajectory, i.e., \NEWW{lowest standard derivation value in Table \ref{table:compare_net_arch}}, from our structured Hamiltonian model without control input and dissipation networks, verifying that the learned model obeys the law of energy conservation. Fig. \ref{fig:quad_predicted_traj} and \NEWW{the prediction error in Table \ref{table:compare_net_arch}} show that our structured Hamiltonian model provides better trajectory predictions compared to the other methods.


\begin{table}[t]
    \caption{\NEWW{Comparison of different neural network architectures for pendulum and quadrotor dynamics learning.}}
    \label{table:compare_net_arch}
    \centering
  \begin{adjustbox}{width=\columnwidth,center}
	\begin{tabular}{ccccc} 
		Metrics & Platform & Black-box  & Unstructured & Structured \\
            & & & Hamiltonian & Hamiltonian \\
		\hline
		\hline
  		Training loss & Pendulum & $0.037$ & $6.3\times 10^{-5}$ & $\mathbf{2.3\times 10^{-7}}$ \\
            $\Vert det(\bfR) - 1 \Vert$ (avg.)  & Pendulum & $1888752$ & $1.4 \times 10^{-3}$ & $\mathbf{1.4 \times 10^{-3}}$\\
            $\Vert\bfR\bfR^\top - \bfI\Vert$ (avg.) & Pendulum & $267971.6$ & $2.1 \times 10^{-3}$ & $\mathbf{2.1 \times 10^{-3}}$ \\
            Total energy (std.) & Pendulum & $13557.1$ & $0.163$ & $\mathbf{0.003}$ \\
            Prediction error (avg.) & Pendulum & $1.02~(rad)$ & $0.08~(rad)$ & $\mathbf{0.008~(rad)}$
            \\\hline
		Training loss & Quadrotor & $2.2\times 10^{-3}$ & $6.4\times 10^{-4}$ & $\mathbf{3.92\times 10^{-6}}$ \\
              $\Vert det(\bfR) - 1 \Vert$ (avg.)  & Pendulum & $3741$ & $2.9 \times 10^{-6}$ & $\mathbf{2.6 \times 10^{-7}}$\\
            $\Vert\bfR\bfR^\top - \bfI\Vert$ (avg.) & Quadrotor & $29336.7$ & $7.6 \times 10^{-6}$ & $\mathbf{1.3 \times 10^{-6}}$ \\
            Total energy (std.) & Quadrotor & $18.1$ & $0.074$ & $\mathbf{1.64 \times 10^{-6}}$ \\
            Prediction error (avg.) & Quadrotor & $0.49~(m)$ & $0.46~(m)$ & $\mathbf{0.02~(m)}$
            \\\hline
	\end{tabular}
  \end{adjustbox}
\end{table}

\begin{table}[t]
    \caption{Position tracking errors using nominal and learned quadrotor models with and without payload.}
    \label{table:compare_tracking_errors}
    \centering
  \begin{adjustbox}{width=\columnwidth,center}
	\begin{tabular}{cccccc} 
		Model & Train  & Test & Circle  & Lemniscate & Piecewise-\\&with&with&&&linear \\&payload&payload&&& \\
		\hline
		\hline
		Nominal & - & No & $0.26~(m)$ & $0.52~(m)$ & $0.62~(m)$ \\
		Learned & No & No  & $\mathbf{0.13 (m)}$ & $\mathbf{0.14 (m)}$ & $\mathbf{0.22 (m)}$ \\
        Learned & No & Yes  & $0.20~(m)$ & $0.40~(m)$ & $0.30~(m)$ \\
		Learned & Yes & Yes & $\mathbf{0.13~(m)}$& $\mathbf{0.12~(m)}$ & $\mathbf{0.21~(m)}$ \\\hline
	\end{tabular}
  \end{adjustbox}
\end{table}

\begin{figure}[t]
\centering
\includegraphics[width=0.5\textwidth]{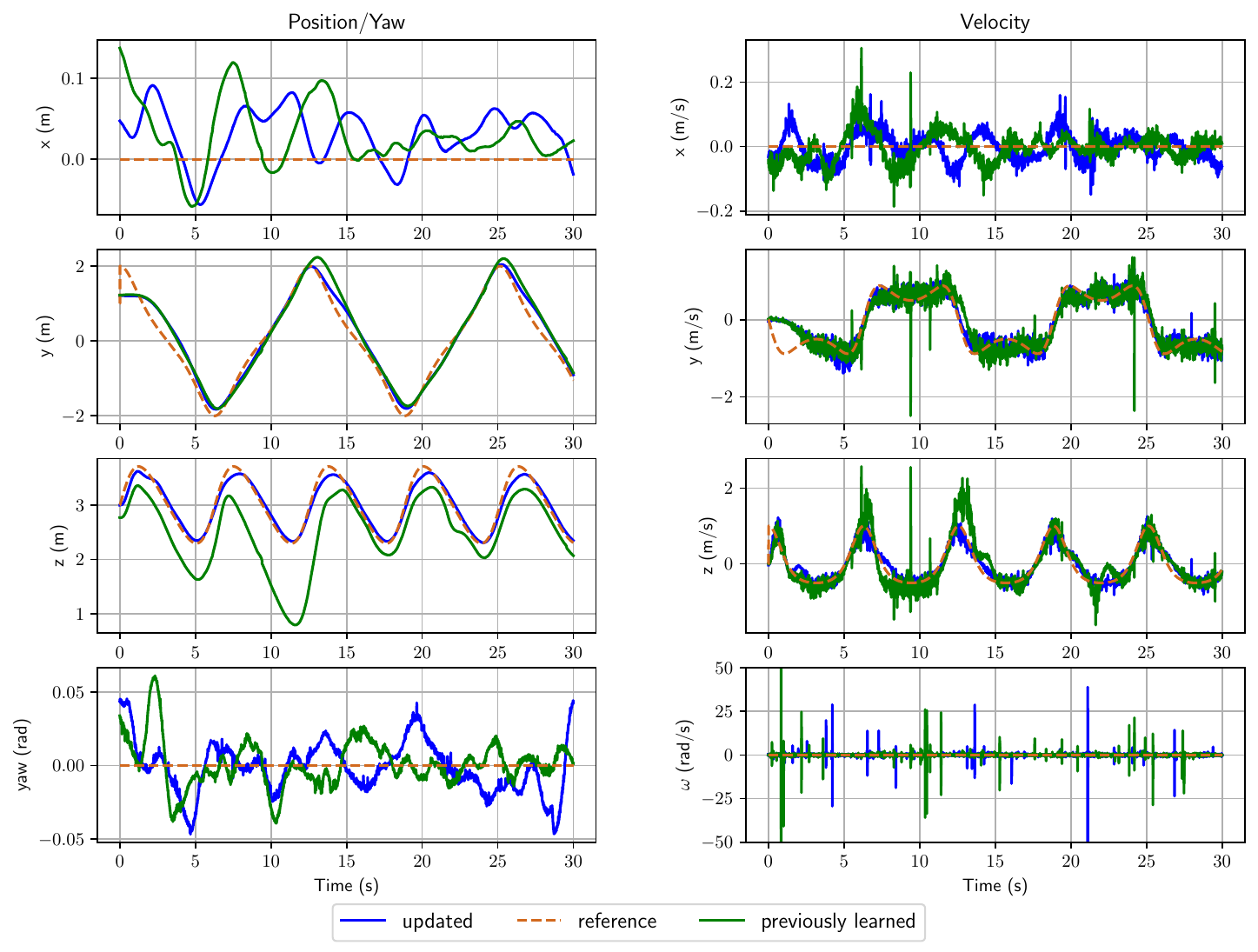}%
\caption{Tracking a piecewise-linear trajectory with extra payload using our previously learned and updated models.}
\label{fig:real_tracking_zigzag_payload}
\end{figure}

\subsection{Real Quadrotor Experiments}
\label{subsec:real_exp_quad}

In this section, we verify our approach using a real quadrotor robot, equipped with an onboard i7 Intel NUC computer and a PX4 flight controller (see Fig. \ref{fig:quadrotor_nuc}). The quadrotor's pose and twist were provided by a motion capture system.  

\subsubsection{Learning quadrotor dynamics after upgrade}
\label{subsubsec:real_quad_upgrade}

We consider a scenario in which the quadrotor is upgraded with a new frame and a new onboard computer, leading to changes in the robot dynamics that we aim to learn from data. \NEWW{The nominal model was obtained from a computer-aided design (CAD) model of another much lighter-weight Raspberry Pi quadrotor with a Raspberry Pi computer and an F450 frame (Fig. \ref{fig:quadrotor_barebone}), which is less accurate and far from the unknown ground-truth model of our upgraded quadrotors. Specifically, the nominal mass and inertia matrix are $\bfM_{\bfv 0} =  1.3\bfI$ and $\bfM_{\bfomega 0} = \diag([0.12, 0.12, 0.2])$, respectively, for the upgraded quadrotor in Fig. \ref{fig:quadrotor_nuc}}. The other nominal matrices were set to zero: $\bfD_{\bfv 0}(\frakq, \frakp) = \bf0$, $\bfD_{\bfomega 0}(\frakq, \frakp) = \bf0$, $V_0(\frakq) = \bf0$ and $\bfB_0(\frakq) = \bf0$. We modified the PX4 firmware \cite{meier2015px4} to expose the normalized thrust and torque being sent to the motors. The firmware's normalization of thrust and torque is unknown and, in fact, is learned from data via the input gain matrix $\bfB(\frakq)$. We collected $12$ state-control trajectories by flying the quadrotor from a starting pose to $12$ different poses using a PID controller provided by the PX4 flight controller \cite{meier2015px4}. The trajectories were used to generate a dataset $\mathcal{D} = \{t_{0:N}^{(i)},\mathbf\frakq_{0:N}^{(i)}, \bfzeta_{0:N}^{(i)}, \bfu^{(i)})\}_{i=1}^D$ with $N = 1$ and $D =~10000$. We trained our model as described in Sec. \ref{subsec:SE3_dyn_learning} for $5000$ steps.

The trained model was used with the control policy in Sec. \ref{subsec:crazieflie_quad} to track different trajectories: a verticle circle, a vertical lemniscate, and a 3D piecewise-linear trajectories. Fig. \ref{fig:drone_upgrade_traj_3d} and \ref{fig:drone_upgrade_trajplot} show that we achieve better tracking performance using our learned dynamics model and energy-based control compared to the nominal model and the geometric controller in \cite{lee2010geometric}. The tracking errors of our controller with a learned model improve by $2-4$ times compared to those of geometric control based on the nominal model, as shown in Table \ref{table:compare_tracking_errors}.

\subsubsection{Learning quadrotor dynamics with extra payload}
\label{subsubsec:real_quadrotor_payload}

In this section, we demonstrate that after our dynamics model is trained, if there is a change in the quadrotor dynamics, e.g., an extra payload is added, we are able to update the dynamics quickly starting from the previously trained model. We attached a coffee can to the quadrotor frame (Fig. \ref{fig:quadrotor_nuc_payload}) to change the mass and inertia matrix of the robot. We, then, collected a new dataset by driving the quadrotor to $12$ different poses, and trained our dynamics model for only $100$ steps, initialized with the trained model in Sec. \ref{subsubsec:real_quad_upgrade}.

In the presence of the coffee can payload, the tracking performance of the controller with the previously learned model degrades as shown in Fig. \ref{fig:real_tracking_traj_payload} and \ref{fig:real_tracking_zigzag_payload}. Meanwhile, after a quick model update, the robot is able to track desired trajectories accurately again. Table \ref{table:compare_tracking_errors} shows that our updated model improves the tracking errors by $1.5$--$4$ times compared to the previously learned model.

%% file: tex/Conclusion.tex
\section{Conclusion}
\label{sec:conclusion}

This paper proposed a neural ODE network design for robot dynamics learning that captures Lie group kinematics, e.g. $SE(3)$, and port-Hamiltonian dynamics constraints by construction. It also developed a general control approach for trajectory tracking based on the learned Lie group port-Hamiltonian dynamics. The learning and control designs are not system-specific and, thus, can be applied to different types of robots whose states evolve on Lie group. These techniques have the potential to enable robots to quickly adapt their models online, in response to changing operational conditions or structural damage, and continue to maintain stability during autonomous operation. Future work will focus on extending our formulation to allow learning multi-rigid-body dynamics, handling contact, and online adaptation to disturbances and structural changes in the dynamics.

%% file: tex/Appendix.tex
\section{Appendix}
\label{sec:appendix}
\subsection{Implementation Details}
\label{subsec:implement_details}

We used fully-connected neural networks whose architecture is shown below. The first number is the input dimension while the last number is the output dimension. The numbers in between are the hidden layers' dimensions and activation functions. The value of $\varepsilon_\bfv$ and $\varepsilon_{\bfomega}$ in \eqref{eq:M_cholesky} is set to $0.01$.
\begin{enumerate}
	\item Pendulum:
	\begin{itemize}
		\item Input dimension: $9$. Action dimension: $1$.
		\item $\bfL(\mathbf\frakq)$:\\ 9 - 300 Tanh - 300 Tanh - 300 Tanh - 300 Linear - 6.
		\item $\calV_{\bftheta}(\mathbf\frakq)$: 9 - 50 Tanh - 50 Tanh - 50 Linear - 1.
		\item $\bfB_{\bftheta}(\mathbf\frakq)$: 9 - 300 Tanh - 300 Tanh - 300 Linear - 3.
	\end{itemize}
	\item Pybullet quadrotor:
		\begin{itemize}
		\item Input dimension: $12$. Action dimension: $4$.
		\item $\bfL_\bfv(\mathbf\frakq)$ only takes the position $\bfp \in \mathbb{R}^3$ as input:\\ 3 - 400 Tanh - 400 Tanh - 400 Tanh - 400 Linear - 6.
		\item $\bfL_{\bfomega}(\mathbf\frakq)$ only takes the rotation matrix $\bfR \in \mathbb{R}^{3\times 3}$ as input:\\ 9 - 400 Tanh - 400 Tanh - 400 Tanh - 400 Linear - 6.
		\item $\calV_{\bftheta}(\mathbf\frakq)$: 12 - 400 Tanh - 400 Tanh - 400 Linear - 1.
		\item $\bfB_{\bftheta}(\mathbf\frakq)$: 12 - 400 Tanh - 400 Tanh - 400 Linear - 24.
	\end{itemize}
 \item Real PX4 quadrotor:
		\begin{itemize}
		\item Input dimension: $12$. Action dimension: $4$.
		\item $\bfL_\bfv(\mathbf\frakq)$ only takes the position $\bfp \in \mathbb{R}^3$ as input:\\ 3 - 20 Tanh - 20 Tanh - 20 Tanh - 20 Linear - 6.
		\item $\bfL_{\bfomega}(\mathbf\frakq)$ only takes the rotation matrix $\bfR \in \mathbb{R}^{3\times 3}$ as input:\\ 9 - 20 Tanh - 20 Tanh - 20 Tanh - 20 Linear - 6.
  		\item $\bfD_{\bfv;{\bftheta}}(\mathbf\frakq)$ only takes the position $\bfp \in \mathbb{R}^3$ as input:\\ 3 - 20 Tanh - 20 Tanh - 20 Tanh - 20 Linear - 6.
		\item $\bfD_{\bfomega;{\bftheta}}(\mathbf\frakq)$ only takes the rotation matrix $\bfR \in \mathbb{R}^{3\times 3}$ as input:\\ 9 - 20 Tanh - 20 Tanh - 20 Tanh - 20 Linear - 6.
		\item $V(\mathbf\frakq)$: 12 - 20 Tanh - 20 Tanh - 20 Linear - 1.
		\item $\bfB_{\bftheta}(\mathbf\frakq)$: 12 - 20 Tanh - 20 Tanh - 20 Linear - 24.
            \end{itemize}
\end{enumerate}

\subsection{Derivation of Hamiltonian Dynamics on SE(3) from Hamiltonian Dynamics on a Matrix Lie Group}
\label{subsec:derivation_dual_map}

\NEWW{
The Hamiltonian dynamics on SE(3) in \eqref{eq:portham_dyn_SE3} can be obtained from the general matrix Lie group Hamiltonian dynamics in \eqref{eq:ham_dyn_lie_group} by computing explicit expressions for the terms $\sfa\sfd^*_{\bfxi} (\frakp)$ and $\sfT_\bfe^*\sfL_\frakq\left(\bfeta\right)$ with $\bfeta = \frac{\partial \calH (\frakq, \frakp)}{\partial \frakq}$. 

To obtain an explicit expression for $\sfT_\bfe^*\sfL_\frakq\left(\bfeta\right)$, we use \eqref{eq:dual_TeLq} and the pairing in Def.~\ref{def:dot_product}:
\begin{equation}
\begin{aligned}
\langle\sfT_\bfe^*\sfL_\frakq (\bfeta), \bfxi\rangle &= \langle\bfeta, \sfT_\bfe\sfL_\frakq (\bfxi)\rangle = \langle \bfeta, \frakq \bfxi\rangle\\
&= \tr(\bfeta^\top\frakq\bfxi) = \langle \frakq ^\top\bfeta, \bfxi\rangle.
\end{aligned}
\end{equation}
Thus, $\sfT_\bfe^*\sfL_\frakq (\bfeta) = \sfP_{\frakg^*}(\frakq ^\top\bfeta)$, where $\sfP_{\frakg^*}$ is an orthogonal projector on $\frakg^*$ \cite[Def.~3.60]{boumal2023intromanifolds}, which depends on the specific matrix Lie group. For example, on SE(3) \cite{arathoon2016coadjoint} with $\bfA \in \bbR^{3 \times 3}$, $\bfa,\bfb \in \bbR^3$, and $c \in \bbR$:
\begin{equation}
\sfP_{\frakg^*}\left( \begin{bmatrix}
\bfA & \bfa \\
\bfb^\top & c 
\end{bmatrix}\right) = \begin{bmatrix}
\frac{1}{2}(\bfA-\bfA^\top) & \bfa \\
\bf0^\top & 0 
\end{bmatrix}.
\end{equation}

To obtain an explicit expression for $\sfa\sfd^*_{\bfxi} (\frakp)$, we use Def.~\ref{def:coadjoint} and the pairing in Def.~\ref{def:dot_product}:
\begin{equation}
\begin{aligned}
\langle \sfad^*_{\bfxi}(\frakp), \bfpsi\rangle &= \langle \frakp, \sfad_{\bfxi}(\bfpsi)\rangle = \langle \frakp, [\bfxi,\bfpsi]\rangle\\
&= \tr(\frakp^\top (\bfxi \bfpsi - \bfpsi\bfxi)) = \langle [\bfxi^\top,\frakp], \bfpsi \rangle.
\end{aligned}
\end{equation}
Thus, $\sfad^*_{\bfxi}(\frakp) = \sfP_{\frakg^*}([\bfxi^\top,\frakp])$.


\subsubsection{Expression for $\sfT_\bfe^*\sfL_\frakq (\bfeta)$ on SE(3)}
On SE(3), we have:
\begin{equation}
\frakq = \begin{bmatrix}
\bfR & \bfp \\
\bf0^\top & 1 
\end{bmatrix}, \quad \bfxi = \begin{bmatrix}
\hat{\bfomega} & \bfv\\
\bf0^\top & 0
\end{bmatrix}, \quad \bfeta  = \begin{bmatrix}
    \bfeta_\bfR  & \bfeta_\bfp \\
    \bf0^\top & 0 
    \end{bmatrix},
\end{equation}
and
\begin{equation} \label{eq:Te_star_Lq}
\begin{aligned}
\langle \bfeta, & \frakq \bfxi\rangle = \langle \bfeta_\bfR, \bfR\hat{\bfomega}\rangle + \langle \bfeta_\bfp,\bfR\bfv\rangle\\
&= \langle\frac{1}{2}( \bfR^\top\bfeta_\bfR - \bfeta_\bfR^\top\bfR), \hat{\bfomega}\rangle + \langle \bfR^\top\bfeta_\bfp, \bfv\rangle\\
&= \left\langle \begin{bmatrix}
\frac{1}{2}( \bfR^\top\bfeta_\bfR - \bfeta_\bfR^\top\bfR) & \bfR^\top\bfeta_\bfp\\
\bf0^\top & 0
\end{bmatrix}, \begin{bmatrix}
\hat{\bfomega} & \bfv\\
\bf0^\top & 0
\end{bmatrix} \right\rangle\\
&= \langle\sfT_\bfe^*\sfL_\frakq (\bfeta), \bfxi\rangle,
\end{aligned}
\end{equation}
where we used the properties $\tr(\bfA\bfB) = \tr(\bfB\bfA)$ and $\tr(\hat{\bfx}\bfA) = \frac{1}{2}\tr(\hat{\bfx}(\bfA-\bfA^\top))$ of the hat map in the second equality.

\subsubsection{Expression for $\sfa\sfd^*_{\bfxi} (\frakp)$ on SE(3)} On $\frakse(3)$, we have:
\begin{equation}
\frakp = \begin{bmatrix}
\hat{\bfa} & \bfb \\
\bf0^\top & 0 
\end{bmatrix}, \quad \bfxi = \begin{bmatrix}
\hat{\bfomega} & \bfv\\
\bf0^\top & 0
\end{bmatrix}, \quad \bfpsi = \begin{bmatrix}
    \hat{\bfc}  & \bfd \\
    \bf0^\top & 0 
    \end{bmatrix},
\end{equation}
and
\begin{equation} 
\begin{aligned}
\langle \frakp,&[\bfxi,\bfpsi]\rangle = \left\langle \begin{bmatrix}
    \hat{\bfa} & \bfb \\ \bf0^\top & 0
\end{bmatrix}, \begin{bmatrix}
    [\hat{\bfomega}, \hat{\bfc}] & \hat{\bfomega} \bfd + \hat{\bfv}\bfc \\ \bf0 & 0\end{bmatrix} \right\rangle\\
&=  \langle \hat{\bfa}, [\hat{\bfomega}, \hat{\bfc}] \rangle + \langle \bfb, \hat{\bfomega} \bfd + \hat{\bfv}\bfc \rangle\\
&= \langle[\hat{\bfa},\hat{\bfomega}],\hat{\bfc}\rangle + \langle \bfb,\hat{\bfv}\bfc \rangle + \langle \bfb, \hat{\bfomega} \bfd \rangle\\
&= \langle[\hat{\bfa},\hat{\bfomega}],\hat{\bfc}\rangle -\tr(\bfv\bfb^\top \hat{\bfc}) + \langle \hat{\bfomega}^\top \bfb,  \bfd \rangle\\
&= \langle[\hat{\bfa},\hat{\bfomega}],\hat{\bfc}\rangle + \langle \frac{1}{2}(\bfv\bfb^\top - \bfb \bfv^\top),\hat{\bfc} \rangle + \langle \hat{\bfb}\bfomega,  \bfd \rangle\\
&= \langle[\hat{\bfa},\hat{\bfomega}] + \frac{1}{2}[\hat{\bfb},\hat{\bfv}],\hat{\bfc}\rangle + \langle \hat{\bfb}\bfomega ,  \bfd \rangle\\
&= \left\langle \begin{bmatrix}
    [\hat{\bfa},\hat{\bfomega}] + \frac{1}{2}[\hat{\bfb},\hat{\bfv}] & \hat{\bfb}\bfomega\\ \bf0^\top & 0
\end{bmatrix}, \begin{bmatrix} \hat{\bfc} & \bfd\\\bf0^\top & 0 
    \end{bmatrix}\right\rangle\\
&= \langle \sfad^*_{\bfxi}(\frakp), \bfpsi\rangle,
\end{aligned}
\end{equation}
where we used the hat map properties $\hat{\bfx}^\top = -\hat{\bfx}$, $\hat{\bfx}\bfy = - \hat{\bfy}\bfx$, and $\tr(\hat{\bfx}\bfA) = \frac{1}{2}\tr(\hat{\bfx}(\bfA-\bfA^\top))$.
}

\subsubsection{Consistency Between Hamiltonian Dynamics on a Matrix Lie Group and on SE(3)}
\NEWW{
Denote the momentum in \eqref{eq:p_lie_alg} as $\frakp_{\bfxi} = \begin{bmatrix}
    \frakp_{\hat{\bfomega}} & \frakp_\bfv \\ \bf0^\top & 0
\end{bmatrix}$ and the momentum in \eqref{eq:momenta_Mtwist} as  $\frakp_{\bfzeta} = \begin{bmatrix}
    \frakp_{\bfomega} \\ \frakp_\bfv
\end{bmatrix}$, where $\bfxi = \hat{\bfzeta}$. Let $\frakp_{\hat{\bfomega}} = \frac{\partial \calL}{\partial \hat{\bfomega}} = \hat{\bfmu}$. By the chain rule, we have:
\begin{equation}
    \frakp_{\omega_i} = \frac{\partial \calL}{\partial \bfomega_i} = \left\langle\hat{\bfmu}, \frac{\partial \hat{\bfomega}}{\partial \omega_i}\right\rangle =  2 \mu_i,
\end{equation}
where $\bfomega = [\omega_1\quad\omega_2\quad\omega_3]^\top$, $\bfmu = [\mu_1\quad\mu_2\quad\mu_3]^\top$ or 
\begin{equation} 
    \frakp_{\hat{\bfomega}} = \frac{1}{2} \hat{\frakp}_{\bfomega}
\end{equation}
Therefore, we have:
\begin{equation}  \label{eq:p_xi_p_zeta}
    \frakp_{\bfxi} =  \begin{bmatrix}
\frac{1}{2} \hat{\frakp}_{\bfomega} & \frakp_\bfv\\
\bf0^\top & 0
\end{bmatrix}, \text{i.e., } \bfa = \frac{1}{2} \hat{\frakp}_{\bfomega}, \bfb = \frakp_\bfv,
\end{equation}
leading to:
\begin{equation} \label{eq:ad_star_xi}
    \sfad^*_{\bfxi}(\frakp_{\bfxi}) = \begin{bmatrix}
    \left( \frac{1}{2} \hat{\frakp}_{\bfomega}\bfomega + \frac{1}{2}\hat{\frakp}_\bfv\bfv \right) ^\wedge & \hat{\frakp}_\bfv\bfomega \\ \bf0^\top & 0
\end{bmatrix},
\end{equation}
By plugging in \eqref{eq:p_xi_p_zeta}, \eqref{eq:ad_star_xi}, \eqref{eq:Te_star_Lq} in the matrix Lie group Hamiltonian dynamics \eqref{eq:ham_dyn_lie_group}, we obtain the Hamiltonian dynamics on SE(3) in \eqref{eq:portham_dyn_SE3}.}